\documentclass{article}
\usepackage{kr}

\usepackage{times}
\usepackage{soul}
\usepackage{url}
\usepackage[hidelinks]{hyperref}
\usepackage[utf8]{inputenc}
\usepackage[small]{caption}
\usepackage{graphicx}
\usepackage{amssymb}
\usepackage{amsmath}
\usepackage{amsthm}
\usepackage{booktabs}
\usepackage{multirow}
\usepackage[hidelinks]{hyperref}
\usepackage[ruled,vlined,linesnumbered]{algorithm2e}
\SetNlSty{textrm}{\normalsize}{:} %
\usepackage{latexsym}

\usepackage{mathrsfs}
\usepackage{stmaryrd}
\usepackage[usenames,dvipsnames,svgnames,table]{xcolor}
\usepackage{tikz}
\usetikzlibrary{fadings}
\usetikzlibrary{shapes,decorations,positioning}
\usetikzlibrary{fadings,shapes,decorations,positioning,calc}
\usetikzlibrary{decorations.markings,decorations.pathreplacing}
\usetikzlibrary{shapes.geometric,automata,positioning}
\usetikzlibrary{shadows}\usetikzlibrary{calc}
\usetikzlibrary{decorations,decorations.pathmorphing,decorations.pathreplacing}
\usetikzlibrary{calc}

\usepackage{refcount}

\usepackage{caption}
\usepackage{subcaption}

\usepackage{xspace}
\usepackage{graphicx}
\usepackage{enumerate}
\usepackage{csquotes}
\usepackage{xparse}
\usepackage{multicol}
\newcommand{\citet}[1]{\citeauthor{#1} \citeyear{#1}}
\usepackage{marvosym}

\newtheorem{theorem}{Theorem}
\newtheorem{proposition}[theorem]{Proposition}
\newtheorem{lemma}[theorem]{Lemma}
\newtheorem{corollary}[theorem]{Corollary}
\newtheorem{definition}[theorem]{Definition}
\newtheorem{example}{Example}

\makeatletter
\g@addto@macro\normalsize{%
		\setlength{\abovedisplayskip}{3.0pt}
		\setlength{\belowdisplayskip}{3.0pt}
	\addtolength\voffset{-0.45cm}
	\addtolength\textheight{0.45cm}
		\setlength{\textfloatsep}{6pt plus 1.0pt minus 2.0pt}
}
\makeatother
\makeatletter
\newcommand{\ifLatexThree}[2]{\@ifpackageloaded{xparse}{#1}{#2}}
\newcommand{\ifAMSmath}[2]{\@ifpackageloaded{amsmath}{#1}{#2}}
\newcommand{\ifMathSCR}[2]{\@ifpackageloaded{mathrsfs}{#1}{#2}}
\newcommand{\ifMathHyperREF}[2]{\@ifpackageloaded{hyperref}{#1}{#2}}
\makeatother

\ifLatexThree{%
	\NewDocumentCommand{\headword}{s o m}{\IfBooleanTF{#1}{#3}{\textbf{#3}}\IfNoValueTF{#2}{\index{#3}}{\index{#2}}}%
}{%
	\makeatletter
	\def\@headword#1{\textbf{#1}\index{#1}}%
	\def\@@headword#1{#1\index{#1}}%
	\def\headword#1{\@ifstar\@headword{#1}\@@headword{#1}}%
	\makeatother
}

\makeatletter
\@ifundefined{naturals}{}{}
\@ifundefined{ol}{\newcommand{\ol}[1]{\overline{#1}}}{}
\makeatother

\makeatletter
\newcommand{\textlabelmarker}[1]{%
	\protected@edef\@currentlabel{#1}%
	\phantomsection%
}
\newcommand{\textlabel}[2]{%
	\textlabelmarker{#1}%
	#1\label{#2}%
}
\makeatother

\newcommand{\ksSubSectionStar}[1]{\medskip\noindent\textbf{#1}}

\ifx\nmableit\undefined
\makeatletter
\ifx\centernot\undefined
\newcommand*{\centernot}{%
	\mathpalette\@centernot
}
\fi
\def\@centernot#1#2{%
	\mathrel{%
		\rlap{%
			\settowidth\dimen@{$\m@th#1{#2}$}%
			\kern.5\dimen@
			\settowidth\dimen@{$\m@th#1=$}%
			\kern-.5\dimen@
			$\m@th#1\not$%
		}%
		{#2}%
	}%
}
\makeatother
\DeclareRobustCommand\nmableitSymb{\mathrel|\mkern-.5mu\joinrel\sim} %
\newcommand{\nmableit}{\ensuremath{\mbox{$\,\nmableitSymb\,$}}} %
\fi

\makeatletter
\ifMathHyperREF{%
	\newcommand{\seqref}[1]{\hyperref[{#1}]{\textup{\tagform@split{\getrefnumber{#1}}}}}%
}{%
	\newcommand{\seqref}[1]{\textup{\tagform@split{\getrefnumber{#1}}}}%
}
\newcommand\tagform@split[1]{%
	\begingroup
	\m@th\normalfont(\ignorespaces #1\unskip\@@italiccorr)%
	\endgroup
}
\makeatother

\makeatletter
\newcommand{\leqnomode}{\tagsleft@true\let\veqno\@@leqno}
\newcommand{\reqnomode}{\tagsleft@false\let\veqno\@@eqno}
\newcommand{\pushright}[1]{\ifmeasuring@#1\else\omit\hfill$\displaystyle#1$\fi\ignorespaces}
\newcommand{\pushleft}[1]{\ifmeasuring@#1\else\omit$\displaystyle#1$\hfill\fi\ignorespaces}
\newcommand{\specialcell}[1]{\ifmeasuring@#1\else\omit$\displaystyle#1$\ignorespaces\fi}

\makeatother
\newcommand{\ksIF}{\text{if }}
\newcommand{\ksTHEN}{\text{, then }}

\newcommand{\ksAND}{\text{ and }}
\newcommand{\ksOR}{\text{ or }}

\newcommand{\ksIFF}{\text{ iff }}

\newcommand{\ksForAll}{\text{ for all }}
\newcommand{\ksOtherwise}{\text{otherwise}}

\newcommand{\modelsOf}[1]{\ensuremath{\llbracket #1\rrbracket}}
\newcommand{\negOf}[1]{{\ensuremath{\neg{#1}}}}

\DeclareMathOperator{\ksMod}{Mod}
\renewcommand{\modelsOf}[1]{\ensuremath{\ksMod(#1)}}

\DeclareMathOperator{\dom}{dom}
\DeclareMathOperator{\Cn}{Cn}

\DeclareMathOperator{\ksBel}{Bel}
\newcommand{\beliefsOf}[1]{\ensuremath{\ksBel(#1)}}

\newcommand{\setAllES}{\ensuremath{\mathcal{E}}}

\newcommand{\propLang}{\ensuremath{\mathcal{L}}}

\ifMathSCR{%
}{%
}

\usepackage{environ,xparse,xkeyval,ifthen}
\newif\ifpostulatepresent\postulatepresentfalse

\ExplSyntaxOn
\NewEnviron{postulate}[2]%
{%
	\tl_gclear:N \g_tmpa_tl%
	\tl_gput_right:Nn \g_tmpa_tl { \ifpostulatepresent\\[\smallskipamount]\fi\global\postulatepresenttrue }%
	\tl_gput_right:Nn \g_tmpa_tl { \noindent\ignorespaces\ifthenelse{\equal{#1}{}}{}{(\textlabel{#1}{pstl:#1})} & }%
	\tl_gput_right:No \g_tmpa_tl { \BODY }%
	\tl_gput_right:Nn \g_tmpa_tl { & \hspace{1.5\medskipamount} #2  }%
	\group_insert_after:N \g_tmpa_tl%
}
\ExplSyntaxOff

\renewcommand{\headword}[1]{\emph{#1}}
\DeclareDocumentCommand{\todo}{o g}{\IfNoValueTF{#1}{\begingroup\color{magenta}TODO: #2\endgroup}{\begingroup\color{magenta}#1 #2\endgroup}}

\newcommand{\change}{\ensuremath{\circ}}

\DeclareDocumentCommand{\revision}{}{\ensuremath{*}}
\DeclareDocumentCommand{\clRevision}{}{\ensuremath{\circledast}}

\newcommand{\nrRevision}{\divideontimes}
\newcommand{\dynamiclimited}{dynamic-limited}
\newcommand{\dylRevision}{\star}
\newcommand{\inherent}{inherent}
\newcommand{\immanent}{immanent}
\newcommand{\ihlimited}{inherence-limited}

\newcommand{\scope}[1]{\mathcal{S}_{#1}}
\newcommand{\Scope}[2]{\mathit{Scp}^{\!#1}\!(#2)}

\newcommand{\atomic}{latent}
\newcommand{\integral}{reasonable}
\newcommand{\preservable}{S1}
\newcommand{\inconlifted}{S2}

\newif\ifshowproofs\showproofsfalse
\title{On Limited Non-Prioritised Belief Revision Operators with Dynamic Scope}

\author{%
	Kai Sauerwald$^1$\and
	Gabriele Kern-Isberner$^2$\and
	Christoph Beierle$^{1}$\\
	\affiliations
	$^1$FernUniversität in Hagen, 58084 Hagen, Germany\\
	$^2$TU Dortmund University, 44227 Dortmund, Germany
}

\begin{document}

\maketitle              %
\begin{abstract}	
The research on non-prioritized revision studies revision operators which do not accept all new beliefs. In this paper, we contribute to this line of research by introducing the concept of \dynamiclimited\ revision, which are revisions expressible by a total preorder over a limited set of worlds.
For a belief change operator, we consider the scope, which
 consists of those beliefs which yield success of revision. 
  We show that for each set satisfying single sentence closure and disjunction completeness there exists a \dynamiclimited\ revision having the union of this set with the beliefs set as scope.
We investigate iteration postulates for belief and scope dynamics and characterise them for \dynamiclimited\ revision. As an application, we employ \dynamiclimited\ revision to studying belief revision in the context of so-called inherent beliefs, which are beliefs globally accepted by the agent. This leads to revision operators which we call inherence-limited. We present a representation theorem for inherence-limited revision, and we compare these operators and \dynamiclimited\ revision  with the closely related credible-limited revision operators. 
\end{abstract}

\section{Introduction}
\label{sec:introduction}
The AGM-approach is the most prominent approach to revision (Alchourr{\'{o}}n,  G{\"{a}}rdenfors and Makinson, \citeyear{KS_AlchourronGaerdenforsMakinson1985}).
This approach has been extended to the iterative case 
by Darwiche and Pearl \shortcite{KS_DarwichePearl1997}, who proposed to consider belief change over epistemic states and showed for iterated revision, that for revision in that framework every epistemic state has to be equipped with a total preorder over all possible worlds. 
At least since then, total preorders are one of most important representation formalisms for iterated change.
However, a potentially undesired behaviour of AGM revision is that it always accepts all new beliefs unquestioned.
As an alternative, the area of non-prioritized revision provides revision operators which do not accept all beliefs.
A lot of work has been done on non-prioritized change over epistemic states, e.g improvement operators \cite{KS_KoniecznyPinoPerez2008,KS_KoniecznyGrespanPinoPerez2010,KS_SchwindKonieczny2020}, credibility-limited revision \cite{KS_BoothFermeKoniecznyPerez2012,KS_BoothFermeKoniecznyPerez2014}, core-revision \cite{KS_Booth2002}, and decrement operators \cite{KS_SauerwaldBeierle2019}; and even more on non-prioritized change in other change frameworks, e.g. \cite{KS_Hansson1999a,KS_HanssonFermeCantwellFalappa2001,KS_GarapaFermeReis2018a}.

A successful model of non-priotized revision are credibility-limited revision operators, limiting the revision process to so-called credible beliefs.
In particular, in the approach by Booth et al. \cite{KS_BoothFermeKoniecznyPerez2012,KS_BoothFermeKoniecznyPerez2014} all 
beliefs of a belief set are considered as credible beliefs.
As a consequence, credibility-limited revision operators are characterizable by total preorders over the models of the belief set and further interpretations, but not necessarily over all possible worlds.

In this paper, we seize on the idea of credible-limited revision  in the framework of Darwiche and Pearl and propose a more radical version which maintains semantically a total preorder over an arbitrary subset of all possible worlds.
We denote these operators as \dynamiclimited\ revision operators, which will behave similarly to credibility-limited revision \cite{KS_BoothFermeKoniecznyPerez2012}: when new beliefs arrive, then we accept this belief when it shares models with the domain of the total preorder, otherwise we keep the prior belief sets. %

The motivation to investigate these operators is twofold. 
First, we think that not all beliefs in a belief set might be reasonable beliefs in a strong sense. 
This is due to various unreliable sources of beliefs, like sensors. 
Clearly, a revision operator should, if possible, yield a belief set that could be considered reasonable.
However, when it comes to (iterative) revision the initial beliefs might not only contain reasonable beliefs.
Belief revision should also deal with such situations.
Second, representing a total preorder over all interpretations of the underlying logic comes with high representational costs and is for larger signatures infeasible, as the number of interpretations grows exponentially with the size of the signature.
Thus, it is desirable to investigate approaches to reduce these representational costs.

We define \dynamiclimited\ revision operators semantically by having three components for every epistemic state \( \Psi \): a belief set \( \beliefsOf{\Psi} \), a set of worlds \( \scope{\Psi} \) and a total preorder \( \preceq_{\Psi} \) over \( \scope{\Psi} \).
Note that for  \dynamiclimited\ revision operators, \( \beliefsOf{\Psi} \) and \( \preceq_{\Psi} \) are not as strongly coupled as for credibility-limited revision operators.
Thus, changes by \dynamiclimited\ revision operators are three-dimensional. %
One dimension is the change on the set of beliefs, the  second is the change of order of elements in $ \preceq_{\Psi} $, and the third is the change of $ \scope{\Psi} $. 	
The set  \( \scope{\Psi} \) describes semantically what beliefs are get accepted for revision. 
Because \( \scope{\Psi} \) is a semantic component, we characterise this set syntactically. 
This will lead to new postulates and a representation theorem for \dynamiclimited\ revision operators.

To investigate what impact a change from  \( \scope{\Psi} \) to \(  \scope{\Psi\circ\alpha} \) has, we consider the interrelation between \( \scope{\Psi} \) and the syntactic concept of \emph{scope}.
A belief \( \alpha \) is in the scope, if \( \alpha \) is believed after a revision by \( \alpha \).
Thus, the scope does not conceptualise the beliefs which are credible in the sense of credibility-limited revision, it describes the beliefs that are acceptable for a revision.
To motivate the relevance of the scope consider the following example.
\begin{example}\label{example:karl}
	Imagine an agent who initially beliefs \( \alpha \lor \beta \), e.g. \( \alpha \) stands\footnote{Thanks to {Diana Howey}, who gave inspiration to this example.} for \enquote{Karl is a camel with one hump} and  \( \beta \) \enquote{Karl is a camel with two humps}. 
	Now the agent starts an iterative revision process to get clear about the state of \( \alpha\lor\beta \). Let assume our agent receives the information \( \beta \).
	In this initial stage of the iterative revision process, the agent might be open-minded, and a revision by \( \beta \) might be successful, i.e. the agent now believes \( \beta \).
	But, when receiving afterwards a belief \( \gamma \) which is contrary to \( \beta \), then it is plausible that the agent does not accept the belief \( \gamma \) (the belief \( \gamma \) is not acceptable for a revision, thus not part of the scope).
	For instance, \( \gamma \) could state \enquote{Karl has no hump}.
\end{example}

Changing the scope in an iterative setting has gained less explicit attention and the investigations in this paper are just a first step, where \dynamiclimited\ revision operators enable us to talk about that phenomenon. 
The only work we know is due to Booth et al. (\citeyear{KS_BoothFermeKoniecznyPerez2012}) for credibility-limited revisions.
We consider postulates for scope dynamics from their work and introduce additional postulates.

As an application of \dynamiclimited\ revision operators, we propose agents that are narrowed to a restricted set of basal beliefs when it comes to a revision, and moreover, only accept new beliefs that are rooted in these basal beliefs.
We formalize the basal beliefs as beliefs that get accepted with all its consequences regardless of the current epistemic state and call them  \inherent\  beliefs . 
The disjunction of \inherent\  beliefs yields what we call \immanent\ beliefs.
Upon these notions, we introduce \ihlimited\ revision operators, which are operators yielding either a set of \immanent\ beliefs or keep the old belief sets. 
We give a representation theorem for \ihlimited\ revision in the framework of epistemic states, by employing \dynamiclimited\ revision operators. We show that \ihlimited\ revision corresponds to maintain a total preorder over a subset of possible worlds, where the set of possible worlds is the same for all epistemic states.  
We investigate the role of \inherent\ and \immanent\  beliefs for AGM, \dynamiclimited\ and credibility-limited revision operators. 

In summary, the main contributions of this paper are:%
\footnote{We have proofs for all results, but due to space limitations, the proofs are presented in the supplementary material.}
\begin{itemize}
	\item Introduction of \dynamiclimited\ revision and \ihlimited\ revision.
	\item Introduction and investigation of postulates for the dynamics of scope and their characterisation for \dynamiclimited\ revision.
\end{itemize}
The paper is organised as follows.
In the next section we start with the background. %
In Section \ref{sec:background} we introduce the background from logic and iterated belief change.
Section \ref{sec:scope} defines the concept of scope and introduces serveral iteration postulates.
Dynamic-limited revision operators are introduced in Section \ref{sec:limited_assignments}.
In Section \ref{sec:s_charactisarion} we syntactically characterise the component \( \scope{\Psi} \) of \dynamiclimited\ revision operators.
Postulates and a full representation theorem for \dynamiclimited\ revision operators are given in Section \ref{sec:dyl_representation_theorem}. 
We investigate the iterative dynamics of \dynamiclimited\ revision operators in Section \ref{sec:dyl_iterative} by giving characterisation theorems for various iteration principles.
Section  \ref{sec:inherent_immanent_beliefs} introduces \ihlimited\ revision operators and we give a representation theorem.
The article closes with a conclusion in Section \ref{sec:discussion}.

\section{Formal Background}
\label{sec:background}

We start by presenting the logical background.

\ksSubSectionStar{Propositional Logic and Preorders}
Let $ \Sigma=\{a,b,c,\ldots\} $ be a propositional signature (non empty finite set of propositional variables) and $ \propLang $ a propositional language over $ \Sigma $. 
With lower Greek letters $ \alpha,\beta,\gamma,\ldots $ we denote formulas in $ \propLang $.
The set of propositional interpretations is denoted by $ \Omega $.
We write elements from $ \Omega $ as sequences of all propositional variables, where overlining denotes assignment to \emph{false}, e.g., in $ \ol{a}bc  $ the variable $ a $ is evaluated to false (and $ b,c $ to \emph{true}).
Propositional entailment is denoted by $ \models $, the set of models of $ \alpha $ with $ \modelsOf{\alpha} $, and $ \Cn(\alpha)=\{ \beta\mid \alpha\models \beta \} $ is the deductive closure of $ \alpha $.
For a set $ X $, we define $ \Cn(X)=\{ \beta \mid X\models\beta \} $, and the expansion of $ X $ by $ \alpha $ as $ X+\alpha = Cn(X\cup\{\alpha\}) $.
By $ \varphi_{\omega_1,\ldots,\omega_k} $ we denote a formula such that $ \modelsOf{\varphi_{\omega_1,\ldots,\omega_k}}=\{ \omega_1,\ldots,\omega_k \} $.
For two set of worlds $ \Omega',\Omega''\subseteq \Omega $ and  a total preorder $ \preceq \,\subseteq  \Omega'\times  \Omega' $ (total, reflexive and transitive relation) over $ \Omega' $, we denote with $ \dom(\preceq)=\Omega' $ the domain of $ \preceq $ and with
$ {\min(\Omega'',\preceq)}=\{ \omega\in \dom(\preceq) \mid  \omega\preceq \omega' \ksForAll \omega'\in\dom(\preceq)\cap\Omega'' \} $
the set of all $ \preceq $-minimal worlds of $ \Omega'' $ in $ \dom(\preceq) $.
For a total preorder $ \preceq $, we define $ x \prec y $ iff $ x \preceq y $ and $ y \not\preceq x $.
As convention, we use $ \leq $ for total preorders over $ \Omega $ and $ \preceq $ for total preorders over (potentially strict) subsets of $ \Omega $.

\ksSubSectionStar{Epistemic States}
Classical belief change theory \cite{KS_AlchourronGaerdenforsMakinson1985}
deals with belief sets as representation of belief states, i.e., deductively closed sets of
propositions. 
The area of iterated belief change abstracts from belief sets to \headword{epistemic states} \cite{KS_DarwichePearl1997}, 
 in which the agent maintains necessary information for all available belief change operators. 
With $ \setAllES $ we denote the set of all epistemic states over $ \propLang $.
In general, the Darwiche and Pearl framework has no further requirements on $ \setAllES $, except
that every epistemic state $ \Psi\in\setAllES $ is equipped with a set of plausible sentences $ \beliefsOf{\Psi}\subseteq \mathcal{L} $, which is assumed to be deductively closed.

In certain cases we will require more structure on $ \setAllES $. Therefore, we propose the follow principle for the set of epistemic states:
\begin{multline}
	\ksIF L\subseteq\propLang \ksAND L \text{ is consistent, then} \\ \text{there exists } \Psi\in\setAllES \text{ with } \Cn(L)=\beliefsOf{\Psi}	\tag{unbiased}\label{pstl:unbiased}
\end{multline}
where \eqref{pstl:unbiased} guarantees that for every consistent belief set there exists at least one corresponding epistemic state.
Another common assumption is consistency of every belief set; we express this by the following principle:
\begin{align}
	&\ksIF \Psi\in\setAllES \ksTHEN \beliefsOf{\Psi} \neq \Cn(\bot)	\tag{global consistency}\label{pstl:globalconsistency}
\end{align}

\noindent We write $ \Psi\models\alpha $ iff $ \alpha\in\beliefsOf{\Psi} $ and we define $ \modelsOf{\Psi}=\modelsOf{\beliefsOf{\Psi}}  $.
A \emph{belief change operator} over $ \setAllES $  (and $\mathcal{L}$) is a function $ \circ : \setAllES \times \propLang \to 
\setAllES $.

\ksSubSectionStar{Belief Revision in Epistemic States}
	Revision deals with the problem of incorporating new beliefs into an agent's belief set, 
	in a consistent way, whenever that is possible.
	The well-known approach to revision given by AGM \cite{KS_AlchourronGaerdenforsMakinson1985}
	has a counterpart in the framework of epistemic states.
Darwiche and Pearl \cite{KS_DarwichePearl1997} propose that an epistemic state $ \Psi $ should be equipped with a total preorder $ \leq_{\Psi} $ of the worlds,
where the compatibility with $ \beliefsOf{\Psi} $ is ensured by the so-called faithfulness. 
\begin{definition}[\citet{KS_DarwichePearl1997}]\label{def:faithful_assignment}
	A function $ \Psi\mapsto \leq_\Psi $ that maps each epistemic state to a total preorder on interpretations is said to be a faithful assignment if:
\begin{description}
		\item[\normalfont(\textlabel{FA1}{pstl:FA1})] \( \!\ksIF \omega_1 \in \modelsOf{\Psi} \ksAND \omega_2 \in \modelsOf{\Psi} \ksTHEN \omega_1 \simeq_\Psi \omega_2 \)
		\item[\normalfont(\textlabel{FA2}{pstl:FA2})] \(  \!\ksIF \omega_1 \in \modelsOf{\Psi} \ksAND \omega_2\notin\modelsOf{\Psi} \ksTHEN \omega_1 <_\Psi \omega_2 \)
\end{description}
\end{definition}
Intuitively, $ \leq_{\Psi} $ orders the worlds by plausibility, such that the minimal worlds 
are the most plausible worlds.
The connection to revision is given as follows.
\begin{proposition}[{\citet{KS_DarwichePearl1997}}]\label{prop:es_revision}
	A belief change operator $ * $ is an \emph{AGM revision operator for epistemic states} if there is a faithful assignment $ \Psi\mapsto \leq_\Psi $ such that:
	\begin{equation}\label{eq:repr_es_revision}
	\modelsOf{\Psi * \alpha} = \min(\modelsOf{\alpha},\leq_{\Psi})
	\end{equation}
\end{proposition}	

In this paper, we say that a belief change operator is an AGM revision operator if it is an AGM revision operator for epistemic states.

\ksSubSectionStar{Credibility-Limited Revision in Epistemic States}
\label{sec:cl_revision}
The central idea of credibility-limited revision \cite{KS_HanssonFermeCantwellFalappa2001} is to restrict the process of revision such that a revision is only performed when the agent receives a credible belief.
Credibility-limited revision operators for the framework of epistemic states were introduced by Booth, Ferm{\'{e}}, Konieczny and {Pino P{\'{e}}rez} \cite{KS_BoothFermeKoniecznyPerez2012} as operators obeying the following postulates.

\begin{definition}[\citet{KS_BoothFermeKoniecznyPerez2012}]
	A belief change operator $ \clRevision $ is a \emph{credibility-limited revision} operator if $ \clRevision $ satisfies:%
	\begin{description}
		\item[\normalfont(\textlabel{CL1}{pstl:LR1})] $ \alpha\in \beliefsOf{\Psi\clRevision\alpha} \ksOR \beliefsOf{\Psi\clRevision\alpha} = \beliefsOf{\Psi} $
		\item[\normalfont(\textlabel{CL2}{pstl:LR2})]\!\!\!\! if $ \beliefsOf{\!\Psi\!}{+}\alpha $ is consistent,
	\! then $ \beliefsOf{\!\Psi{\clRevision}\alpha\!} {=} \beliefsOf{\!\Psi\!}{+}\alpha $
		\item[\normalfont(\textlabel{CL3}{pstl:LR3})] $ \beliefsOf{\Psi\clRevision\alpha} $ is consistent
		\item[\normalfont(\textlabel{CL4}{pstl:LR4})] if $ \alpha\equiv\beta $, then $ \beliefsOf{\Psi\clRevision\alpha} = \beliefsOf{\Psi\clRevision\beta} $
		\item[\normalfont(\textlabel{CL5}{pstl:LR5})] if $ \alpha \! \in\! \beliefsOf{\Psi\clRevision\alpha} $ and $ \alpha\models\beta $, then $ \beta \in \beliefsOf{\Psi\clRevision\beta} $
		\item[\normalfont(\textlabel{CL6}{pstl:LR6})] $ \beliefsOf{\Psi\!\clRevision\!(\alpha\lor\beta)}\! =\!\begin{cases}
			\beliefsOf{\Psi\clRevision\alpha} \text{ or} \\
			\beliefsOf{\Psi\clRevision\beta} \text{ or} \\
			\beliefsOf{\Psi\clRevision\alpha} \cap \beliefsOf{\Psi\clRevision\beta}
		\end{cases}  $
	\end{description}
\end{definition}

\noindent The postulate \eqref{pstl:LR1} is known as relative success and denotes that either nothing happens or the belief change is successful in achieving the success condition of revision.
Through \eqref{pstl:LR2}, known as vacuity, new beliefs are just added when they are not in conflict with $ \beliefsOf{\Psi} $.
The postulate \eqref{pstl:LR3} ensures consistency, and by \eqref{pstl:LR4} the operator has to implement independence of syntax.
Postulate \eqref{pstl:LR5} guarantees that when the revision by a belief $ \alpha $ is successful, then it is also successful for every more general belief $ \beta $.
The trichotomy postulate \eqref{pstl:LR6} guarantees decomposability of revision of disjunctive beliefs.

Booth, Ferm{\'{e}}, Konieczny and {Pino P{\'{e}}rez} \cite{KS_BoothFermeKoniecznyPerez2012} provide the following variation of faithful assignment to capture the class of credibility-limited revision operators.

\begin{definition}[CLF-assignment \cite{KS_BoothFermeKoniecznyPerez2012}]\label{def:clf-assignment}
	A function $ \Psi \mapsto (\leq_\Psi,C_\Psi) $ is a credibility-limited faithful assignment (CLF-assignment) if $ C_\Psi $ is a set of interpretations such that $ \modelsOf{\Psi} \subseteq C_\Psi \subseteq \Omega $ and $ \leq_\Psi$ is a total preorder over $  C_\Psi $ such that $ \modelsOf{\Psi} = \min(C_\Psi,\preceq_{\Psi}) $.
\end{definition}
Booth, Ferm{\'{e}}, Konieczny and {Pino P{\'{e}}rez} showed that CLF-assignments are completely able to capture credible-limited revisions \cite[Thm. 2]{KS_BoothFermeKoniecznyPerez2012}.
\begin{proposition}[\citet{KS_BoothFermeKoniecznyPerez2012}]\label{prop:clr_booth_et_all}
	Assume \ref{pstl:globalconsistency} for \( \setAllES \).
	A belief change operator $ \clRevision $ is a credibility-limited revision operator if and only if there is a CLF-assignment $ \Psi\mapsto\leq_{\Psi} $ such that the following holds:
	\begin{equation}
		\tag{CLR}\label{eq:cl_revision}
		\modelsOf{\Psi \clRevision \alpha} = \begin{cases}
			\min(\modelsOf{\alpha},\leq_\Psi) & , \modelsOf{\alpha} \cap C_\Psi \neq \emptyset \\
			\modelsOf{\Psi} & , \ksOtherwise 
		\end{cases}
	\end{equation}
\end{proposition}

\section{Scope and Iteration Principles}
\label{sec:scope}
In this section we present different iteration principles and introduce the notion of scope.

\ksSubSectionStar{Iterated Revision.}
When considering iterative revision, the most well-known work is due to Darwiche and Pearl  \cite{KS_DarwichePearl1997}.
Driven by the insight that iteration needs additional constraints, Darwiche and Pearl proposed the following postulates:
\begin{description}
	\item[\normalfont(\textlabel{DP1}{pstl:DP1})] \( \ksIF \beta\models\alpha \ksTHEN \beliefsOf{\Psi * \alpha * \beta} = \beliefsOf{\Psi * \beta} \) 
	\item[\normalfont(\textlabel{DP2}{pstl:DP2})] \( \ksIF \beta\models\negOf{\alpha} \ksTHEN \beliefsOf{\Psi * \alpha *\beta} = \beliefsOf{\Psi * \beta} \)
	\item[\normalfont(\textlabel{DP3}{pstl:DP3})] \( \ksIF \Psi*\beta \models \alpha \ksTHEN (\Psi * \alpha) *\beta \models \alpha \)
	\item[\normalfont(\textlabel{DP4}{pstl:DP4})] \( \ksIF \Psi*\beta \not\models \negOf{\alpha} \ksTHEN (\Psi * \alpha) *\beta \not\models \negOf{\alpha} \)
\end{description}
It is well-known that these operators can be characterised in the semantic framework of total preorders.
\begin{proposition}[{\cite{KS_DarwichePearl1997}}]\label{prop:it_es_revision}
	Let $ * $ be an AGM revision operator for epistemic states. Then $ * $ satisfies \eqref{pstl:DP1} to \eqref{pstl:DP4} if and only there exists a faithful assignment $ \Psi\mapsto\leq_{\Psi} $ such that \eqref{eq:repr_es_revision} and the following is satisfied:
	\begingroup
	\begin{description}
		\item[\normalfont(\textlabel{CR8}{pstl:RR8})] \( \ksIF \omega_1,\omega_2 {\in} \modelsOf{\alpha} \ksTHEN \omega_1 \!\leq_{\Psi}\! \omega_2 \Leftrightarrow \omega_1 \!\leq_{\Psi * \alpha}\! \omega_2 \) 
		\item[\normalfont(\textlabel{CR9}{pstl:RR9})] \( \ksIF \omega_1,\omega_2 {\in} \modelsOf{\negOf{\alpha}} \ksTHEN \omega_1 \!\leq_{\Psi}\! \omega_2 \Leftrightarrow \omega_1 \!\leq_{\Psi * \alpha}\! \omega_2 \)
		\item[\normalfont(\textlabel{CR10}{pstl:RR10})] \( \ksIF \omega_1 \!\in\! \modelsOf{\alpha} \ksAND \omega_2 \!\in\! \modelsOf{\negOf{\alpha}}  \), \\\null\hfill then \(    \omega_1 \!<_{\Psi}\! \omega_2 \! \Rightarrow \! \omega_1 \!<_{\Psi * \alpha}\! \omega_2 \)
		\item[\normalfont(\textlabel{CR11}{pstl:RR11})] \( \ksIF \omega_1  \!\in\! \modelsOf{\alpha} \ksAND \omega_2 \!\in\! \modelsOf{\negOf{\alpha}}  \), \\\null\hfill then \(    \omega_1 \!\leq_{\Psi}\! \omega_2 \! \Rightarrow \! \omega_1 \!\leq_{\Psi * \alpha}\! \omega_2 \)
	\end{description}
	\endgroup
\end{proposition}%

\ksSubSectionStar{Scope.}
We define the concept of a scope as those beliefs which are accepted after a revision.
\begin{definition}
	Let \( \circ \) be a belief change operator and \( \Psi \) an epistemic state. The \emph{scope} of \( \circ \) with respect to \( \Psi \) is the set:
	\begin{equation*}
		\Scope{\circ}{\Psi} = \{ \alpha \in\propLang \mid \alpha \in \beliefsOf{\Psi\circ\alpha} \}
	\end{equation*}
\end{definition}

Note that the concept is not novel, and has been introduced before in the context of credibility-limited revision \cite{KS_BoothFermeKoniecznyPerez2012} with the notion \( C_\circ(\Psi) \). 
In this context, the context of credibility-limited revision, it has been shown to describe the set of credible beliefs. 
However, for operators we investigate in this paper, the set \( \Scope{\circ}{\Psi} \) might contain elements that someone may not describe as credible beliefs. 
This gives rationale for choosing a different notation here. %

\ksSubSectionStar{AGM Revision and Scope.}
To get an impression what \( \Scope{\circ}{\Psi} \) express, remember that every AGM revision operator accepts every new belief without doubt. Consequently, the scope of an AGM revision contains all elements of the logical language.
\begin{proposition}\label{prop:agm_scope}
For every AGM revision operator \( * \) we have \( \Scope{\revision}{\Psi}  =\propLang  \) for every \( \Psi\in\setAllES \).
\end{proposition}
\ifshowproofs
\begin{proof}
	A direct consequence of Proposition \ref{prop:es_revision}.
\end{proof} \fi

\ksSubSectionStar{Credibility-Limited Revision: Scope and Iteration Principles.}
To describe the scope of belief change operators we will make use of conditions \cite{KS_HanssonFermeCantwellFalappa2001} for a set of formulas \( X \):
\begin{align*}
&	\ksIF \alpha\lor\beta \in X \ksTHEN \alpha\in X \ksOR \beta\in X	\tag{disjunction completeness} \label{pstl:disjunction_completeness}\\
&	\ksIF \alpha{\in} X \ksAND \alpha{\models}\beta \ksTHEN \beta{\in} X	\tag{single-sentence closure}\label{pstl:singlesentenceclosure} 
\end{align*}

When it comes to the class of credibility-limited revision operators, we see a much wider variety of different scopes.
For this class, the scope is given by \ref{pstl:singlesentenceclosure} and \ref{pstl:disjunction_completeness}.

\begin{proposition}\label{prop:clr_scope}
	Let  \( \Psi \) be an epistemic state. The following statements hold:
	\begin{enumerate}[(a)]
		\item If \( \clRevision \) is a credibility-limited revision operator, then \( \beliefsOf{\Psi} \subseteq \Scope{\clRevision}{\Psi} \), and \( \Scope{\clRevision}{\Psi} \) satisfies \ref{pstl:singlesentenceclosure} and \ref{pstl:disjunction_completeness}.
		\item For each \( X\subseteq\mathcal{L} \) with  \( \beliefsOf{\Psi} \subseteq X \), and \( X \) satisfies \ref{pstl:singlesentenceclosure} and \ref{pstl:disjunction_completeness}, there exists a credibility-limited revision operator \( \clRevision \) such that \( \Scope{\clRevision}{\Psi} = X \).
	\end{enumerate}
\end{proposition}
\ifshowproofs
\begin{proof}[Proof (sketch)] Remember that every credibility-limited revision operator \( \clRevision \) is compatible with a CLF-assignment \( \Psi\mapsto (\preceq_{\Psi},C_\Psi) \).
	Clearly, by the semantic characterisation of the operator we have \( \Scope{\clRevision}{\Psi}=\{ \alpha \mid \modelsOf{\alpha} \cap C_\Psi \neq \emptyset  \} \). Employing Lemma \ref{lem:ssc_dc_M} yields the statements.
\end{proof} \fi

In
the  research on credibility-limited revision operators, the dynamics of \(  \Scope{\clRevision}{\Psi} \) was investigated.
In particular, different iteration postulation principles where considered by Booth et al. (\citeyear{KS_BoothFermeKoniecznyPerez2012}), which we rephrase in the following in our notation.

\begin{description}
\item[\normalfont(\textlabel{CLDP1}{pstl:CLDP1})] \( \ksIF \beta\models\alpha \ksAND \beta\in\Scope{\circ}{\Psi}, \) \\\null\hfill then \(  \beliefsOf{\Psi\circ\alpha\circ\beta} = \beliefsOf{\Psi\circ\beta} \)
\item[\normalfont(\textlabel{CLDP2}{pstl:CLDP2})] \( \ksIF \beta\models\negOf{\alpha} \ksAND \alpha,\beta\in\Scope{\circ}{\Psi} , \) \\\null\hfill then \( \beliefsOf{\Psi\circ\alpha\circ\beta} = \beliefsOf{\Psi\circ\beta} \)
\end{description}
The postulates \eqref{pstl:CLDP1} and \eqref{pstl:CLDP2} are meant as an analogy to \eqref{pstl:DP1} and \eqref{pstl:DP2} for credibility-limited revision.
\begin{description}
\item[\normalfont(\textlabel{CLP}{pstl:CLP})] \( \negOf{\alpha} \notin \beliefsOf{\Psi\circ\beta} \ksAND \alpha,\beta\in\Scope{\circ}{\Psi} , \) \\\null\hfill then \( \alpha \in \beliefsOf{\Psi\circ\alpha\circ\beta} \)
\item[\normalfont(\textlabel{CLCD}{pstl:CLCD})] \( \ksIF \beta\models\negOf{\alpha} \ksAND \beta\notin \Scope{\circ}{\Psi} \ksAND \alpha\in\Scope{\circ}{\Psi} , \) \\\null\hfill then \( \beta \notin \Scope{\circ}{\Psi\circ\alpha} \)
\end{description}
The postulate \eqref{pstl:CLP} corresponds to the postulate P by Booth and Meyer \cite{KS_BoothMeyer2006}, respectively the independence condition by Jin and Thielscher (\citeyear{KS_JinThielscher2007}).
Coherence of credible formulas is maintained by the postulate \eqref{pstl:CLCD}.
Furthermore, the following two postulates were introduced:
\begin{description}
\item[\normalfont(\textlabel{CM1}{pstl:CM1})] \( \ksIF \beta \in \Scope{\circ}{\Psi} \ksAND \beta\models\alpha \ksTHEN \beta \in \Scope{\circ}{\Psi\circ\alpha} \) \\{\null\hfill} (Credibility Monotony 1)
\item[\normalfont(\textlabel{CM2}{pstl:CM2})] \( \ksIF \alpha,\beta {\in} \Scope{\circ}{\Psi} \ksAND \beta\models\negOf{\alpha} \ksTHEN \beta \in \Scope{\circ}{\Psi\circ\alpha} \) {\null\hfill} (Credibility Monotony 2)
\end{description}
\eqref{pstl:CM1} states that when changing by \( \alpha \), the scope retains all beliefs more specific than \( \alpha \).
The postulate \eqref{pstl:CM2} makes the same claim about the beliefs more specific than~\( \negOf{\alpha} \).

For further explanation, the interrelation and characterisations of these postulates we refer to the original paper \cite{KS_BoothFermeKoniecznyPerez2012}.

\ksSubSectionStar{Further Iteration Principles}
By employing the notion of scope, we define further attitudes towards the acceptance of beliefs when iterating.
\begin{description}
	\item[\normalfont(\textlabel{FC}{pstl:FC})]  \(  \ksIF \alpha\notin \Scope{\circ}{\Psi}  \ksTHEN  \Scope{\circ}{\Psi} \subseteq \Scope{\circ}{\Psi\circ\alpha}) \) \\\null\hfill (failure conservatism)
	\item[\normalfont(\textlabel{FR}{pstl:FR})]  \(  \ksIF \alpha\notin \Scope{\circ}{\Psi}  \ksTHEN  \Scope{\circ}{\Psi\circ\alpha} \subseteq \Scope{\circ}{\Psi}) \) \\\null\hfill (failure reticence)
\end{description}
The postulates \eqref{pstl:FC} (respectively \eqref{pstl:FR}) states that when a belief is not accepted, that set of beliefs which get accepted does not decrease (respectively increases).
For success of revision, we provide the following analogue postulates:
\begin{description}
	\item[\normalfont(\textlabel{SC}{pstl:SC})]  \(  \ksIF \alpha\in \Scope{\circ}{\Psi}  \ksTHEN  \Scope{\circ}{\Psi} \subseteq \Scope{\circ}{\Psi\circ\alpha}) \) \\\null\hfill (success conservatism)
	\item[\normalfont(\textlabel{SR}{pstl:SR})]  \(  \ksIF \alpha\in \Scope{\circ}{\Psi}  \ksTHEN  \Scope{\circ}{\Psi\circ\alpha} \subseteq \Scope{\circ}{\Psi}) \) \\\null\hfill (success reticence)
\end{description}
The postulates \eqref{pstl:FC} to \eqref{pstl:SR} are relatively unspecific with respect to the scope, thus we provide, in the fashion of \eqref{pstl:CM1} and \eqref{pstl:CM2} the following two postulates:
\begin{description}
	\item[\normalfont(\textlabel{DOC}{pstl:DOC})]  \(  \ksIF \alpha\in \Scope{\circ}{\Psi} \ksAND \beta\models\neg\alpha \ksTHEN  \beta \notin \Scope{\circ}{\Psi\circ\alpha} \) \\\null \hfill (denial of contrary)
	\item[\normalfont(\textlabel{COM}{pstl:COM})]  \(  \ksIF \alpha\notin \Scope{\circ}{\Psi} \ksTHEN  \alpha \in \Scope{\circ}{\Psi\circ\alpha} \) \\\null\hfill (change of mind)
\end{description}
By \eqref{pstl:DOC} we express the attitude to refuse beliefs contrary to a just accepted belief. 
The postulates \eqref{pstl:COM} expresses that refusal of a belief leads to acceptance of this belief in the subsequent situation. 
\section{Dynamic-Limited Revision}
\label{sec:limited_assignments}
We now adapt the machinery of Darwiche and Pearl \cite{KS_DarwichePearl1997} by employing ideas from credibility-limited revision to operators that maintain a total preorder only over a limited set of worlds.
In particular, we start by restricting the total preorder in assignments to \( \scope{\Psi} \), a subset of $  \Omega $.
The set \( \scope{\Psi}  \) is supposed to hold the models of beliefs from the scope which are not in \( \beliefsOf{\Psi} \).

\begin{definition}%
	\label{def:limited_assignment}
	A mapping $ \Psi \mapsto (\preceq_\Psi,\scope{\Psi}) $ is called a \emph{limited assignment} if for each epistemic state $ \Psi $ the relation $ \preceq_\Psi $ is a total preorder with\footnote{Note: In principle we could permit the emptiness of $ \scope{\Psi} $, but due to space restrictions we focus on the non-empty case.} $ \emptyset \neq \scope{\Psi}={\dom(\preceq_{\Psi})} \subseteq \Omega $.                       
\end{definition}
Note that $ \scope{\Psi} $ depends on the state, thus, for each state $ \Psi $ the set $ \scope{\Psi} $ might be a completely different subset from $ \Omega $. 

As \( \modelsOf{\Psi} \) and \( \scope{\Psi} \) in general do not share the same interpretations, we adapt the concept of faithfulness.
\begin{definition}\label{def:faithfullness}
		A limited assignment $ \Psi \mapsto (\preceq_\Psi,\scope{\Psi}) $ is called \emph{faithful} 
	if for each $ \Psi $ we have:
	\begin{equation*}
		\modelsOf{\Psi} \cap {\scope{\Psi}}\neq \emptyset \Rightarrow {\min(\scope{\Psi},\preceq_\Psi)} = \modelsOf{\Psi} \cap {\scope{\Psi}}
	\end{equation*}
\end{definition}

This faithfulness condition states that minimal elements of $ \preceq_{\Psi} $ are models of $ \Psi $ (if $ \modelsOf{\Psi} $ shares elements with $ \scope{\Psi} $), but $ \min(\scope{\Psi},\preceq_{\Psi}) $ does \emph{not} necessarily contains all models of $ \Psi $. 
Clearly, limited assignments are a generalisation of Darwiche and Pearl's assignments (c.f. Definition \ref{def:faithful_assignment}), because if $ \scope{\Psi}=\Omega $ for every $ \Psi \in \setAllES $, then a (faithful) limited assignment is a (faithful) assignment.

We link limited assignments and belief change operators similar to the approaches of credibility-limited revision in Proposition \ref{prop:clr_booth_et_all} by defining a notion of compatibility \cite{KS_FalakhRudolphSauerwald2021}.
\begin{definition}
A limited assignment $ \Psi \mapsto (\preceq_\Psi,\scope{\Psi}) $  is called \emph{compatible} with a belief change operator $ \dylRevision $ if it is faithful and the following holds:
\begin{equation*}
	\tag{limited-revision}\label{eq:limited_revision}
	\modelsOf{\Psi \dylRevision \alpha} \! = \! \begin{cases}
\min(\modelsOf{\alpha},\preceq_\Psi\!) & \!\!\!\!\!, \scope{\Psi}\cap\modelsOf{\alpha}\! \neq\! \emptyset \\
		\modelsOf{\Psi} &\!\!\!\! , \ksOtherwise
	\end{cases} 
\end{equation*}
\end{definition}
Operators captured by this machinery are called \dynamiclimited\ revision operators.
\begin{definition}%
\label{def:dynamic_limited_revision}
	A belief change operator $ \dylRevision $ is called a \emph{\dynamiclimited\ revision operator} if it is compatible with some limited assignment.
\end{definition}

Note that even in the sense of the original definition by Hansson et. al (\citeyear{KS_HanssonFermeCantwellFalappa2001}), \dynamiclimited\ revision is not a form of credibility-limited revision.
Moreover, in the framework of Darwiche and Pearl, \dynamiclimited\ revision is a generalisation of credibility-limited revision.
\begin{corollary}
	Every credibility-limited revision operator 
	is a \dynamiclimited\ revision operator. Furthermore, there is a \dynamiclimited\ revision operator that is not a credibility-limited revision operator.
\end{corollary}

We give an example for \dynamiclimited\ revision.
\begin{example}\label{example:karl_again}
	Consider again the agent from Example \ref{example:karl} and Karl the camel. We express  this example by the means of \dynamiclimited\ revision operators.
	Therefore, let \( \Sigma=\{ z,o,t \} \) be a propositional signature with the following intended meaning: \( z \) stands for \enquote{Karl has Zero humps}, \( o \) stands for \enquote{Karl has One hump} and \( t \) stands for \enquote{Karl has Two humps}.
	Remember that our agent believes initially \( o \lor t \). We decide to model the initial state \( \Psi \) as follows:  let \( \modelsOf{\Psi}=\{ \overline{z}o\overline{t}, \overline{z}\,\overline{o}t, \overline{z}ot \} \) and let \( \scope{\Psi}=\{ z\overline{o}\overline{t},\overline{z}o\overline{t},\overline{z}\,\overline{o}t \} \) and \( \overline{z}o\overline{t} \simeq_{\Psi} \overline{z}\,\overline{o}t \prec_{\Psi} z\overline{o}\overline{t}   \).  
	The rationale for the choice of \( \scope{\Psi} \) is that \( z,o,t \) semantically exclude each other, 
	and we think a belief revision operator should cover only (semantically) reasonable worlds.
	
	When revising \( \Psi \) by \( \beta = t \), then \( \modelsOf{\Psi\dylRevision\beta} = \{ \overline{z}\,\overline{o}t \}  \) by the machinery defined in this section.
	Clearly, we could choose \( \scope{\Psi\dylRevision\beta} \). 
	However, by following Example \ref{example:karl}, we choose \( \scope{\Psi\dylRevision\beta}= \{ \overline{z}\,\overline{o}t\} \).
	This decision corresponds to the conception of the postulate \eqref{pstl:DOC}.	
	Now, when revising by \( \gamma=o \) the operator \( \dylRevision \) denies this belief (because \( \modelsOf{\gamma}\cap\scope{\Psi\dylRevision\beta} =\emptyset \)), and yields \( \modelsOf{\Psi\dylRevision\beta} = \{ \overline{z}\,\overline{o}t \}  \).
	This complies with the behaviour we presented in Example \ref{example:karl}.
\end{example}

One might argue that \( \Psi \) in Example \ref{example:karl_again} is not modelled adequately, because \( \modelsOf{\Psi} \) contains the unlikely world \(  z\overline{o}t  \). 
However, it seems very natural that agents might came up with irrational beliefs for several reasons, but revision should be one of the processes to transform these irrational beliefs (in the light of new information) into rational beliefs.

The class of \dynamiclimited\ revision operators provides a large variety of scopes which obey \ref{pstl:singlesentenceclosure} and \ref{pstl:disjunction_completeness}. 
\begin{theorem}\label{prop:dlr_scope}
	Let \( \Psi\in\setAllES \) be an epistemic state. The following two statements hold:
	\begin{itemize}
		\item For every consistent set \( X  \) which satisfies \ref{pstl:singlesentenceclosure}  and \ref{pstl:disjunction_completeness} there
		exists a \dynamiclimited\ revision operator \( \dylRevision \) with \( {\Scope{\dylRevision}{\Psi}=\beliefsOf{\Psi} \cup X} \).
		\item If \( \dylRevision \) is an \dynamiclimited\ revision operator, then \( {\Scope{\dylRevision}{\Psi}} \) satisfies \ref{pstl:singlesentenceclosure}. Moreover, the set \( {\Scope{\dylRevision}{\Psi}} \setminus \beliefsOf{\Psi} \) satisfies \ref{pstl:singlesentenceclosure}   and \ref{pstl:disjunction_completeness}.
	\end{itemize}
\end{theorem}
\ifshowproofs
\begin{proof}
A consequence of \eqref{eq:limited_revision} and Employing Lemma \ref{lem:ssc_dc_M} yields the statements.
\end{proof} \fi

\section{Characterisation of \( \scope{\Psi} \)}
\label{sec:s_charactisarion}

In the following, we present a syntactic interpretation for the semantic concept of \( \scope{\Psi} \) for \dynamiclimited\ revision operators. 
We show that these concepts conform to each other. 
Moreover, this connects the scope \( \Scope{\dylRevision}{\Psi} \) and \( \scope{\Psi} \).

\begin{definition}\label{def:atomic}
	Let $ \change $ be a belief change operator and $ \Psi $ an epistemic state. 
	For a belief $ \alpha\in\propLang$  we define the following two conditions:
	\begin{description}
		\item[\normalfont (S1)] If
$ \beliefsOf{\Psi}+\alpha $ is consistent,\\\null\hfill then
		$ \beliefsOf{\Psi \change \beta} \subseteq \beliefsOf{\Psi \change  \alpha} $ for all \( \beta \) with $ \alpha\models\beta $.

		\item[\normalfont (S2)] %
		 For all \( \beta \), $ \beliefsOf{\Psi \change \beta} \subseteq \beliefsOf{\Psi \change \alpha} $ \\\null\hfill implies consistency of $ \beliefsOf{\Psi \change \beta} + \alpha $.

	\end{description}
\end{definition}
Condition (S1) states that if \( \alpha \) is consistent with the prior beliefs, then a change by a more general belief results in the beliefs of \( \beliefsOf{\Psi\circ\alpha} \) or fewer. 
Condition (S2) expresses (in the contrapositive version) that when changing by a belief \( \alpha \), which is inconsistent with the result of a change with a belief \( \beta \), then the result will never contain all beliefs obtained by the change with \( \beta \).

As first result, we obtain that conditions (S1) and (S2) allow to capture the interpretations in \( \scope{\Psi} \).

\begin{lemma}\label{lem:dyl_charactersion_interpretationsinTPO}
	Let $ \dylRevision $ be a \dynamiclimited\ revision operator compatible with a limited assignment $ \Psi\mapsto\preceq_{\Psi} $. The following holds:
	
	\begin{enumerate}[(a)]
		\item If  $ \omega \in \modelsOf{\Psi} $, then $ \omega\in\scope{\Psi} $ iff $\varphi_{\omega}$ satisfies \preservable\ in $ \Psi $.
		\item If  $ \omega \notin \modelsOf{\Psi} $, then $ \omega\in\scope{\Psi} $ iff $\varphi_{\omega}$ satisfies \inconlifted\ in $ \Psi $.
	\end{enumerate}
\end{lemma}
\ifshowproofs
\begin{proof}We show (a) and (b) independently.
	
	\smallskip
	\emph{Statement (a).} Let $ \omega\in\modelsOf{\Psi} $ and thus $ \beliefsOf{\Psi} \subseteq \Cn(\varphi_{\omega}) $. 
	If $ \omega\in\scope{\Psi} $, then by the faithfulness of $ \preceq_{\Psi} $ we obtain $ \omega\in{\min(\scope{\Psi},\preceq_{\Psi})} $. This implies $ \omega\in  {\min(\modelsOf{\beta},\preceq_{\Psi})} = \modelsOf{\Psi\dylRevision\beta} $ for $ \beta $ with $ \omega\models\beta $.
	Therefore, $ \varphi_{\omega} $ is \preservable.
	
	If $ \omega\notin\scope{\Psi} $, then by the non-triviality of $ \dylRevision $ we obtain $ \omega'\in\scope{\Psi} $. Now choose $ \omega\models\alpha $ where $ \alpha $ chosen such that $ \modelsOf{\alpha}=\{\omega,\omega'\} $. 
	We obtain that $ \varphi_{\omega} $ is not \preservable, because application of Definition \ref{def:dynamic_limited_revision} yields $ \modelsOf{\Psi\dylRevision\varphi_{\omega}} \not\subseteq \modelsOf{\Psi\dylRevision\alpha} $.

	\smallskip
	\emph{Statement (b).} Let $ \omega\notin\modelsOf{\Psi} $. 
	
	We Consider the case of  $ \omega\in\scope{\Psi} $ and show $ \varphi_{\omega} $ is \inconlifted.
	Let $ \beta $ a formula such that $ \omega\notin\modelsOf{\Psi\dylRevision\beta} $.
	If $ \modelsOf{\alpha}\cap\scope{\Psi}=\emptyset $, then $ \modelsOf{\Psi\dylRevision\alpha}=\modelsOf{\Psi} $ but $ \omega\notin \modelsOf{\Psi} $.
	If $ \modelsOf{\alpha}\cap\scope{\Psi}\neq\emptyset $, then $ \omega\notin{\min(\modelsOf{\alpha},\preceq_{\Psi})} $.
	Because $ \modelsOf{\Psi\dylRevision\varphi_{\omega}}=\{\omega\} $ we obtain $ \beliefsOf{\Psi\dylRevision\beta}\not\subseteq \beliefsOf{arg1} $.
	Consequently, $ \varphi_{\omega} $ is \inconlifted.

	We Consider the case of $ \omega\notin\scope{\Psi} $ and show $ \varphi_{\omega} $ is not \inconlifted.
	If $ \modelsOf{\Psi}=\emptyset $, then $ \modelsOf{\Psi\dylRevision\varphi_{\omega}}+\varphi_{\omega} $ is inconsistent.
	Consequently, we obtain that $ \varphi_{\omega} $ is not \inconlifted, because $ \modelsOf{\Psi\dylRevision\varphi_{\omega}} = \modelsOf{\Psi\dylRevision\varphi_{\omega}} $.
	If  $ \modelsOf{\Psi}\neq\emptyset $,
	then exist $ \omega'\in\modelsOf{\Psi} $. 
	Because $ \omega\notin\modelsOf{\Psi} $ we obtain $ \beliefsOf{\Psi\dylRevision\varphi_{\omega'}}+\varphi_{\omega} $ is inconsistent.
	From $ \omega\notin\scope{\Psi} $ and Definition \ref{def:dynamic_limited_revision} obtain that $ \modelsOf{\Psi\dylRevision\varphi_{\omega}}=\modelsOf{\Psi} $.
	Therefore, the formula $ \varphi_{\omega} $ is not \inconlifted. \qedhere
\end{proof} \fi

In a broader sense, (S1) expresses that more specific beliefs consistent with prior beliefs restrict the potential beliefs after the change by more general beliefs.
The condition (S2) states that when a belief results in more beliefs after revision, then this belief is compatible with the more specific belief set.
This gives rise to the following two notions.

\begin{definition}\label{def:compound}
	Let $ \change $ be a belief change operator and $ \Psi $ an epistemic state. 
	A belief $ \alpha\in\propLang$ is called
	\begin{itemize}
		\item \emph{\atomic\ in $ \Psi $} if $ \alpha $ and every consistent $ \beta $  with $ \beta \models \alpha $ satisfies \preservable\ and \inconlifted\ in $ \Psi $.
		\item \emph{\integral\ in $ \Psi $} if $ \alpha{\equiv}\alpha_1{\lor}{\ldots}{\lor}\alpha_n $ where each $ \alpha_i $ is \atomic.
	\end{itemize}
	A belief set $ X $ is \emph{\integral\ in $ \Psi $} if $ \beta {\equiv} X $ is \integral\ in~$ \Psi $.
\end{definition}

The notions of \atomic\ and \integral\ are helpful, as they capture those beliefs whose models are elements in $ \scope{\Psi} $.
\begin{proposition}\label{prop:inherent_dlr}
	Let $ \dylRevision $ be a \dynamiclimited\ revision operator compatible with a limited assignment $ \Psi\mapsto\preceq_{\Psi} $.	
	A belief $ \alpha $ is \integral\  in $ \Psi $ if and only if $ \modelsOf{\alpha}\subseteq\scope{\Psi} $ and $ \alpha $ is consistent.
\end{proposition}
\ifshowproofs
\begin{proof}

	By definition for every $ \omega\models\alpha $  we obtain that $ \varphi_{\omega} $ is \atomic. 
	From Lemma \ref{lem:dyl_charactersion_interpretationsinTPO} we obtain that $ \omega\in\scope{\Psi} $.
	Therefore $ \modelsOf{\alpha}\subseteq\scope{\Psi} $.
	
	Now let $ \alpha $ a formula such that $ \modelsOf{\alpha}\subseteq\scope{\Psi} $.
	If $ \beliefsOf{\Psi}+\alpha $ is consistent, then obtain $ \modelsOf{\Psi\dylRevision\alpha}\subseteq{\min(\scope{\Psi},\preceq_{\Psi})} $.
	Then for every $ \beta $ with $ \alpha\models\beta $ we obtain $ \modelsOf{\Psi\dylRevision\alpha}\subseteq \modelsOf{\Psi\dylRevision\beta} $ by Definition \ref{def:dynamic_limited_revision}. 
	This shows that $ \alpha $ is an \preservable\ belief in $ \Psi $.
	We show that $ \alpha $ is \inconlifted\ in $ \Psi $. 
	From $ \modelsOf{\Psi \dylRevision \alpha} \subseteq \modelsOf{\Psi \dylRevision \beta} $ obtain $ {\min(\modelsOf{\alpha},\preceq_{\Psi})} \subseteq {\modelsOf{\Psi \dylRevision \beta}} $.
	This yields immediately consistency of  $ \beliefsOf{\Psi\dylRevision\beta}+\alpha $ and thus $ \alpha $ is \inconlifted\ in $ \Psi $.
	\qedhere	
\end{proof} \fi

We can characterise the scope of a \dynamiclimited\ revision operator by \integral\ beliefs. 
\begin{theorem}\label{prop:scope_by_s}
	For an epistemic state \( \Psi \) and \dynamiclimited\ revision operator \( \dylRevision \) the following statements hold:
	\begin{itemize}
		\item Syntactically, the scope of \( \dylRevision \) and \( \Psi \) is given by:
		\begin{equation*}
			\Scope{\dylRevision}{\Psi} = \beliefsOf{\Psi} \cup \{ \alpha \mid \beta\models\alpha \ksAND \beta \text{ is \integral} \}
		\end{equation*}
		
		\item If \( \dylRevision \) is compatible with \( \Psi\mapsto(\preceq_{\Psi},\scope{\Psi}) \), then: %
		\begin{equation*}
			\Scope{\dylRevision}{\Psi} = \beliefsOf{\Psi} \cup \{  \alpha \mid  \modelsOf{\alpha} \cap \scope{\Psi} \neq \emptyset   \} 
		\end{equation*}
	\end{itemize}	
\end{theorem}
\ifshowproofs
\begin{proof}
We start by showing that $ \Scope{\dylRevision}{\Psi} = \beliefsOf{\Psi} \cup \{  \alpha \mid  \modelsOf{\alpha} \cap \scope{\Psi} \neq \emptyset   \}  $.
By Definition \ref{def:dynamic_limited_revision}, if $ \alpha $ in \( \beliefsOf{\Psi} \), then we have \( \alpha \in \beliefsOf{\Psi\dylRevision\alpha}  \).
Thus, we have \( \alpha \in \Scope{\dylRevision}{\Psi} \) for all $ \alpha $ in \( \beliefsOf{\Psi} \), i.e. \( \beliefsOf{\Psi}\subseteq \Scope{\dylRevision}{\Psi} \).
For \( \alpha \in \Scope{\dylRevision}{\Psi}\setminus \beliefsOf{\Psi} \) we obtain \( \modelsOf{\alpha}\cap\scope{\Psi}\neq\emptyset \) from Definition \ref{def:dynamic_limited_revision}.
Likewise, if \( \alpha \notin\beliefsOf{\Psi} \) and \( \modelsOf{\alpha}\cap\scope{\Psi}\neq\emptyset \), then from Definition \ref{def:dynamic_limited_revision} we obtain \( \alpha \in \beliefsOf{\Psi\dylRevision\alpha} \). 

As next step, we show \( \Scope{\dylRevision}{\Psi} = \beliefsOf{\Psi} \cup \{ \alpha \mid \beta\models\alpha \ksAND \beta \text{ is \integral} \} \).
Observe that by Proposition \ref{prop:inherent_dlr} and Definition \ref{def:dynamic_limited_revision} we have 
that \( \modelsOf{\beta}\subseteq \scope{\Psi} \) for every reasonable belief \( \beta \).
Consequently, we obtain \( \{ \alpha \mid \beta\models\alpha \ksAND \beta \text{ is \integral} \} = \{  \alpha \mid  \modelsOf{\alpha} \cap \scope{\Psi} \neq \emptyset   \} \).
This shows the claim.
\end{proof}
\fi

In the next section, we will use Proposition \ref{prop:inherent_dlr} to obtain a syntactic characterisation of the class of \dynamiclimited\ revision operators.

\section{Representation Theorem}
\label{sec:dyl_representation_theorem}

In this section, we provide postulates for \dynamiclimited\ revision operators and show that these postulates capture exactly the class of  \dynamiclimited\ revision operators.
By employing the notion of \integral\ beliefs, we propose the following postulates:

\begin{description}		
		\item[\normalfont(\textlabel{DL1}{pstl:DL1})]\!\! $ \beliefsOf{\Psi\dylRevision\alpha} = \beliefsOf{\Psi}  \ksOR \alpha\in \beliefsOf{\Psi\dylRevision\alpha}  $
\item[\normalfont(\textlabel{DL2}{pstl:DL2})]\!\! $ \beliefsOf{\Psi{\dylRevision}\alpha} = \beliefsOf{\Psi} $ or $ \beliefsOf{\Psi{\dylRevision}\alpha} $ is \integral\ in~$ \Psi $  
\item[\normalfont(\textlabel{DL3}{pstl:DL3})]\!\! if $ \beliefsOf{\Psi}\cup\{\alpha\} $ is consistent and $ \alpha  $ \integral\ in $ \Psi $,
\\\mbox{}	\hfill 
then $ \beliefsOf{\Psi\dylRevision\alpha} = \Cn(\beliefsOf{\Psi}\cup\{\alpha\}) $
\item[\normalfont(\textlabel{DL4}{pstl:DL4})]\!\! if $ \alpha\models\beta $ and $ \alpha $ is \integral\ in $ \Psi $,
then $ \beliefsOf{\Psi\dylRevision\beta} $ is \integral\ in $ \Psi $. %
\item[\normalfont(\textlabel{DL5}{pstl:DL5})]\!\!  if $ \beliefsOf{\!\Psi\!} $ is consistent, then $ \beliefsOf{\Psi\!\dylRevision\!\alpha} $ is consistent
\item[\normalfont(\textlabel{DL6}{pstl:DL6})]\!\! if $ \alpha\equiv\beta $, then $ \beliefsOf{\Psi\dylRevision\alpha} = \beliefsOf{\Psi\dylRevision\beta} $
\item[\normalfont(\textlabel{DL7}{pstl:DL8})]\!\! $ \beliefsOf{\Psi\!\dylRevision\!(\alpha\lor\beta)}\! =\!\begin{cases}
	\beliefsOf{\Psi\dylRevision\alpha} \text{ or} \\
	\beliefsOf{\Psi\dylRevision\beta} \text{ or} \\
	\beliefsOf{\Psi\dylRevision\alpha} \cap \beliefsOf{\Psi\dylRevision\beta}
\end{cases}  $
\end{description}

The postulates \eqref{pstl:DL1} and \eqref{pstl:DL6}--\eqref{pstl:DL8} are the same as the postulates \eqref{pstl:LR1}, \eqref{pstl:LR4} and \eqref{pstl:LR6} for credibility-limited revision.
By \eqref{pstl:DL2}, the operator yields either the prior belief set or a \integral\ belief set.
The postulate \eqref{pstl:DL3} is a weaker form of the vacuity postulate \eqref{pstl:LR2}, which limits vacuity to \integral\ beliefs. 
Postulate \eqref{pstl:DL4} guarantees immanence of the result of the change by a belief $ \beta $ when $ \beta $ is more general than an \integral\  belief $ \alpha $.
By \eqref{pstl:DL5}, the posterior belief set is guaranteed to be consistent, when the prior belief set is consistent.
Clearly, \eqref{pstl:DL5} is related to \eqref{pstl:LR2}, but note that credibility-limited revision assumes consistency of \( \beliefsOf{\Psi} \), an assumption we do not make for \dynamiclimited\ revision operators. 
Note that the postulate \eqref{pstl:LR5} from credibility-limited revision is implied by \eqref{pstl:DL1}--\eqref{pstl:DL8}.

Indeed, the postulates \eqref{pstl:DL1}--\eqref{pstl:DL8} capture exactly the \dynamiclimited\ revision operators.
\begin{theorem}\label{thm:dylr_representationtheorem}
	A belief change operator $ \dylRevision $ is a \dynamiclimited\ revision operator if and only if $ \dylRevision $ satisfies {\eqref{pstl:DL1}--\eqref{pstl:DL8}}.
\end{theorem}
\ifshowproofs
\begin{proof}
	\noindent\textit{The \enquote*{$ \Rightarrow $}-direction.} 
	Let $ \dylRevision $ be an operator satisfying the postulates \eqref{pstl:DL1}--\eqref{pstl:DL8}. 
	For $ \Psi $, we construct $ (\preceq_{\Psi},\Omega_\Psi) $ as follows,  let $ \preceq_{\Psi} $ be the relation with
	\(	\omega_1 \preceq_\Psi \omega_2 \ksIFF \omega_1 \in \modelsOf{\Psi\dylRevision (\varphi_{\omega_1}\lor\varphi_{\omega_2})}\)
	and $ \dom(\preceq_{\Psi})={\scope{\Psi}}= \{ \omega \mid \varphi_\omega \text{ is \atomic\ in $ \Psi $ for } \dylRevision \} $.
	The relation $ \preceq_{\Psi} $ is a total preorder:

	\emph{Totality/reflexivity.} Let $ \omega_1,\omega_2\in\scope{\Psi}) $. Then by \eqref{pstl:DL8} we have that $ \modelsOf{\Psi\dylRevision(\varphi_{\omega_1}\lor\varphi_{\omega_2})} $ is equivalent to $ \{ \omega_1 \} $ or $ \{ \omega_2 \} $ or $ \{ \omega_1,\omega_2 \} $. Therefore, the relation must be total. Reflexivity follows from totality.

	\emph{Transitivity.} 	Let $ \omega_1,\omega_2,\omega_3\in\scope{\Psi}) $ with $ \omega_1 \preceq_{\Psi} \omega_2 $ and $ \omega_2\preceq_{\Psi} \omega_3 $. Towards a contradiction assume that $ \omega_1 \not\preceq_{\Psi} \omega_3 $ holds. This implies  $ \modelsOf{\Psi\dylRevision\varphi_{\omega_1,\omega_3}}=\{\omega_3\} $.
	Now assume that $ \modelsOf{\Psi\dylRevision\varphi_{\omega_1,\omega_2,\omega_3}} = \{\omega_3\} $.
	Then by \eqref{pstl:DL8} we obtain that $ \modelsOf{\Psi\dylRevision\varphi_{\omega_2,\omega_3}}=\{\omega_3\} $, a contradiction to $ \omega_2\preceq_{\Psi}\omega_3 $.
	Assume for the remaining case $ \modelsOf{\Psi\dylRevision\varphi_{\omega_1,\omega_2,\omega_3}} \neq \{\omega_3\} $. 
	We obtain from \eqref{pstl:DL8} that $ \modelsOf{\Psi\dylRevision\varphi_{\omega_1,\omega_2,\omega_3}} $ equals either $ \modelsOf{\Psi\dylRevision\varphi_{\omega_1,\omega_2}} $ or $ \modelsOf{\Psi\dylRevision\varphi_{\omega_3}}$. 
	Since the second case is impossible, $ \omega_1\preceq_{\Psi} \omega_2 $ implies $ \omega_1\in \modelsOf{\Psi\dylRevision\varphi_{\omega_1,\omega_2}} $. Now apply \eqref{pstl:DL8} again to $ \modelsOf{\Psi\dylRevision\varphi_{\omega_1,\omega_2,\omega_3}} $ and obtain that it is either equivalent to $ \modelsOf{\Psi\dylRevision\varphi_{\omega_1,\omega_3}} $ or $ \modelsOf{\Psi\dylRevision\varphi_{\omega_2}} $.
	In both cases, we obtain a contradiction because $ \omega_1,\omega_3\in \modelsOf{\Psi\dylRevision\varphi_{\omega_1,\omega_2,\omega_3}} $.

	\smallskip
	\noindent The construction yields a faithful limited assignment:\\
	We show $ \modelsOf{\Psi}\cap\scope{\Psi})=\min(\Omega_\Psi,\preceq_{\Psi}) $. 
	Let $ \omega_1,\omega_2\in\scope{\Psi}) $ and $ \omega_1,\omega_2\in\modelsOf{\Psi} $. 
	By definition of $ {\scope{\Psi})} $ the interpretations $ \omega_1,\omega_2 $ are \atomic\ in $ \Psi $. 
	Thus, the formula $ \varphi_{\omega_1,\omega_2} $ is \integral.
	From \eqref{pstl:DL3} we obtain $ \modelsOf{\Psi\dylRevision\varphi_{\omega_1,\omega_2}}=\{ \omega_1,\omega_2 \} $ which yields by definition $ \omega_1 \preceq_{\Psi} \omega_2 $ and $ \omega_2 \preceq_{\Psi} \omega_1 $.
	Let $ \omega_1,\omega_2\in\scope{\Psi}) $ with  $ \omega_1\in\modelsOf{\Psi} $ and $ \omega_2\notin\modelsOf{\Psi} $.
	Then $ \varphi_{\omega_1,\omega_2} $ is consistent with $ \beliefsOf{\Psi} $. 
	Therefore, by 
	\eqref{pstl:DL3} and \eqref{pstl:DL5} we have $ \modelsOf{\Psi\dylRevision \varphi_{\omega_1,\omega_2}} = \modelsOf{\Psi}\cap\modelsOf{\varphi_{\omega_1,\omega_2}}=\{ \omega_1 \} $.
	Together we obtain faithfulness.
	
	\smallskip \noindent
	We show satisfaction of \eqref{eq:limited_revision} in two steps:
	
	For the first step, assume $ \modelsOf{\alpha}\cap\scope{\Psi})=\emptyset $.
	By the postulate  \eqref{pstl:DL2} we have either $ \beliefsOf{\Psi}=\beliefsOf{\Psi\dylRevision\alpha} $ or $ \beliefsOf{\Psi\dylRevision\alpha} $ is \integral\ in $ \Psi $. In the first case we are done. 
	For the second case, by the postulate \eqref{pstl:DL1}, we obtain $ \modelsOf{\Psi\dylRevision\alpha}\subseteq\modelsOf{\alpha} $. 
	Because $ \beliefsOf{\Psi\dylRevision\alpha} $ is \integral,
	the set  $ \modelsOf{\alpha }$ contains an interpretation $ \omega $ such that $ \varphi_{\omega} $ is \atomic\ in $ \Psi $, a contradiction to $ \modelsOf{\alpha}\cap\scope{\Psi})=\emptyset $.

	For the second step assume $ \modelsOf{\alpha}\cap\scope{\Psi})\neq\emptyset $.
	We show the equivalence $ \modelsOf{\Psi\dylRevision\alpha}=\min(\modelsOf{\alpha},\preceq_{\Psi}) $ by showing both set inclusions separately.
	
	We show $ \min(\modelsOf{\alpha},\preceq_{\Psi})\subseteq \modelsOf{\Psi\dylRevision\alpha} $.
	Let $ \omega\in\min(\modelsOf{\alpha},\preceq_{\Psi}) $ with $ \omega\notin\modelsOf{\Psi\dylRevision\alpha}$. 
	By construction of \( \scope{\Psi}) \) the formula $ \varphi_{\omega} $ is \atomic, and therefore \integral. 
	From postulate \eqref{pstl:DL4}, we obtain that $\beliefsOf{\Psi\dylRevision\alpha} $ is \integral\ in $ \Psi $. 
	Consequently, by Lemma \ref{prop:inherent_dlr}, every interpretation in $ \modelsOf{\Psi\dylRevision\alpha} $ is \atomic\ in $ \Psi $.
	Hence, there is at least one $ \omega'\in\beliefsOf{\Psi\dylRevision\alpha}  $ with $ \omega'\in\scope{\Psi}) $.
	If $ \beliefsOf{\Psi\dylRevision\alpha}=\beliefsOf{\Psi} $ we obtain $ \min(\modelsOf{\alpha},\preceq_{\Psi}) = \modelsOf{\Psi\dylRevision\alpha} $ by the faithfulness of $ \Psi\mapsto\leq_{\Psi} $. If $ \beliefsOf{\Psi\dylRevision\alpha}\neq\beliefsOf{\Psi} $, then by \eqref{pstl:DL1} we have $ \modelsOf{\Psi\dylRevision\alpha}\subseteq\modelsOf{\alpha} $.
	Let $ \beta=\varphi_{\omega}\lor\varphi_{\omega'} $ and $ \gamma=\beta\lor\gamma' $ such that $ \modelsOf{\gamma}=\modelsOf{\alpha} $ and $ \modelsOf{\gamma'}=\modelsOf{\alpha}\setminus\{\omega,\omega'\} $.
	By \eqref{pstl:DL8} we have either $ \modelsOf{\Psi\dylRevision\alpha}=\modelsOf{\Psi\dylRevision\beta} $ or $ \modelsOf{\Psi\dylRevision\alpha}=\modelsOf{\Psi\dylRevision\gamma'} $ or $ \modelsOf{\Psi\dylRevision\alpha}=\modelsOf{\Psi\dylRevision\beta}\cup\modelsOf{\Psi\dylRevision\gamma'} $.
	The first and the third case are impossible, because $ \omega\notin\modelsOf{\Psi\dylRevision\alpha} $ and by the minimality of $ \omega $ we have $ \omega \in \modelsOf{\Psi\dylRevision\beta} $.
	It remains the case of $ \modelsOf{\Psi\dylRevision\alpha}=\modelsOf{\Psi\dylRevision\gamma'} $. 
	Let $ \omega_{\gamma'} $ such that $ \omega_{\gamma'}\in\modelsOf{\Psi\dylRevision\alpha} $. Note that $ \omega_{\gamma'}\in\modelsOf{\gamma'}\subseteq\modelsOf{\alpha} $ and $ \omega_{\gamma'}\in\scope{\Psi}) $.
	Now let $ \delta=\beta'\lor\delta' $ with $ \delta\equiv\alpha $ such that $ \modelsOf{\beta'}=\{\omega,\omega_{\gamma'} \} $ and $ \modelsOf{\delta'}=\modelsOf{\alpha}\setminus\{ \omega,\omega_{\gamma'} \} $. 
	By minimality of $ \omega $ we have $ \omega\in\modelsOf{\Psi\dylRevision\beta'} $.
	By \eqref{pstl:DL8} we obtain $ \modelsOf{\Psi\dylRevision\alpha} $ is either equivalent to $ \modelsOf{\Psi\dylRevision\beta'} $ or to $ \modelsOf{\Psi\dylRevision\delta'} $ or to $  \modelsOf{\Psi\dylRevision\beta'} \cup \modelsOf{\Psi\dylRevision\delta'} $. The first and third case are impossible since $ \omega\notin \modelsOf{\Psi\dylRevision\alpha} $. Moreover, the second case is also impossible, because of $ \omega_{\gamma'}\notin \modelsOf{\Psi\dylRevision\delta'} $.

	We show $ \modelsOf{\Psi\dylRevision\alpha} \subseteq \min(\modelsOf{\alpha},\preceq_{\Psi}) $.
	Let  $ \omega\in\modelsOf{\Psi\dylRevision\alpha} $ with  $ \omega\notin \min(\modelsOf{\alpha},\preceq_{\Psi})  $.
	From non-emptiness of $ \min(\modelsOf{\alpha},\preceq_{\Psi}) $ and $ \min(\modelsOf{\alpha},\preceq_{\Psi})\subseteq \modelsOf{\Psi\dylRevision\alpha} $ we obtain $ \varphi_{\omega'} $ is \atomic\ in $ \Psi $ where $ \omega'\in\modelsOf{\Psi\dylRevision\alpha} $ such that $ \omega'\in \min(\modelsOf{\alpha},\preceq_{\Psi}) $.
	
	Assume that $ \varphi_\omega $ is not \atomic\ in $ \Psi $, and therefore, $ {\omega\notin\scope{\Psi})} $.
	By the postulate \eqref{pstl:DL4} we obtain from the existence of $ \omega' $ that $ \beliefsOf{\Psi\dylRevision\alpha} $ is \integral\ in $ \Psi $.  Therefore, by Lemma \ref{lem:inherent_inherence_limited} every model in $ \modelsOf{\Psi\dylRevision \alpha} $ is \atomic\ in $ \Psi $, a contradiction, and  therefore, $ \varphi_\omega $ has to be \atomic\ in $ \Psi $.
	
	Because $ \varphi_\omega $ is \atomic\ in $ \Psi $ we obtain $ \modelsOf{\Psi\dylRevision\varphi_{\omega}}=\{\omega\} $.
	Using the postulate \eqref{pstl:DL4} we obtain that $ \beliefsOf{\Psi\dylRevision\alpha} $ is \integral\ in $ \Psi $.
	Since $ \omega $ is not minimal, we have $ \omega'\in\modelsOf{\Psi\dylRevision\varphi_{\omega,\omega'}} $ and $ \omega\notin\modelsOf{\Psi\dylRevision\varphi_{\omega,\omega'}} $
	Now let $ \gamma=\varphi_{\omega,\omega'}\lor\gamma' $ with $ \modelsOf{\gamma'}=\modelsOf{\alpha}\setminus\{\omega,\omega'\} $.
	By \eqref{pstl:DL8} we have either $ \modelsOf{\Psi\dylRevision\alpha}=\modelsOf{\Psi\dylRevision\beta} $ or $ \modelsOf{\Psi\dylRevision\alpha}=\modelsOf{\Psi\dylRevision\gamma'} $ or $ \modelsOf{\Psi\dylRevision\alpha}=\modelsOf{\Psi\dylRevision\beta}\cup\modelsOf{\Psi\dylRevision\gamma'} $.
	All cases are impossible because for every case we obtain $ \omega\notin\modelsOf{\Psi\dylRevision\alpha} $.
	This shows $ \modelsOf{\Psi\dylRevision\alpha} \subseteq \min(\modelsOf{\alpha},\preceq_{\Psi}) $, and, in summary, we obtain $ \modelsOf{\Psi\dylRevision\alpha} = \min(\modelsOf{\alpha},\preceq_{\Psi}) $ and thus, \eqref{eq:limited_revision} holds.
	
	\smallskip
	\noindent\textit{The \enquote*{$ \Leftarrow $}-direction.}
	Let $ \Psi \mapsto (\preceq_\Psi,\Omega_\Psi) $ be a dynamic-limited assignment compatible with $ \dylRevision $. 
	Note that by Lemma \ref{prop:inherent_dlr} for every model $ \omega $ of an \atomic\ beliefs $ \alpha $ we obtain that $ \varphi_{\omega} $ is \atomic\ in $ \Psi $. %
	Moreover, $ \varphi_\omega $ is \atomic\ in $ \Psi $ if $ \omega\in\scope{\Psi} $.
	
	We show the satisfaction of \eqref{pstl:DL1}--\eqref{pstl:DL8}.
	From \eqref{eq:limited_revision} we obtain straightforwardly \eqref{pstl:DL1}, \eqref{pstl:DL2}, \eqref{pstl:DL5} and \eqref{pstl:DL6}.
	\begin{description}
		\item[\eqref{pstl:DL3}] Assume $ \alpha $ to be \integral\ in $ \Psi $ and $ \modelsOf{\Psi}\cap\modelsOf{\alpha}\neq\emptyset $.
		By Proposition \ref{prop:inherent_dlr} we have $ \modelsOf{\Psi}\cap\modelsOf{\alpha} \subseteq \scope{\Psi} $.
		From faithfulness of $ \Psi\mapsto\preceq_{\Psi} $ we obtain $ \min(\modelsOf{\alpha},\preceq_{\Psi})=\modelsOf{\Psi}\cap\modelsOf{\alpha} $.
		\item[\eqref{pstl:DL4}] Let $ \alpha\models\beta $ and $ \alpha $ \integral.
		By Proposition \ref{prop:inherent_dlr} we have $ \modelsOf{\alpha}\subseteq\scope{\Psi}) $.
		Then by \eqref{eq:limited_revision} we obtain that $ \beliefsOf{\Psi\dylRevision\beta} $ is \integral.
		\item[\eqref{pstl:DL8}] 
		Suppose $ \modelsOf{\alpha\lor\beta}\cap\scope{\Psi}=\emptyset $, then by \eqref{eq:limited_revision} for every formula $ \gamma $ with $ \modelsOf{\gamma} \subseteq \modelsOf{\alpha\lor\beta} $ we obtain $ \modelsOf{\Psi\dylRevision(\alpha\lor\beta)}=\modelsOf{\Psi}=\modelsOf{\Psi\dylRevision\gamma} $.
		
		Assume that $ \modelsOf{\alpha}\cap\scope{\Psi}\neq\emptyset $ and $ \modelsOf{\beta}\cap\scope{\Psi}=\emptyset $.
		Then by the postulate \eqref{eq:limited_revision} we obtain $ \modelsOf{\Psi\dylRevision(\alpha\lor\beta)}=\modelsOf{\Psi\dylRevision\alpha} $.
		The case of $ \modelsOf{\alpha}\cap\scope{\Psi}=\emptyset $ and $ \modelsOf{\beta}\cap\scope{\Psi}\neq\emptyset $ is analogue.

		Assume that $ \modelsOf{\alpha}\cap\scope{\Psi}\neq\emptyset $ and $ \modelsOf{\beta}\cap\scope{\Psi}\neq\emptyset $.
		Then we have $ \modelsOf{\Psi\dylRevision(\alpha\lor\beta)}={\min(\modelsOf{\alpha\lor\beta},\preceq_{\Psi})} $. Using logical equivalence, we obtain $ \min(\modelsOf{\alpha\lor\beta},\preceq_{\Psi})=\min(\modelsOf{\alpha}\cup\modelsOf{\beta},\preceq_{\Psi}) $.
		By Lemma \ref{lem:sem_trichotonomy} we obtain directly \eqref{pstl:DL8}. \qedhere
	\end{description}
\end{proof} \fi

\section{Iteration and {Dynamic-Limited Revision}}
\label{sec:dyl_iterative}

This section considers different iteration principles for \dynamiclimited\ revision operators.
We provide characterisation theorems for each principle. 

\ksSubSectionStar{Darwiche and Pearl Postulates.}
As a first result of this section, we investigate a semantic characterisation of Darwiche and Pearl postulates \eqref{pstl:DP1}--\eqref{pstl:DP4} for \dynamiclimited\ revision operators.

\begin{proposition}\label{prop:dylr_dp1}
	A \dynamiclimited\ revision operator $ \dylRevision $ compatible with $ \Psi\mapsto(\preceq_{\Psi},\Omega_\Psi) $ satisfies \eqref{pstl:DP1} if and only if the following conditions hold:
	\begin{enumerate}[(i)]
		\item if \(  \omega_1,\omega_2 {\in} \modelsOf{\alpha} \cap \scope{\Psi} \cap \scope{\Psi\dylRevision\alpha} \), then
		\begin{equation*}
			\omega_1 \preceq_{\Psi} \omega_2 \Leftrightarrow \omega_1 \preceq_{\Psi\dylRevision\alpha} \omega_2
		\end{equation*}
		\item if \( \left| \scope{\Psi} \cap \modelsOf{\alpha} \right| \geq 2 \), then
		\begin{equation*}
			\scope{\Psi} \cap \modelsOf{\alpha} \subseteq \scope{\Psi\dylRevision\alpha}
		\end{equation*}
		otherwise \( \left(\scope{\Psi} \cap \modelsOf{\alpha}\right) \setminus \modelsOf{\Psi\dylRevision\alpha}  \subseteq \scope{\Psi\dylRevision\alpha} \)
		\item if \( \left|  \modelsOf{\Psi} \right| \geq 2 \), then
		\begin{equation*}
			\scope{\Psi\dylRevision\alpha} \cap \modelsOf{\alpha} \subseteq \scope{\Psi}
		\end{equation*}
		otherwise \( \left(\scope{\Psi\dylRevision\alpha} \cap \modelsOf{\alpha}\right) \setminus \modelsOf{\Psi}  \subseteq \scope{\Psi} \)
	\end{enumerate}
\end{proposition}
\ifshowproofs
\begin{proof}
	\emph{The $ \Rightarrow $ direction.} Assume $ \dylRevision $ satisfies \eqref{pstl:DP1}.
	We show (i) to (iii).
	
	\smallskip
	\emph{(i).} Let $ \omega_1,\omega_2 \in \modelsOf{\alpha} $ and $ \omega_1,\omega_2\in\scope{\Psi}\cap\scope{\Psi\dylRevision\alpha} $. From \eqref{eq:limited_revision} obtain that $ \omega_1 \in \modelsOf{\Psi\dylRevision\varphi_{\omega_1,\omega_2}} $ implies $ \omega_1 \preceq_{\Psi} \omega_2 $. Consequently, by \eqref{pstl:DP1} obtain $ \omega_1\preceq_{\Psi}\omega_2 \Leftrightarrow   \omega_1\preceq_{\Psi\dylRevision\alpha}\omega_2 $.
	
	\smallskip
	\emph{(ii)} Consider the case of $ \left|\scope{{\Psi}}\cap\modelsOf{\alpha}\right| \geq 2 $. 
	Suppose there is some $ \omega\in \scope{\Psi} \cap \modelsOf{\alpha} $ but $ \omega\notin\scope{\Psi\dylRevision\alpha} $. Let $ \omega' \in\scope{\Psi} $. Observe that $ \modelsOf{\Psi\dylRevision\varphi_{\omega}} =\{\omega\} $ and $ \modelsOf{\Psi\dylRevision\varphi_{\omega'}} =\{\omega'\} $ by \eqref{eq:limited_revision}.
	If $ \modelsOf{\Psi\dylRevision\alpha}\neq\{\omega\} $, then from \eqref{eq:limited_revision} and $ \omega\notin\scope{\Psi\dylRevision\alpha} $ follows $ \modelsOf{\Psi\dylRevision\alpha\dylRevision\varphi_{\omega}} \neq\{\omega\} $ a contradiction to \eqref{pstl:DP1}.
	If $ \modelsOf{\Psi\dylRevision\alpha}=\{\omega\} $, then $ \min(\modelsOf{\alpha},\preceq_{\Psi}) =\{\omega\} $.
	Now let $ \beta=\varphi_{\omega,\omega'} $.
	From faithfulness obtain that $ \modelsOf{\Psi\dylRevision\beta}=\{\omega\} $. 
	The are two possibilities: $ \modelsOf{\Psi\dylRevision\alpha\dylRevision\beta}=\{\omega\} $ or $ \modelsOf{\Psi\dylRevision\alpha\dylRevision\beta}=\{\omega'\} $. The latter case directly contradicts \eqref{pstl:DP1}.
	The first case implies $ \modelsOf{\Psi\dylRevision\alpha\dylRevision\varphi_{\omega'}} =\{\omega\} $ a contradiction to \eqref{pstl:DP1}, because of $ \modelsOf{\Psi\dylRevision\varphi_{\omega'}} =\{\omega'\} $.
	
	Consider the case of $ \left|\scope{\Psi}\cap\modelsOf{\alpha}\right| < 2 $.
	Suppose there is some $ \omega\in \scope{\Psi} \cap \modelsOf{\alpha} $ but $ \omega\notin \modelsOf{\Psi\dylRevision\alpha}\cup\scope{\Psi\dylRevision\alpha} $. Then we obtain from \eqref{eq:limited_revision} a contradiction to \eqref{pstl:DP1}, because $ \modelsOf{\Psi\dylRevision\varphi_{\omega}} =\{\omega\} $ and $ \modelsOf{\Psi\dylRevision\alpha\dylRevision\varphi_{\omega}} \neq \{\omega\} $.
	
	\smallskip
	\emph{(iii)} Let $ \omega\in \scope{\Psi\dylRevision\alpha}\cap\modelsOf{\alpha} $ and therefore $ \modelsOf{\Psi\dylRevision\alpha\dylRevision\varphi_{\omega}}=\{\omega\} $.	
	If $ \left|\modelsOf{\Psi}\right| < 2 $ and $ \omega\notin\scope{\Psi}\cap\modelsOf{\Psi} $, then $ \modelsOf{\Psi\dylRevision\varphi_{\omega}}=\modelsOf{\Psi}\neq\{\omega\}  $. 	
	If $ \left|\modelsOf{\Psi}\right| \geq 2 $ and $ \omega\notin\scope{\Psi} $, then $ \modelsOf{\Psi\dylRevision\varphi_{\omega}}=\modelsOf{\Psi}\neq \{\omega\}  $.	
	A contradiction to \eqref{pstl:DP1} in both cases.
	
	\medskip
	\noindent\emph{The $ \Rightarrow $ direction.} 
	We prove by contraposition and show that a violation of (i), (ii) or (iii) implies a violation of \eqref{pstl:DP1}:
	
	\smallskip
	\emph{(i)} If $ \omega_1 \preceq_{\Psi} \omega_2 \not\Leftrightarrow \omega_1 \preceq_{\Psi\dylRevision\alpha} \omega_2 $, then because of $ \omega_1,\omega_2\in\scope{\Psi}\cap\scope{\Psi\dylRevision\alpha} $ we obtain $ \modelsOf{\Psi\dylRevision\varphi_{\omega_1,\omega_2}} \neq \modelsOf{\Psi\dylRevision\alpha\dylRevision\varphi_{\omega_1,\omega_2}} $ from \eqref{eq:limited_revision}.
	
	\smallskip
	\emph{(ii)} Let $ \omega\notin\scope{\Psi\dylRevision\alpha} $ and $ \omega\in\scope{\Psi}\cap\modelsOf{\alpha} $. Then, $ \modelsOf{\Psi\dylRevision\varphi_{\omega}} = \{ \omega \} $ and $ \modelsOf{\Psi\dylRevision\alpha\dylRevision\varphi_{\omega}} = \modelsOf{\Psi} $.
	
	If $ \left|\scope{\Psi}\cap\modelsOf{\alpha}\right| \geq 2 $, 
	then let  $ \omega' \scope{\Psi}\cap\modelsOf{\alpha} $ with $ \omega'\neq\omega $.
	For $ \modelsOf{\Psi\dylRevision\alpha}\neq\{\omega\} $, we directly obtain a violation of \eqref{pstl:DP1}.
	For $ \modelsOf{\Psi\dylRevision\alpha} = \{\omega\} $ obtain that $ \min(\modelsOf{\alpha},\preceq_{\Psi})=\{\omega\} $. Now observe that either $ \modelsOf{\Psi\dylRevision\alpha\dylRevision\varphi_{\omega'}} = \{\omega'\} $ or $ \modelsOf{\Psi\dylRevision\alpha\dylRevision\varphi_{\omega'}} = \modelsOf{\Psi\dylRevision\alpha} $.
	In the first case we obtain that $ \modelsOf{\Psi\dylRevision\varphi_{\omega,\omega'}}\neq\modelsOf{\Psi\dylRevision\alpha\dylRevision\varphi_{\omega,\omega'}} $. Likewise, the second case yields a violation of \eqref{pstl:DP1}, because $ \modelsOf{\Psi\dylRevision\alpha\dylRevision\varphi_{\omega'}} = \{\omega\} \neq \modelsOf{\Psi\dylRevision\varphi_{\omega'}} $.
	
	If $ \left|\scope{\Psi}\cap\modelsOf{\alpha}\right| < 2 $ and $ \omega\notin\modelsOf{\Psi\dylRevision\alpha} $, then we directly obtain $ \modelsOf{\Psi\dylRevision\alpha\dylRevision\varphi_{\omega}}=\modelsOf{\Psi\dylRevision\alpha}\neq \modelsOf{\Psi\dylRevision\varphi_{\omega}} $.
	
	\smallskip
	\emph{(iii)}  Let $ \omega\notin\scope{\Psi} $ and $ \omega\in\scope{\Psi\dylRevision\alpha}\cap\modelsOf{\alpha} $. Then $ \modelsOf{\Psi\dylRevision\alpha\dylRevision\varphi_{\omega}} = \{ \omega \} $ and $ \modelsOf{\Psi\dylRevision\varphi_{\omega}} = \modelsOf{\Psi} $.
	If $ \left|\modelsOf{\Psi}\right| \geq 2 $, or $ \left|\modelsOf{\Psi}\right| < 2 $ and $ \omega\notin\modelsOf{\Psi} $, then $ \modelsOf{\Psi\dylRevision\varphi_{\omega}} \neq \modelsOf{\Psi\dylRevision\alpha\dylRevision\varphi_{\omega}} $. \qedhere

\end{proof} \fi

One might wonder how \eqref{pstl:DP1} can enforce such a complex behaviour when using \dynamiclimited\ revision operators. The cause for this behaviour is given by the threefoldness of \dynamiclimited\ revision operators.
	When considering two-step changes, for AGM revision the posterior belief set is determined by a total preorder (and the change of the order of elements), but for \dynamiclimited\ revision operators, the result depends on \( \beliefsOf{\Psi} \), the relation \( \preceq_{\Psi} \), the set \( \scope{\Psi} \) and how these components evolve.
For the same reason, we will obtain similar complex conditions when considering \eqref{pstl:DP2}--\eqref{pstl:DP4}, \eqref{pstl:CLDP1}, \eqref{pstl:CLDP2} and \eqref{pstl:CLP}.

\begin{proposition}\label{prop:dylr_dp2}
	A \dynamiclimited\ revision operator $ \dylRevision $ compatible with $ \Psi\mapsto(\preceq_{\Psi},\Omega_\Psi) $ satisfies \eqref{pstl:DP2} if and only if the following conditions hold:
	\begin{enumerate}[(i)]
		\item if \(  \omega_1,\omega_2 {\in} \modelsOf{\negOf{\alpha}} \cap \scope{\Psi} \cap \scope{\Psi\dylRevision\alpha} \), then
		\begin{equation*}
			\omega_1 \preceq_{\Psi} \omega_2 \Leftrightarrow \omega_1 \preceq_{\Psi\dylRevision\alpha} \omega_2
		\end{equation*}
		\item if \( \left| \scope{\Psi} \cap \modelsOf{\negOf{\alpha}} \right| \geq 2 \), then
		\begin{equation*}
			\scope{\Psi} \cap \modelsOf{\negOf{\alpha}} \subseteq \scope{\Psi\dylRevision\alpha}
		\end{equation*}
		otherwise \( 
		\left(\scope{\Psi} \cap \modelsOf{\negOf{\alpha}}\right) \setminus \modelsOf{\Psi\dylRevision\alpha}  \subseteq \scope{\Psi\dylRevision\alpha} \).
		\item if \( \left|  \modelsOf{\Psi} \right| \geq 2 \), then
		\begin{equation*}
			\scope{\Psi\dylRevision\alpha} \cap \modelsOf{\negOf{\alpha}} \subseteq \scope{\Psi}
		\end{equation*}
		otherwise \( 
		\left(\scope{\Psi\dylRevision\alpha} \cap \modelsOf{\negOf{\alpha}}\right) \setminus \modelsOf{\Psi}  \subseteq \scope{\Psi} \).
	\end{enumerate}
\end{proposition}
\ifshowproofs

\begin{proof}
	Analogue to the proof of Proposition \ref{prop:dylr_dp1}.
\end{proof}
 \fi

\begin{proposition}\label{prop:dylr_dp3}
	A \dynamiclimited\ revision operator $ \dylRevision $ compatible with $ \Psi\mapsto(\preceq_{\Psi},\Omega_\Psi) $ satisfies \eqref{pstl:DP3} if and only if the following conditions hold:
	\begin{enumerate}[(i)]
		\item if and  \(  \omega_1 {\models} \alpha \) and \( \omega_2 {\not\models} \alpha\) and \( \omega_1,\omega_2{\in}\scope{\Psi}{\cap}\scope{\Psi\dylRevision\alpha} \), then
		\begin{equation*}
			\omega_1 \prec_{\Psi} \omega_2 \Rightarrow \omega_1 \prec_{\Psi\dylRevision\alpha} \omega_2
		\end{equation*}
		\item if  \(  \omega_1 \models\alpha \) and \( \omega_2 \not\models\alpha\) and \( \omega_1 \prec_{\Psi} \omega_2 \), then
		\begin{equation*}
			\omega_2 \in \scope{\Psi\dylRevision\alpha} \Rightarrow  \omega_1 \in \scope{\Psi\dylRevision\alpha} %
		\end{equation*}
		\item if \( \omega \not\models \alpha \) and \( \Psi\models \alpha  \), then  %
		\(  \omega \in\scope{\Psi\dylRevision\alpha} \Rightarrow \omega \in\scope{\Psi} \)
		\item if  \(  \omega_1 {\models} \alpha \) and \( \omega_2 {\not\models} \alpha\) and \( \omega_1 {\in} \scope{\Psi} \) and \( \omega_2 {\in} \scope{\Psi\dylRevision\alpha} \), then
		\begin{equation*}
			\left( \omega_1\notin\scope{\Psi\dylRevision\alpha} \ksOR \omega_2 \preceq_{\Psi\dylRevision\alpha} \omega_1 \right) \Rightarrow  \omega_2 \in \scope{\Psi}
		\end{equation*}
	\end{enumerate}
\end{proposition}
\ifshowproofs
\begin{proof}
	\emph{The \enquote{$ \Rightarrow $} direction.} 
	We prove satisfaction of (i) -- (iv) in the presence of \eqref{pstl:DP3}:
	
	\smallskip%
	\emph{(i)} Let \( \omega_1,\omega_2\in\scope{\Psi}\cap\scope{\Psi\dylRevision\alpha} \) with \(  \omega_1 \models\alpha \) and \( \omega_2 \not\models\alpha\), and \( \omega_1 \prec_{\Psi}\omega_2 \).
	From \eqref{eq:limited_revision} and \eqref{pstl:DP3} we easily obtain \( \omega_1 \prec_{\Psi\dylRevision\alpha} \omega_2 \) by choosing \( \beta=\varphi_{\omega_1,\omega_2} \).

	\smallskip%
	\emph{(ii)} Assume \( \omega_2\in\scope{\Psi\dylRevision\alpha} \) and \( \omega_1\notin \scope{\Psi\dylRevision\alpha} \).
	Then for \( \beta=\varphi_{\omega_1,\omega_2} \)  we obtain \( \Psi\dylRevision\beta\models\alpha  \) and \( \Psi\dylRevision\alpha\dylRevision\beta\models\alpha  \).
	A contradiction to \eqref{pstl:DP3}.

	\smallskip%
	\emph{(iii)} Let \( \omega\in\scope{\Psi\dylRevision\alpha} \) and \( \omega\notin\scope{\Psi} \). 
	From \( \Psi\models\alpha \) and \eqref{eq:limited_revision} obtain \(  \Psi\dylRevision\varphi_{\omega} \models\alpha \).
	But by \( \omega\not\models\alpha \) and \( \omega\in\scope{\Psi\dylRevision\alpha} \) we obtain a contradiction to \eqref{pstl:DP3}, because \( \Psi\dylRevision\alpha\dylRevision\varphi_{\omega}\not\models\alpha \).
	
	\smallskip%
	\emph{(iv)} Suppose \( \omega_2\notin\scope{\Psi} \) and let \( \beta=\varphi_{\omega_1,\omega_2} \).
	Both, \( \omega_1\notin \scope{\Psi\dylRevision\alpha} \) or \( \omega_2 \preceq_{\Psi\dylRevision\alpha} \omega_1 \), yields \( \Psi\dylRevision\alpha\dylRevision\beta\not\models\alpha \) because of \eqref{eq:limited_revision}.
	From \eqref{pstl:DP3} obtain \( \Psi\dylRevision\beta\not\models\alpha \). However, from \( \omega_2\notin\scope{\Psi} \) and \( \omega_1\in\scope{\Psi} \) we obtain \( \Psi\dylRevision\beta\models\alpha \).
	
	\medskip\noindent
	\emph{The \enquote{$ \Leftarrow $} direction.} Let \( \Psi\dylRevision\beta\models\alpha \). We show \( \Psi\dylRevision\alpha\dylRevision\beta\models\alpha \).
	Towards a contradiction, assume \( \omega\in\modelsOf{\Psi\dylRevision\alpha\dylRevision\beta} \) with \( \omega\notin\modelsOf{\alpha} \).
	
	Consider the case of \( \omega\in\scope{\Psi\dylRevision\alpha} \).
	If \( \omega\in\scope{\Psi} \),  then by \( \Psi\dylRevision\beta\models\alpha \) there exist \( \omega'\models\beta \land\alpha \) with \( \omega' \prec_\Psi \omega \) .
	From (i) and (ii) obtain the contradiction \( \omega' \prec_{\Psi\dylRevision\alpha} \omega \).\\
	Having \( \omega\notin\scope{\Psi} \) and \( \Psi\models\alpha \) at the same time is impossible by (iii). 
	From %
	\( \Psi\not\models\alpha \) %
	obtain that \( \modelsOf{\Psi\dylRevision\beta} = \min(\modelsOf{\beta},\preceq_{\Psi})  \) and therefore there exist \( \omega'\in\scope{\Psi} \) with \( \omega'\in \modelsOf{\alpha\land\beta} \).
	From (i) and (iv) obtain the contradiction \( \Psi\dylRevision\beta\not\models\alpha \) %
	Consider the case of \( \omega\notin\scope{\Psi\dylRevision\alpha} \). Then obtain \( \modelsOf{\Psi\dylRevision\alpha}=\modelsOf{\Psi\dylRevision\alpha\dylRevision\beta} \) from \eqref{eq:limited_revision}.
	Because of \( \omega\not\models\alpha \) and \( \omega\in\modelsOf{\Psi\dylRevision\alpha} \) we have \( \modelsOf{\Psi\dylRevision\alpha}=\modelsOf{\Psi} \) and \( \Psi\not\models\alpha \). 
	We obtain \( \modelsOf{\Psi\dylRevision\beta} \subseteq \scope{\Psi} \) as consequence.
	Because  \( \Psi\dylRevision\beta\models\alpha \) there exists \( \omega'\in\scope{\Psi} \) with \( \omega'\models\alpha\land\beta \).
	This contradicts \( \modelsOf{\Psi\dylRevision\alpha}=\modelsOf{\Psi} \).\qedhere	
\end{proof} \fi

\begin{proposition}\label{prop:dylr_dp4}
	A \dynamiclimited\ revision operator $ \dylRevision $ compatible with $ \Psi\mapsto(\preceq_{\Psi},\Omega_\Psi) $ satisfies \eqref{pstl:DP4} if and only if the following conditions hold:
	\begin{enumerate}[(i)]
		\item if and  \(  \omega_1 {\models} \alpha \) and \( \omega_2 {\not\models} \alpha\) and \( \omega_1,\omega_2{\in}\scope{\Psi}{\cap}\scope{\Psi\dylRevision\alpha} \), then
		\begin{equation*}
			\omega_1 \prec_{\Psi\dylRevision\alpha} \omega_2 \Rightarrow \omega_1 \prec_{\Psi} \omega_2
		\end{equation*}
		\item if  \(  \omega_1 \models\alpha \) and \( \omega_2 \not\models\alpha\) and \( \omega_2 \prec_{\Psi\dylRevision\alpha} \omega_1 \), then
		\begin{equation*}
			\omega_1 \in \scope{\Psi} \Rightarrow  \omega_2 \in \scope{\Psi} %
		\end{equation*}
		\item if \( \omega \not\models \alpha \) and \( \Psi \not\models \negOf{\alpha}  \), then 
		\(  \omega \in\scope{\Psi\dylRevision\alpha} \Rightarrow \omega \in\scope{\Psi} \)
		\item if  \(  \omega_1 {\models} \alpha \) and \( \omega_2 {\not\models} \alpha\) and \( \omega_1 {\in} \scope{\Psi} \) and \( \omega_2 {\in} \scope{\Psi\dylRevision\alpha} \), then
		\begin{equation*}
			\left( \omega_2\notin\scope{\Psi} \ksOR \omega_1 \preceq_{\Psi} \omega_2 \right) \Rightarrow  \omega_1 \in \scope{\Psi\dylRevision\alpha}
		\end{equation*}
	\end{enumerate}
\end{proposition}
\ifshowproofs
\begin{proof}
	Analogue to the proof of Proposition \ref{prop:dylr_dp3}.
\end{proof} \fi

As Propositions \ref{prop:dylr_dp1} to Propositions \ref{prop:dylr_dp4} show, the postulates \eqref{pstl:DP1}--\eqref{pstl:DP4} do not only impose constraints on the order of interpretations (c.f. Proposition \ref{prop:it_es_revision}), they also impose constraint on the \( \scope{\Psi} \).
Therefore, by  Proposition \ref{prop:scope_by_s}, these postulates constrains indirectly scope \( \Scope{\dylRevision}{\Psi} \).
However, Propositions \ref{prop:dylr_dp1} to Propositions \ref{prop:dylr_dp4} are rather technical, and we think that the results demonstrate that modified versions of the Darwiche and Pearl postulates, like the ones represented in Section \ref{sec:scope}, are more suitable for \dynamiclimited\ revision operators.

\ksSubSectionStar{Monotonic dynamics of \( \scope{\Psi} \) and \( \Scope{\dylRevision}{\Psi} \).} 
We continue by characterising monotonic change of \( \scope{\Psi} \) and \( \Scope{\dylRevision}{\Psi} \).

From Proposition \ref{prop:inherent_dlr} we obtain the following observation:
\begin{corollary}\label{col:s_dynamics}
	Let \( \dylRevision \) be a \dynamiclimited\ revision operator compatible with \(  \Psi\mapsto(\preceq_{\Psi},\scope{\Psi}) \).
	The following statements hold:
	\begin{enumerate}[(a)]
		\item \(  \scope{\Psi} \subseteq \scope{\Psi\dylRevision\alpha}  \) if and only if the following holds:
		\begin{equation*}
			\ksIF	\beta  \text{ is \integral\  in }  \Psi \ksTHEN   \beta  \text{ is \integral\  in }  \Psi\dylRevision\alpha  
		\end{equation*}
		\item \(  \scope{\Psi\dylRevision\alpha} \subseteq \scope{\Psi}  \) if and only if the following holds: 
		\begin{equation*}
			\ksIF	\beta  \text{ is \integral\  in }  \Psi\dylRevision\alpha \ksTHEN   \beta  \text{ is \integral\  in }  \Psi  
		\end{equation*}
	\end{enumerate}
\end{corollary}

By Proposition \ref{prop:scope_by_s}, the set \( \scope{\Psi} \) encodes semantically the beliefs of \( \Scope{\dylRevision}{\Psi} \) which are not part of the belief set.
This gives rise to the following proposition.

\begin{proposition}\label{prop:monotone_scope}
	Let \( \dylRevision \) be a \dynamiclimited\ revision operator compatible with \( \Psi\mapsto(\preceq_{\Psi},\scope{\Psi}) \). Then,
	\begin{enumerate}[(a)]
		\item 	 \( \Scope{\dylRevision}{\Psi} \subseteq \Scope{\dylRevision}{\Psi\dylRevision\alpha}  \) if and only if the following holds:
			\begin{description}
			\item[\normalfont(\textlabel{SI1}{pstl:SI1})] \( \ksIF \modelsOf{\beta} \cap \scope{\Psi} \neq\emptyset \), \\\null\hfill then \( \modelsOf{\beta} \cap \scope{\Psi\dylRevision\alpha} \neq\emptyset \ksOR \Psi\dylRevision\alpha\models\beta   \)
			\item[\normalfont(\textlabel{SI2}{pstl:SI2})] \( \ksIF \Psi{\models}\beta \ksAND \Psi\dylRevision\alpha{\not\models}\beta \ksTHEN \modelsOf{\beta}\cap\scope{\Psi\dylRevision\alpha}\neq\emptyset \)
			\end{description}	
		\item \( \Scope{\dylRevision}{\Psi\dylRevision\alpha} \subseteq \Scope{\dylRevision}{\Psi}  \) if and only if the following holds:
		\begin{description}
			\item[\normalfont(\textlabel{SD1}{pstl:SD1})] \( \ksIF \modelsOf{\beta} \cap \scope{\Psi\dylRevision\alpha} \neq\emptyset \), \\\null\hfill then \( \modelsOf{\beta} \cap \scope{\Psi} \neq\emptyset \ksOR \Psi\models\beta   \)
			\item[\normalfont(\textlabel{SD2}{pstl:SD2})] \( \ksIF \Psi\dylRevision\alpha{\models}\beta \ksAND \Psi{\not\models}\beta \ksTHEN \modelsOf{\beta}\cap\scope{\Psi}\neq\emptyset \)
			\end{description}
	\end{enumerate}
\end{proposition}
\ifshowproofs
\begin{proof}
	A consequence of Proposition \ref{prop:scope_by_s}.
\end{proof} \fi

\ksSubSectionStar{Iteration Principles considering the Scope.}
We consider further iteration principles from Section \ref{sec:scope} and characterise them for \dynamiclimited\ revision operators.
By this approach the insightful Proposition \ref{prop:dlr_scope} is very helpful.

\begin{proposition}\label{prop:CLDPX}
	Let $ \dylRevision $ be a \dynamiclimited\ revision operator compatible with $ \Psi\mapsto(\preceq_{\Psi},\scope{\Psi}) $. Then the following statements hold:
	\begin{itemize}
		\item \( \dylRevision  \) satisfies \eqref{pstl:CLDP1} if and only for all \( \omega_1,\omega_2 \) with \( \varphi_{\omega_1}\in\beliefsOf{\Psi\dylRevision\varphi_{\omega_1}} \) and \( \varphi_{\omega_2}\in\beliefsOf{\Psi\dylRevision\varphi_{\omega_2}} \) the conditions (i)--(iii) of Proposition \ref{prop:dylr_dp1} hold.
		\item \( \dylRevision  \) satisfies \eqref{pstl:CLDP2} if and only if for all \( \omega_1,\omega_2 \) with \( \varphi_{\omega_1}\in\beliefsOf{\Psi\dylRevision\varphi_{\omega_1}} \) and \( \varphi_{\omega_2}\in\beliefsOf{\Psi\dylRevision\varphi_{\omega_2}} \) the conditions (i)--(iii) of Proposition \ref{prop:dylr_dp2} hold.
	\end{itemize}
\end{proposition}

A closer look at Proposition \ref{prop:CLDPX} points out that \eqref{pstl:CLDP1} and \eqref{pstl:CLDP2} are unsatisfactory postulates for general \dynamiclimited\ revision operators, because they seem to intermingle concepts of scope dynamics and belief dynamics.
As an alternative to \eqref{pstl:CLDP1} and \eqref{pstl:CLDP2}, we consider the following two postulates:
\begin{description}
	\item[\normalfont(\textlabel{DLDP1}{pstl:DLDP1})] \( \ksIF \beta\models\alpha \ksAND \alpha,\beta \text{ are \integral\ in }  \Psi \), \\ \null\hfill \text{then } \(\beliefsOf{\Psi\dylRevision\alpha\dylRevision\beta}=\beliefsOf{\Psi\dylRevision\beta} \)
	\item[\normalfont(\textlabel{DLDP2}{pstl:DLDP2})] \( \ksIF \beta\models\negOf{\alpha} \ksAND \alpha,\beta \text{ are \integral\ in }  \Psi \), \\ \null\hfill \text{then } \(\beliefsOf{\Psi\dylRevision\alpha\dylRevision\beta}=\beliefsOf{\Psi\dylRevision\beta} \)
\end{description}

By Corollary \ref{col:s_dynamics} we obtain the following result.
\begin{proposition}\label{prop:DLDPX}
	Let $ \dylRevision $ be a \dynamiclimited\ revision operator compatible with $ \Psi\mapsto(\preceq_{\Psi},\scope{\Psi}) $. Then the following statements hold:
	\begin{itemize}
		\item \( \dylRevision  \) satisfies \eqref{pstl:DLDP1}
		if and only if for \(  \omega_1,\omega_2 {\in} \modelsOf{\alpha} \) with \(\omega_1,\omega_2 {\in} \scope{\Psi}{\cap} \scope{\Psi} \) we have \( \omega_1 \preceq_{\Psi} \omega_2 {\Leftrightarrow} \omega_1 \preceq_{\Psi\dylRevision\alpha} \omega_2 \).
		\item \( \dylRevision  \) satisfies \eqref{pstl:DLDP2}
		if and only if for \(  \omega_1,\omega_2 {\in} \modelsOf{\negOf{\alpha}} \) with \(\omega_1,\omega_2 {\in} \scope{\Psi}{\cap} \scope{\Psi} \) we have 
			\( \omega_1 \preceq_{\Psi} \omega_2 {\Leftrightarrow} \omega_1 \preceq_{\Psi\dylRevision\alpha} \omega_2 \).
	\end{itemize}
\end{proposition}

For the postulate \eqref{pstl:CLP} we obtain a result similar to Proposition \ref{prop:CLDPX}.
\begin{proposition}\label{prop:clp_dyl}
	Let $ \dylRevision $ be a \dynamiclimited\ revision operator compatible with $ \Psi\mapsto(\preceq_{\Psi},\scope{\Psi}) $. Then  \( \dylRevision  \) satisfies \eqref{pstl:CLP} if and only if the following conditions holds for all \( \omega_1,\omega_2 \) with \( \varphi_{\omega_1}\in\beliefsOf{\Psi\dylRevision\varphi_{\omega_1}} \) and \( \varphi_{\omega_2}\in\beliefsOf{\Psi\dylRevision\varphi_{\omega_2}} \):
\begin{enumerate}[(i)]
	\item \(  \ksIF \omega_1\models\alpha \ksAND \omega_2\not\models\alpha \ksAND \omega_1,\omega_2{\in}\scope{\Psi}{\cap}\scope{\Psi\dylRevision\alpha}  \), then
	\begin{equation*}
		\omega_1 \preceq_{\Psi} \omega_2 \Rightarrow \omega_1 \prec_{\Psi\dylRevision\alpha} \omega_2
	\end{equation*}
	\item \( \ksIF \omega_1\models\alpha \ksAND \omega_2\not\models\alpha \ksAND  \omega_1 \preceq_{\Psi} \omega_2 \), then:
	\begin{equation*}
		\omega_2 \in \scope{\Psi\dylRevision\alpha} \Rightarrow \omega_1 \in \scope{\Psi\dylRevision\alpha}
	\end{equation*}
	\item \( \ksIF \omega\not\models\alpha \ksAND \Psi\not\models\negOf{\alpha} \), then \( \omega\in\scope{\Psi\dylRevision\alpha} \Rightarrow \omega\in\scope{\Psi} \)
	\item \(  \ksIF \omega_1{\models}\alpha \ksAND \omega_2{\not\models}\alpha \ksAND \omega_1{\in}\scope{\Psi} \ksAND  \omega_2{\in}\scope{\Psi\dylRevision\alpha} \), then
	\begin{equation*}
		(\omega_2\notin\scope{\Psi} \ksOR \omega_1\preceq_{\Psi} \omega_2 ) \Rightarrow \omega_1\in\scope{\Psi\dylRevision\alpha}
	\end{equation*}
\end{enumerate}
\end{proposition}
\ifshowproofs
\begin{proof}
Analogue to the proof of Proposition \ref{prop:dylr_dp3}.
\end{proof} \fi

	\begin{proposition}
		Let $ \dylRevision $ be a \dynamiclimited\ revision operator compatible with $ \Psi\mapsto(\preceq_{\Psi},\scope{\Psi}) $. Then the following equivalences hold:
		\begin{itemize}
		\item  \( \dylRevision  \) satisfies \eqref{pstl:CLCD} if and only if:
		\begin{multline*}
			\ksIF\beta\models\negOf{\alpha} \ksAND  \Psi\dylRevision\alpha\models\alpha \ksAND   \Psi\dylRevision\beta\not\models\beta  ,\\ \text{then } \modelsOf{\beta}\cap \scope{\Psi\dylRevision\alpha} = \emptyset
		\end{multline*}
		\item  \( \dylRevision  \) satisfies \eqref{pstl:CM1} if and only if:
		\begin{multline*}
			\ksIF \beta\models\alpha \ksAND  \Psi\dylRevision\beta\models\beta \ksOR \modelsOf{\beta}\cap\scope{\Psi}\neq\emptyset,\\
			\text{then } \Psi\dylRevision\alpha\models\beta \ksOR \modelsOf{\beta}\cap\scope{\Psi\dylRevision\alpha}\neq\emptyset
		\end{multline*}
		\item  \( \dylRevision  \) satisfies \eqref{pstl:CM2} if and only if: 
		\begin{multline*}
			\ksIF \beta\models\negOf{\alpha} \ksAND  \Psi\dylRevision\alpha\models\alpha \ksAND   \Psi\dylRevision\beta\models\beta ,\\
			\text{then } \Psi\dylRevision\alpha\models\beta \ksOR \modelsOf{\beta}\cap\scope{\Psi\dylRevision\alpha}\neq\emptyset
		\end{multline*}
	\end{itemize}
\end{proposition}

As a consequence of  Proposition \ref{prop:scope_by_s} and  Proposition \ref{prop:monotone_scope}, we obtain the following result:
\begin{proposition}
	Let $ \dylRevision $ be a \dynamiclimited\ revision operator compatible with $ \Psi\mapsto(\preceq_{\Psi},\scope{\Psi}) $. Then the following equivalences holds:
	\begin{itemize}
		\item \( \dylRevision  \) satisfies \eqref{pstl:FC} if and only if for each \( \alpha \) with \( \Psi\dylRevision\alpha\not\models\alpha \) holds \eqref{pstl:SI1} \ksAND \eqref{pstl:SI2}.
		\item \( \dylRevision  \) satisfies \eqref{pstl:FR} if and only if for each \( \alpha \) with \( \Psi\dylRevision\alpha\not\models\alpha \) holds \eqref{pstl:SD1} \ksAND \eqref{pstl:SD2}.
		\item \( \dylRevision  \) satisfies \eqref{pstl:SC} if and only if for each \( \alpha \) with \( \Psi\dylRevision\alpha\models\alpha \) holds \eqref{pstl:SI1} \ksAND \eqref{pstl:SI2}.
		\item \( \dylRevision  \) satisfies \eqref{pstl:SR} if and only if for each \( \alpha \) with \( \Psi\dylRevision\alpha\models\alpha \) holds \eqref{pstl:SD1} \ksAND \eqref{pstl:SD2}.
	\end{itemize}
\end{proposition}
		
		\begin{proposition}
			Let $ \dylRevision $ be a \dynamiclimited\ revision operator compatible with $ \Psi\mapsto(\preceq_{\Psi},\scope{\Psi}) $. Then the following equivalences holds:
			\begin{itemize}
\item \( \dylRevision  \) satisfies \eqref{pstl:DOC} if and only if the following two statements hold:
\begin{itemize}
	\item \( \ksIF \modelsOf{\!\alpha\!} {\cap} \scope{\Psi} \neq\emptyset \ksAND \beta{\models}\negOf{\alpha} \ksTHEN \modelsOf{\beta} {\cap} \scope{\Psi} {=} \emptyset \)
	\item \( \ksIF \Psi\models\alpha \ksAND \beta{\models}\negOf{\alpha} \ksTHEN \modelsOf{\beta} {\cap} \scope{\Psi} {=} \emptyset \)
\end{itemize}
\item \( \dylRevision  \) satisfies \eqref{pstl:COM} if and only if the following holds:
\begin{equation*}
	\ksIF \modelsOf{\alpha} \cap \scope{\Psi} = \emptyset \ksAND \Psi\not\models\alpha ,
	\text{then } \modelsOf{\alpha} \cap \scope{\Psi} \neq \emptyset 
\end{equation*}
	\end{itemize}
\end{proposition}

\section{Inherent and Immanent Beliefs}
\label{sec:inherent_immanent_beliefs}

As an application for \dynamiclimited\ revision operators, 
	we define the concept of \immanent\ beliefs, which are always credible and fundamentally plausible for the agent when revising.
Therefore we assume that these
immanent beliefs are rooted in fundamental building-blocks of the mind when
forming a new belief set. We call these building-blocks \enquote*{inherent beliefs}.
It seems plausible 
that a (non-prioritised) belief revision operator encodes the attitude towards these beliefs by accepting them regardless of the epistemic state.

\begin{definition}[inherence, immanence]\label{def:inherence}
Let $ \circ $ be a belief change operator. A belief $ \alpha $ is called \emph{\inherent} (for $ \circ $) if $ \Cn(\alpha)=\beliefsOf{\Psi\circ\alpha} $ for all $ \Psi\in\setAllES $. 
A belief $ \alpha $ is called \emph{\immanent} (for $ \circ $) if $ \alpha\equiv\alpha_1\lor\ldots\lor\alpha_k $ such that $ \alpha_1,\ldots,\alpha_k $ are \inherent\ for $ \circ $.
 A belief set $ X $ is called \emph{\immanent} (for $ \circ $) if $ X=\Cn(\alpha) $ such that $ \alpha $ is \immanent\ for $ \circ $.
\end{definition}

Intuitively, \inherent\  beliefs are such beliefs that get accepted  with all their consequences regardless of the current epistemic state.
The notion of \immanent\ beliefs covers beliefs that are composed of \inherent\  beliefs.%
\begin{example}
	Consider $ \Sigma=\{e,f\} $, where $ e $ has the intended meaning \enquote*{there is extraterrestrial life} and $ f $ has the intended meaning \enquote*{it is a fictive story}. 
	An agent might deny the possibility of extraterrestrial life, but she is open to fictitious stories.
	One might expect that $ e\land f $ and $ \neg e\land f $ may be \inherent\ beliefs, $ e \to f  $ is an \immanent\ belief and $ \neg e \to f  $ is no \immanent\ belief.
\end{example}

\ksSubSectionStar{Relation to Belief Change Operators.}
For AGM revisions, the \inherent\ beliefs are exactly the worlds, and thus for AGM revision operators there are no non-\immanent\ beliefs. 
\begin{proposition}\label{prop:inherence_agm}
Let \( \setAllES \) be \ref{pstl:unbiased}. For every AGM revision operator $ * $  every consistent belief is \immanent, and a belief $ \alpha\in\propLang $ is \inherent\ if and only if it has exactly one model.
\end{proposition}
\ifshowproofs
\begin{proof}
	For every consistent belief set $ L $ there exists at least one state $ \Psi $ having this belief set, i.e., $ \beliefsOf{\Psi}=L $.
	For $ \omega\in\Omega $, regardless of the state, we have $ \modelsOf{\Psi*\varphi_{\omega}}=\{\omega\} $ due to Proposition \ref{prop:es_revision}.
	Thus, every $ \varphi_{\omega} $ is an \inherent\ belief with exactly one model.
Assume now that $ \alpha\in\propLang $ has two or more models. Let \( \omega \in \modelsOf{\alpha} \).
	Then there is some belief set \( \modelsOf{L} \) with \( \modelsOf{L}=\modelsOf{\alpha} \setminus\{\alpha\} \).
	Because \( L \) is consistent and \( \setAllES \) \ref{pstl:unbiased}, there is some state \( \Psi_L\in \setAllES \) with \( \beliefsOf{\Psi_L}= L \). By Proposition \ref{prop:es_revision}, we obtain $ \modelsOf{\Psi \revision \alpha}=\modelsOf{L} $. Thus, \( \alpha \) is no \inherent\ belief.
	As $ \varphi_{\omega} $ for every $ \omega\in\Omega $  is \inherent, every consistent belief is \immanent. \qedhere
\end{proof} \fi

We obtain the following result about \inherent\ and \immanent\ beliefs for credibility-revision operators.
\begin{proposition}\label{prop:inherent_clr}
	Let $ \nrRevision $ be a credibility-limited revision operator and $ \Psi\mapsto(\leq_\Psi,C_\Psi) $ be a corresponding CLF-assignment. A belief $ \alpha $ is \inherent\ for $ \nrRevision $ if and only if $ \alpha $ has exactly one model $ \omega $ and $ \omega\in C_\Psi $ for every $ \Psi\in\setAllES $.
	Therefore, there is a credibility-limited revision operator with no \immanent\  and no \inherent\ beliefs.
\end{proposition}
\ifshowproofs
\begin{proof}
	If $ \alpha $ has more than one model, then by \eqref{pstl:LR2} it is no \inherent\ belief when coosing $ \Psi $ with $ \modelsOf{\Psi}\subsetneq \modelsOf{\alpha} $.
	Likewise by Equation \eqref{eq:cl_revision} and \eqref{pstl:LR2} the condition $ \modelsOf{\alpha} \subseteq C_\Psi $ for all $ \Psi\in\setAllES $ is easy to see.	
	
	For the last statement, choose a  CLF-assignment $ \Psi\mapsto(\leq_\Psi,C_\Psi) $ with $ \modelsOf{\Psi}=C_\Psi $ for every epistemic state $ \Psi $. \qedhere
\end{proof} \fi

For \dynamiclimited\ revision operators, inherent beliefs are limited by the domain of every epistemic state.
\begin{proposition}\label{lem:inherent_full_dynamic}
	Let \( \setAllES \) be \ref{pstl:unbiased}, and $ \Psi \mapsto (\preceq_{\Psi},\scope{\Psi}) $ a limited assignment compatible with a \dynamiclimited\ revision operator $ \dylRevision $.
	If $ \alpha $ is \inherent, then either  $ \modelsOf{\alpha} \subseteq \scope{\Psi} $ or \( \scope{\Psi} \cap \modelsOf{\alpha}=\emptyset  \). Moreover, in the latter case \( \modelsOf{\Psi}=\modelsOf{\alpha} \).
	\end{proposition}
\ifshowproofs
\begin{proof}
	Let $ \alpha $ be an \inherent\ belief of $ \circ $ and let $ \Psi\in\setAllES $ an arbitrary epistemic state, with $ \scope{\Psi}=\{\omega_1,\ldots,\omega_k\} $.
	Assume \( \omega \) such that \( \omega\models\alpha \), but \( \omega\notin \scope{\Psi} \). 
	If \( \scope{\Psi}\cap\modelsOf{\alpha} \neq \emptyset \), then by Definition \ref{def:dynamic_limited_revision} we obtain \( \modelsOf{\Psi\dylRevision\alpha} \subsetneq \modelsOf{\alpha} \). 
	This is a contradiction, because \( \alpha \) is \inherent. Consequently, we obtain \( \scope{\Psi}\cap\modelsOf{\alpha} = \emptyset \). Moreover, because \( \alpha \) is \inherent, we obtain \( \modelsOf{\Psi}=\modelsOf{\alpha} \).\qedhere

\end{proof} \fi
 
\ksSubSectionStar{\mbox{Inherence-Limited} Revision Operators}
\label{sec:legitimation}
We design and characterise a new class of belief change operators, called \ihlimited\ revision operators, which yield either an \immanent\ belief set or does not change the belief set.
We start by defining the  \ihlimited\ revision operators by postulates about their belief dynamics.

\begin{definition}%
	A belief change operator $ \nrRevision $ is an \emph{\ihlimited\ revision operator} if $ \nrRevision $ satisfies:
	\begin{description}		
		\item[\normalfont(\textlabel{IL1}{pstl:IL1})] $ \beliefsOf{\Psi\nrRevision\alpha} = \beliefsOf{\Psi}  \ksOR \alpha\in \beliefsOf{\Psi\nrRevision\alpha}  $
		\item[\normalfont(\textlabel{IL2}{pstl:IL2})] $ \beliefsOf{\Psi\nrRevision\alpha} = \beliefsOf{\Psi} $ or $ \beliefsOf{\Psi\nrRevision\alpha} $ is \immanent  
		\item[\normalfont(\textlabel{IL3}{pstl:IL3})] if $ \beliefsOf{\Psi}\cup\{\alpha\} $ is consistent and $ \alpha  $ \immanent,
		\\\mbox{}	\hfill 
		then $ \beliefsOf{\Psi\nrRevision\alpha}\cup\{\alpha\} = \Cn(\beliefsOf{\Psi}\cup\{\alpha\}) $
\item[\normalfont(\textlabel{IL4}{pstl:IL4})] if $ \alpha\models\beta $ and $ \alpha $ is \immanent,
then $ \beliefsOf{\Psi\nrRevision\beta} $ is \immanent.
\item[\normalfont(\textlabel{IL5}{pstl:IL5})] $ \beliefsOf{\Psi\!\nrRevision\!\alpha} $ is inconsistent only if $ \beliefsOf{\!\Psi\!} $ is inconsistent
\item[\normalfont(\textlabel{IL6}{pstl:IL6})] if $ \alpha\equiv\beta $, then $ \beliefsOf{\Psi\nrRevision\alpha} = \beliefsOf{\Psi\nrRevision\beta} $
\item[\normalfont(\textlabel{IL7}{pstl:IL8})] $ \beliefsOf{\Psi\!\nrRevision\!(\alpha\lor\beta)}\! =\!\begin{cases}
	\beliefsOf{\Psi\nrRevision\alpha} \text{ or} \\
	\beliefsOf{\Psi\nrRevision\beta} \text{ or} \\
	\beliefsOf{\Psi\nrRevision\alpha} \cap \beliefsOf{\Psi\nrRevision\beta}
\end{cases}  $
	\end{description}
\end{definition}

Clearly, the postulates are inspired by those for \dynamiclimited\ revision operators. 
Thus, for their explanation we refer to Section \ref{sec:dyl_representation_theorem}.

Every model of an \immanent\ belief of an \ihlimited\ revision operator is \inherent. %
\begin{lemma}\label{lem:inherent_inherence_limited}
	Let $ \nrRevision $ be an \ihlimited\ revision operator. 
	If $ \alpha $ is \immanent\ for $ \nrRevision $, then $ \varphi_\omega{\models}\alpha $ is \inherent\ for $ \nrRevision $.
\end{lemma}
\ifshowproofs
\begin{proof}
	Let $ \alpha $ and $ \varphi_\omega $ be as above. Let $ \Psi $ be such that $ \modelsOf{\Psi}=\{\omega\} $. By \eqref{pstl:IL3}, we have $ \modelsOf{\Psi\nrRevision\alpha}=\{\omega\} $. From \eqref{pstl:IL4} we obtain that $ \Cn(\varphi_{\omega})  $ is an \inherent\ belief set. This implies that $ \varphi_{\omega} $ is \inherent. \qedhere
\end{proof} \fi

From the postulates of \ihlimited\ revision operators we obtain the following result.
\begin{lemma}\label{col:scope_immanence}
	Let $ \nrRevision $ be an \ihlimited\ revision operator. 
	If \( \alpha \in \Scope{\nrRevision}{\Psi} \setminus \beliefsOf{\Psi} \), then \( \alpha \) is \immanent.
\end{lemma}

Lemma \ref{lem:inherent_inherence_limited}, 
Lemma \ref{col:scope_immanence} and Definition \ref{def:inherence}
gives rise to the following representation theorem for \ihlimited\ revision operators.

\begin{theorem}\label{thm:inherence_limited_operator}
	A belief change operator $ \nrRevision $ is an \ihlimited\ revision operator
if and only if \( \nrRevision \) is a \dynamiclimited\ revision operator compatible with $ \Psi \mapsto {(\preceq_\Psi,\scope{\Psi})} $ such that there is \( \Omega'\subseteq \Omega \)  with \( \scope{\Psi} = \Omega' \) for every epistemic states $ \Psi $.
\end{theorem}
\ifshowproofs
\begin{proof}
\noindent\textit{The \enquote*{$ \Rightarrow $}-direction.} 
Let $ \nrRevision $ be an operator satisfying the postulates \eqref{pstl:IL1}--\eqref{pstl:IL8}. 
For $ \Psi $, we construct $ (\preceq_{\Psi},\scope{\Psi}) $. In particular, let $ \preceq_{\Psi} $ be the relation with
\(	\omega_1 \preceq_\Psi \omega_2 \ksIFF \omega_1 \in \modelsOf{\Psi\nrRevision (\varphi_{\omega_1}\lor\varphi_{\omega_2})}\)
and $ \scope{\Psi}=\dom(\preceq_\Psi)= \{ \omega \mid \varphi_\omega \text{ is \inherent\ for } \nrRevision \} $.
By the definition of inherence, $ \scope{\Psi} $ is the same for each $ \Psi $, i.e. for every $ \Psi,\Gamma\in\setAllES $ we have $ \omega\in{\scope{\Psi}} $ if and only if $ \omega\in{\scope{\Psi}} $. 
The order $ \preceq_{\Psi} $ is a total preorder:

	\emph{Totality/reflexivity.} Let $ \omega_1,\omega_2\in\scope{\Psi} $. Then by \eqref{pstl:IL8} we have that $ \modelsOf{\Psi\nrRevision(\varphi_{\omega_1}\lor\varphi_{\omega_2})} $ is equivalent to $ \{ \omega_1 \} $ or $ \{ \omega_2 \} $ or $ \{ \omega_1,\omega_2 \} $. Therefore, the relation must be total. Reflexivity follows from totality.

\emph{Transitivity.} 	Let $ \omega_1,\omega_2,\omega_3\in\scope{\Psi} $ with $ \omega_1 \preceq_{\Psi} \omega_2 $ and $ \omega_2\preceq_{\Psi} \omega_3 $. Towards a contradiction assume that $ \omega_1 \not\preceq_{\Psi} \omega_3 $ holds. This implies  $ \modelsOf{\Psi\nrRevision\varphi_{\omega_1,\omega_3}}=\{\omega_3\} $.
	Now assume that $ \modelsOf{\Psi\nrRevision\varphi_{\omega_1,\omega_2,\omega_3}} = \{\omega_3\} $.
	Then by \eqref{pstl:IL8} we obtain that $ \modelsOf{\Psi\nrRevision\varphi_{\omega_2,\omega_3}}=\{\omega_3\} $, a contradiction to $ \omega_2\preceq_{\Psi}\omega_3 $.
	Assume for the remaining case $ \modelsOf{\Psi\nrRevision\varphi_{\omega_1,\omega_2,\omega_3}} \neq \{\omega_3\} $. 
	We obtain from \eqref{pstl:IL8} that $ \modelsOf{\Psi\nrRevision\varphi_{\omega_1,\omega_2,\omega_3}} $ equals either $ \modelsOf{\Psi\nrRevision\varphi_{\omega_1,\omega_2}} $ or $ \modelsOf{\Psi\nrRevision\varphi_{\omega_3}}$. 
	Since the second case is impossible, $ \omega_1\preceq_{\Psi} \omega_2 $ implies $ \omega_1\in \modelsOf{\Psi\nrRevision\varphi_{\omega_1,\omega_2}} $. Now apply \eqref{pstl:IL8} again to $ \modelsOf{\Psi\nrRevision\varphi_{\omega_1,\omega_2,\omega_3}} $ and obtain that it is either equivalent to $ \modelsOf{\Psi\nrRevision\varphi_{\omega_1,\omega_3}} $ or $ \modelsOf{\Psi\nrRevision\varphi_{\omega_2}} $.
	In both cases, we obtain a contradiction because $ \omega_1\in \modelsOf{\Psi\nrRevision\varphi_{\omega_1,\omega_2,\omega_3}} $.

The construction yields a faithful limited assignment:\\
We show $ \modelsOf{\Psi}\cap\scope{\Psi}=\min(\Omega,\preceq_{\Psi}) $. 
Let $ \omega_1,\omega_2\in\scope{\Psi} $ and $ \omega_1,\omega_2\in\modelsOf{\Psi} $. By definition of $ {\scope{\Psi}} $ the interpretations $ \omega_1,\omega_2 $ are \inherent. From \eqref{pstl:IL3} we obtain $ \modelsOf{\Psi\nrRevision(\varphi_{\omega_1}\lor\varphi_{\omega_2})}=\{ \omega_1,\omega_2 \} $ which yields by definition $ \omega_1 \preceq_{\Psi} \omega_2 $ and $ \omega_2 \preceq_{\Psi} \omega_1 $.
Let $ \omega_1,\omega_2\in\scope{\Psi} $ with  $ \omega_1\in\modelsOf{\Psi} $ and $ \omega_2\notin\modelsOf{\Psi} $.
Then $ \varphi_{\omega_1}\lor\varphi_{\omega_2} $ is consistent with $ \beliefsOf{\Psi} $. 
Therefore, by 
\eqref{pstl:IL3} and \eqref{pstl:IL5} we have $ \modelsOf{\Psi\nrRevision (\varphi_{\omega_1}\lor\varphi_{\omega_2})} = \modelsOf{\Psi}\cap\modelsOf{\varphi_{\omega_1}\lor\varphi_{\omega_2}}=\{ \omega_1 \} $.
Together we obtain faithfulness.

We show the satisfaction of \eqref{eq:limited_revision} in two case.

For the first case, assume $ \modelsOf{\alpha}\cap\scope{\Psi}=\emptyset $.
By the postulate  \eqref{pstl:IL2} we have either $ \beliefsOf{\Psi}=\beliefsOf{\Psi\nrRevision\alpha} $ or $ \beliefsOf{\Psi\nrRevision\alpha} $ is \immanent. In the first case we are done. 
For the second case, by the postulate \eqref{pstl:IL1}, we obtain $ \modelsOf{\Psi\nrRevision\alpha}\subseteq\modelsOf{\alpha} $. 
Thus, by Lemma \ref{lem:inherent_inherence_limited} the set  $ \modelsOf{\alpha }$ contains an interpretation $ \omega $ such that $ \varphi_{\omega} $ is \inherent, a contradiction to $ \modelsOf{\alpha}\cap\scope{\Psi}=\emptyset $.

For the second case assume $ \modelsOf{\alpha}\cap\scope{\Psi}\neq\emptyset $.
We show the equivalence $ \modelsOf{\Psi\nrRevision\alpha}=\min(\modelsOf{\alpha},\preceq_{\Psi}) $ by showing both set inclusions separately.

	We show $ \min(\modelsOf{\alpha},\preceq_{\Psi})\subseteq \modelsOf{\Psi\nrRevision\alpha} $.
Let $ \omega\in\min(\modelsOf{\alpha},\preceq_{\Psi}) $ with $ \omega\notin\modelsOf{\Psi\nrRevision\alpha}$. 
By construction of \( \scope{\Psi} \) and \eqref{pstl:IL3} we have $ \omega\in\modelsOf{\Psi\nrRevision\varphi_{\omega} }$. 
From Lemma \ref{lem:inherent_full_dynamic} and Lemma \ref{lem:inherent_inherence_limited} and postulate \eqref{pstl:IL4}, we obtain that $\beliefsOf{\Psi\nrRevision\alpha} $ is \immanent\ and every $ \omega\in\modelsOf{\Psi\nrRevision\alpha} $ is \inherent. 
Hence, there is at least one $ \omega'\in\beliefsOf{\Psi\nrRevision\alpha}  $ with $ \omega'\in\scope{\Psi} $.
If $ \beliefsOf{\Psi\nrRevision\alpha}=\beliefsOf{\Psi} $ we obtain by the faithfulness $ \min(\modelsOf{\alpha},\preceq_{\Psi}) = \modelsOf{\Psi\nrRevision\alpha} $. If $ \beliefsOf{\Psi\nrRevision\alpha}\neq\beliefsOf{\Psi} $, then by \eqref{pstl:IL1} we have $ \modelsOf{\Psi\nrRevision\alpha}\subseteq\modelsOf{\alpha} $.
Let $ \beta=\varphi_{\omega}\lor\varphi_{\omega'} $ and $ \gamma=\beta\lor\gamma' $ such that $ \modelsOf{\gamma}=\modelsOf{\alpha} $ and $ \modelsOf{\gamma'}=\modelsOf{\alpha}\setminus\{\omega,\omega'\} $.
By \eqref{pstl:IL8} we have either $ \modelsOf{\Psi\nrRevision\alpha}=\modelsOf{\Psi\nrRevision\beta} $ or $ \modelsOf{\Psi\nrRevision\alpha}=\modelsOf{\Psi\nrRevision\gamma'} $ or $ \modelsOf{\Psi\nrRevision\alpha}=\modelsOf{\Psi\nrRevision\beta}\cup\modelsOf{\Psi\nrRevision\gamma'} $.
The first and the third case are impossible, because $ \omega\notin\modelsOf{\Psi\nrRevision\alpha} $ and by the minimality of $ \omega $ we have $ \omega \in \modelsOf{\Psi\nrRevision\beta} $.
It remains the case of $ \modelsOf{\Psi\nrRevision\alpha}=\modelsOf{\Psi\nrRevision\gamma'} $. 
Let $ \omega_{\gamma'} $ such that $ \omega_{\gamma'}\in\modelsOf{\Psi\nrRevision\alpha} $. Note that $ \omega_{\gamma'}\in\modelsOf{\gamma'}\subseteq\modelsOf{\alpha} $ and $ \omega_{\gamma'}\in\scope{\Psi} $.
Now let $ \delta=\beta'\lor\delta' $ with $ \delta\equiv\alpha $ such that $ \modelsOf{\beta'}=\{\omega,\omega_{\gamma'} \} $ and $ \modelsOf{\delta'}=\modelsOf{\alpha}\setminus\{ \omega,\omega_{\gamma'} \} $. 
By minimality of $ \omega $ we have $ \omega\in\modelsOf{\Psi\nrRevision\beta'} $.
By \eqref{pstl:IL8} we obtain $ \modelsOf{\Psi\nrRevision\alpha} $ is either equivalent to $ \modelsOf{\Psi\nrRevision\beta'} $ or to $ \modelsOf{\Psi\nrRevision\delta'} $ or to $  \modelsOf{\Psi\nrRevision\beta'} \cup \modelsOf{\Psi\nrRevision\delta'} $. The first and third case are impossible since $ \omega\notin \modelsOf{\Psi\nrRevision\alpha} $. Moreover, the second case is also impossible, because of $ \omega_{\gamma'}\notin \modelsOf{\Psi\nrRevision\delta'} $.

	We show $ \modelsOf{\Psi\nrRevision\alpha} \subseteq \min(\modelsOf{\alpha},\preceq_{\Psi}) $.
	Let  $ \omega\in\modelsOf{\Psi\nrRevision\alpha} $ with  $ \omega\notin \min(\modelsOf{\alpha},\preceq_{\Psi})  $.
	From non-emptiness of $ \min(\modelsOf{\alpha},\preceq_{\Psi}) $ and $ \min(\modelsOf{\alpha},\preceq_{\Psi})\subseteq \modelsOf{\Psi\nrRevision\alpha} $ we obtain an \inherent\ belief $ \varphi_{\omega'} $, where $ \omega'\in\modelsOf{\Psi\nrRevision\alpha} $ such that $ \omega'\in \min(\modelsOf{\alpha},\preceq_{\Psi}) $.

Assume that $ \varphi_\omega $ is not \inherent, and therefore, $ {\omega\notin\scope{\Psi}} $.
By the postulate \eqref{pstl:IL4} we obtain from the existence of $ \omega' $ that $ \beliefsOf{\Psi\nrRevision\alpha} $ is \immanent.  Therefore, by Lemma \ref{lem:inherent_inherence_limited} every model in $ \modelsOf{\Psi\nrRevision \alpha} $ is \inherent, a contradiction, and  therefore, $ \varphi_\omega $ has to be \inherent.

From inherence of $ \varphi_\omega $ we obtain $ \modelsOf{\Psi\nrRevision\varphi_{\omega}}=\{\omega\} $.
Using the postulate \eqref{pstl:IL4} we obtain that $ \beliefsOf{\Psi\nrRevision\alpha} $ is \immanent.
Since $ \omega $ is not minimal, we have $ \omega'\in\modelsOf{\Psi\nrRevision(\varphi_{\omega,\omega'})} $ and $ \omega\notin\modelsOf{\Psi\nrRevision(\varphi_{\omega,\omega'})} $
Now let $ \gamma=\varphi_{\omega,\omega'}\lor\gamma' $ with $ \modelsOf{\gamma'}=\modelsOf{\alpha}\setminus\{\omega,\omega'\} $.
	By \eqref{pstl:IL8} we have either $ \modelsOf{\Psi\nrRevision\alpha}=\modelsOf{\Psi\nrRevision\beta} $ or $ \modelsOf{\Psi\nrRevision\alpha}=\modelsOf{\Psi\nrRevision\gamma'} $ or $ \modelsOf{\Psi\nrRevision\alpha}=\modelsOf{\Psi\nrRevision\beta}\cup\modelsOf{\Psi\nrRevision\gamma'} $.
	All cases are impossible because for every case we obtain $ \omega\notin\modelsOf{\Psi\nrRevision\alpha} $.
This shows $ \modelsOf{\Psi\nrRevision\alpha} \subseteq \min(\modelsOf{\alpha},\preceq_{\Psi}) $, and, in summary, we obtain $ \modelsOf{\Psi\nrRevision\alpha} = \min(\modelsOf{\alpha},\preceq_{\Psi}) $ and thus, \eqref{eq:limited_revision} holds.

\smallskip
\noindent\textit{The \enquote*{$ \Leftarrow $}-direction.}
Let $ \Psi \mapsto (\preceq_\Psi,\scope{\Psi}) $ be a limited assignment compatible with \( \nrRevision \) such that \( \scope{\Psi}=\scope{\Gamma} \) for all \( \Psi,\Gamma\in\setAllES \).
Furthermore, let $ \alpha $ be an \inherent\  belief of $ \nrRevision $. Then by Lemma \ref{lem:inherent_full_dynamic} every model $ \omega $ of $ \alpha $ is an element of $ \scope{\Psi} $. Moreover, every $ \varphi_\omega $ with $ \omega\in\scope{\Psi} $ is an \inherent\ belief.

We show the satisfaction of \eqref{pstl:IL1}--\eqref{pstl:IL8}.
From \eqref{eq:limited_revision} we obtain straightforwardly \eqref{pstl:IL1}, \eqref{pstl:IL2}, \eqref{pstl:IL5} and \eqref{pstl:IL6}.
\begin{description}
	\item[\eqref{pstl:IL3}] Assume $ \alpha $ to be \immanent\ and $ \modelsOf{\Psi}\cap\modelsOf{\alpha}\neq\emptyset $.
	By Lemma \ref{lem:inherent_full_dynamic} we have $ \modelsOf{\Psi}\cap\modelsOf{\alpha} \subseteq \scope{\Psi} $.
	Thus, $ \min(\modelsOf{\alpha},\preceq_{\Psi})\neq\emptyset $. From faithfulness we obtain $ \min(\modelsOf{\alpha},\preceq_{\Psi})=\modelsOf{\Psi}\cap\modelsOf{\alpha} $.
	\item[\eqref{pstl:IL4}] Let $ \alpha $ be immanent and thus by Lemma \ref{lem:inherent_full_dynamic} we have $ \modelsOf{\Psi}\cap\modelsOf{\alpha} \subseteq \scope{\Psi} $.
	This implies $ \alpha\models\beta $ and we obtain $ \mid(\modelsOf{\beta},\preceq_{\Psi})\neq\emptyset $. Then by \eqref{eq:limited_revision} we obtain $ \beta\in \beliefsOf{\Psi\nrRevision} $.
	\item[\eqref{pstl:IL8}] 
	Suppose $ \modelsOf{\alpha\lor\beta}\cap\scope{\Psi}=\emptyset $, then by \eqref{eq:limited_revision} for every formula $ \gamma $ with $ \modelsOf{\gamma} \subseteq \modelsOf{\alpha\lor\beta} $ we obtain $ \modelsOf{\Psi\nrRevision(\alpha\lor\beta)}=\modelsOf{\Psi}=\modelsOf{\Psi\nrRevision\gamma} $.
	
	Assume that $ \modelsOf{\alpha}\cap\scope{\Psi}\neq\emptyset $ and $ \modelsOf{\beta}\cap\scope{\Psi}=\emptyset $.
	Then by the postulate \eqref{eq:limited_revision} we obtain $ \modelsOf{\Psi\nrRevision(\alpha\lor\beta)}=\modelsOf{\Psi\nrRevision\alpha} $.
	The case of $ \modelsOf{\alpha}\cap\scope{\Psi}=\emptyset $ and $ \modelsOf{\beta}\cap\scope{\Psi}\neq\emptyset $ is analogue.

	Assume that $ \modelsOf{\alpha}\cap\scope{\Psi}\neq\emptyset $ and $ \modelsOf{\beta}\cap\scope{\Psi}\neq\emptyset $.
	Then we have $ \modelsOf{\Psi\nrRevision(\alpha\lor\beta)}={\min(\modelsOf{\alpha\lor\beta},\preceq_{\Psi})} $. Using logical equivalence, we obtain $ \min(\modelsOf{\alpha\lor\beta},\preceq_{\Psi})=\min(\modelsOf{\alpha}\cup\modelsOf{\beta},\preceq_{\Psi}) $.
	By Lemma \ref{lem:sem_trichotonomy} we obtain directly \eqref{pstl:IL8}. \qedhere
\end{description}
\end{proof} \fi

We give an example for \ihlimited\ revision operators.
\begin{example}
Consider $ \nrRevision $ and the epistemic states $ \Psi_\mathrm{il}^1 $ and $ \Psi_\mathrm{il}^2 $ which satisfy the conditions from  Figure \ref{fig:ihl_example}.
While the change from $ \Psi_\mathrm{il}^1 $ to $ \Psi_\mathrm{il}^1 \nrRevision (a\land b) $ suggests that $ a\land b $ is an inherent or \immanent\ belief, the change from $ \Psi_\mathrm{il}^2  $ to $ \Psi_\mathrm{il}^2\nrRevision (a\land b) $ proves this wrong, because $ a\land b\notin\beliefsOf{\Psi_\mathrm{il}^2 \nrRevision (a\land b)} $. 
This is because $ a\land b $ has no \inherent\  model.
Thus, when invoking $ \nrRevision $ by $ a\land b $, in both states $ \Psi_\mathrm{il}^1$ and $\Psi_\mathrm{il}^2  $ the belief set is kept.
However, the belief $ a\land \neg b $ could be \inherent.
Consider the change from $ \Psi_\mathrm{il}^1 $ to $ \Psi_\mathrm{il}^1\nrRevision a $. 
Here $ a $ and the belief base of $ \Psi_\mathrm{il}^1 $  are not inconsistent, i.e. $ \modelsOf{a}\cap \modelsOf{\Psi_\mathrm{il}^1}= \{ ab \}\neq \emptyset $. 
We obtain $ \modelsOf{\Psi_\mathrm{il}^1 \nrRevision a} = \{ a\ol{b} \} = \min(\modelsOf{a},\preceq_{\Psi_\mathrm{il}^1})  $, since $ a\land b $ is no \immanent\  belief and therefore \eqref{pstl:IL3} is not applicable.
In all these changes the set \( \scope{\Psi} \) remains the same, respectively the domains of the assigned total preorder do not alter.
\end{example}

\begin{figure}
			\centering
	\resizebox{\columnwidth}{!}{
		\begin{tabular}{r|ccc|cc}
			\toprule
			$ \Psi $ & $ \Psi_\mathrm{il}^1 $ & $ \Psi_\mathrm{il}^1\!\nrRevision\! a $ & $ \Psi_\mathrm{il}^1\!\nrRevision\! (a\!\land\! b) $ & $ \Psi_\mathrm{il}^2 $ & $ \Psi_\mathrm{il}^2\!\nrRevision\! (a\!\land\!b) $ \\ \midrule
			$ \modelsOf{\Psi} $ &    $  ab,\ol{a}b  $    &             $ a\ol{b} $             &               $  ab,\ol{a}b  $               &   $  \ol{a}\ol{b}  $   &               $ \ol{a}\ol{b} $               \\
			$ \scope{\Psi} $ & $  \ol{a}b,a\ol{b}  $  &         $ \ol{a}b,a\ol{b} $         &            $  \ol{a}b,a\ol{b}  $             & $   \ol{a}b,a\ol{b}  $ &             $  \ol{a}b,a\ol{b} $             \\ \midrule
			\multirow{2}{*}{\rotatebox[origin=c]{0}{$ \preceq_{\Psi} $}} &      $ a\ol{b} $       &             $ \ol{a}b $             &                 $ a\ol{b} $                  &      $ a\ol{b} $       &                $  a\ol{b}  $                 \\
			&       $\ol{a}b$        &             $ a\ol{b} $             &                 $ \ol{a}b $                  &     $  \ol{a}b  $      &                 $ \ol{a}b  $                 \\ \bottomrule
		\end{tabular}
	}
	\caption{An \ihlimited\ revision operator $ \nrRevision $. 
	}\label{fig:ihl_example}
\end{figure}

\begin{figure}\centering
\begin{tikzpicture}[thick]	
	\draw [rounded corners] (0,0) rectangle (6,3.7) ;
	\draw [rounded corners] (0.25,2.75) rectangle (5.75,1.1) ;
	\draw [rounded corners] (0.25,0.25) rectangle (5.75,1.9) ;
	
	\node () [anchor=north] at (3,3.5) {Dynamic-Limited Revision};
	\node () at (3,0.625) {Credibility-Limited Revision};
	\node () at (3,2.375) {Inherence-Limited Revision};
	\node () at (3,1.5) {AGM Revision};
\end{tikzpicture}
\caption{Interrelation between the operator classes.}\label{fig:op_interrelation}
\end{figure}
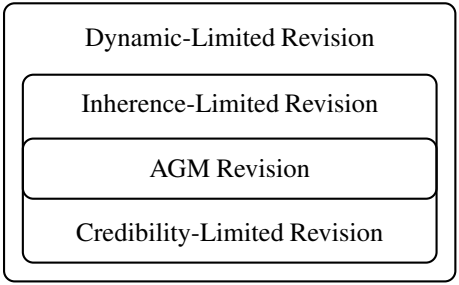

From Proposition \ref{prop:dlr_scope} and Theorem \ref{thm:inherence_limited_operator} we obtain the following result.
\begin{proposition}
	For every \ihlimited\ revision operator \( \nrRevision \) there exist a set \( X \subseteq \propLang \) which satisfies \ref{pstl:singlesentenceclosure} and \ref{pstl:disjunction_completeness} such that \( \Scope{\nrRevision}{\Psi}=\beliefsOf{\Psi}\cup X  \) for all \( \Psi\in\setAllES \).
\end{proposition}

As a last result, we show that the classes of credibility-limited revision operators and \ihlimited\ revision operators are disjoint except for AGM revision operators. Their relation is also illustrated in Figure \ref{fig:op_interrelation}.
\begin{proposition}\label{prop:inhlimited_clr_agm}
	Let \( \setAllES \) be \ref{pstl:unbiased}. A belief change operator $ \circ $ is an \ihlimited\ revision operator and a credibility-limited revision operator at the same time if and only if $ \circ $ is an AGM revision operator.
\end{proposition}
\ifshowproofs
\begin{proof}
	Because every faithful assignment is also a CLF-assignment and a faithful limited assignment, every AGM revision operator is also a credibility-limited revision operator and an \ihlimited\ revision operator.
	
	Let $ \nrRevision $ an \ihlimited\  revision operator, but no AGM revision operator.
	Then there is an faithful limited assignment  $ \Psi\mapsto(\preceq_{\Psi},\Omega') $ with $ \Omega'\subsetneq \Omega $.
	Let $ \Psi_\top $ an epistemic state with $ \modelsOf{\Psi_\top}=\Omega$. 
	Let $ \alpha $ such that there is an $ \omega\models\alpha $ with $ \omega\notin\Omega' $. Then, because of \eqref{pstl:LR2}, for every credibility-limited revision $ \clRevision $ operator $ \omega\in\modelsOf{\Psi_\top\clRevision\alpha} $, but $ \omega\notin\modelsOf{\Psi_\top\nrRevision\alpha} $.

	Let $ \clRevision $ an credibility-limited revision operator, but no AGM revision operator. Then there is a CLF-assignment  $ \Psi\mapsto(\leq_{\Psi},C_\Psi) $ with $ \modelsOf{\Psi}\subseteq C_\Psi\subseteq \Omega $. 
	Moreover, there is at least one $ \Gamma\in\setAllES $ with $ C_\Gamma\subsetneq \Omega $, i.e., there is $ \omega\not\in C_\Gamma $, and therefore, $ \omega\notin\modelsOf{\Gamma} $. 
	Let again denote $ \Psi_\top $ an epistemic state with $ \modelsOf{\Psi_\top}=\Omega$.
	We obtain, because of \eqref{pstl:LR2}, that $ \modelsOf{\Gamma\clRevision\varphi_{\omega}}=\modelsOf{\Gamma} $ and $ \modelsOf{\Psi_\top\clRevision\varphi_{\omega}}=\{\omega\} $. 
	Assume that $ \clRevision $ is also an \ihlimited\ revision operator. 
	From \eqref{pstl:IL2} and $ \modelsOf{\Psi_\top\clRevision\varphi_{\omega}}=\{\omega\} $ we obtain the immanence of $ \varphi_{\omega} $.
	Lemma \ref{lem:inherent_inherence_limited} implies that $ \varphi_{\omega} $ is \inherent. This leads to the contradiction $ \modelsOf{\Gamma\clRevision\varphi_{\omega}}=\{\omega\} $. \qedhere
\end{proof} \fi

\section{Summary and Future Work}
\label{sec:discussion}

In this article we investigated \dynamiclimited\ revision operators, which are operators that maintain a total preorder over an arbitrary set of worlds. We considered the concept of scope, which describes those beliefs which get accepted by a (non-prioritized) revision operator.
Different postulates for dynamics of beliefs and the scope were considered and novel postulates were introduced. 
We gave representation theorems for all the these postulates for \dynamiclimited\ revision operators.

As an application, we introduced \immanent\ and \inherent\ beliefs as global properties of an operator. 
The notion of \immanent\ beliefs provides rationale to the \ihlimited\ revision operators. 
We have presented a  representation theorem for \inherent\ revision operators. 
Credibility-limited revisions and  \ihlimited\ revision operators are elements of the class of \dynamiclimited\ revision operators which are a common generalisation. We showed that \ihlimited\ revision and credibility-limited revision operators are disjoint except for the AGM revision operators.

For future work, we see that further investigations on postulates for the dynamics of the scope could be fruitful.
In particular, the connections between scope dynamics and the results by Booth and Meyer (\citeyear{KS_BoothMeyer2011}) on the dynamics of total preorders are open to explore. Moreover, the connections between \immanent\ beliefs and \emph{core beliefs} by Booth \cite{KS_Booth2002} and the revision thereof should be elaborated.

The \dynamiclimited\ revision operators follow the idea of credibility-limited revisions by keeping the prior belief set when dealing with inputs which have no models in the domain of the corresponding total preorder.
Another approach open to investigate would be to accept the input as it is. %
This would yield operators which behave according to the following scheme (or variants thereof):
\begin{equation*}
	\modelsOf{\Psi * \alpha} \! = \! \begin{cases}
		\min(\modelsOf{\alpha},\preceq_\Psi) & \!\!\!\!\min(\modelsOf{\alpha},\!\preceq_\Psi\!)\! \neq\! \emptyset \\
		\modelsOf{\Psi} \!\cap\! \modelsOf{\alpha} & \!\!\!\! \modelsOf{\Psi} \!\cap\! \modelsOf{\alpha}\!\neq\!\emptyset  \\
		\modelsOf{\alpha} &\!\!\!\! \ksOtherwise
	\end{cases} 
\end{equation*}
Another application is composition of operators for the separation of the epistemic state  into substructures, e.g. for complexity purposes \cite{KS_Liberatore1997,KS_SchwindKoniecznyLagniezMarquis2020} or syntax splitting techniques \cite{KS_PeppasWilliamsChopraFoo2015,KS_KernIsbernerBrewka2017%
}.

Furthermore, belief revision is often connoted as a theory to describe rational change, thus, the theory is used normatively. 
However, in research like the cognitive logics approach \cite{RagniKern-IsbernerBeierleSauerwald2020a} the human reasoner are the norm. 
By this point of view, logic and belief revision operators are tools to describe the mechanisms of human and real world reasoning.
We think that flexible formalisms, like the \dynamiclimited\ revision and \ihlimited\ revision operators presented here, are useful for such investigations.

\clearpage
\bibliographystyle{kr}

\newcommand{\verzeichnisBibtex}{\string~/BibTeXReferencesSVNlink}
\bibliography{bibexport}

\clearpage
\appendix
\pagenumbering{Roman}
\section*{Supplementary Material: Proofs}

We will make use of the following two lemma.
\begin{lemma}\label{lem:ssc_dc_M} Let \( X \subseteq \propLang \) such that every element of \( X \) is consistent.
	Then \( X \) satisfies \ref{pstl:singlesentenceclosure} and \ref{pstl:disjunction_completeness} if and only if there is a set of interpretations \( M\subseteq \Omega \) such that \( X = \{ \alpha\in\propLang \mid \modelsOf{\alpha}\cap M \neq\emptyset  \} \).
\end{lemma}
\begin{proof}
This has been shown implicitly by Hansson et al. (\citeyear[Thm. 11]{KS_HanssonFermeCantwellFalappa2001}). By their proof, \( M \) is given by the models of \( \{ \alpha \mid \negOf{\alpha}\not\in X \} \).
\end{proof}

\begin{lemma}\label{lem:sem_trichotonomy}
	For  a total preorder $ \preceq $  over a set $ X $, and $ A,B\subseteq X $, it holds either $ {\min(A\cup B,\preceq )} = {\min(A,\preceq)} $ or $ {\min(A\cup B,\preceq )} = {\min(B,\preceq)} $ or $ {\min(A\cup B,\preceq )} = {\min(A,\preceq)} \cup {\min(B,\preceq)} $.
\end{lemma}
\begin{proof} 
Let $ x_1,x_2\in \min(X,\preceq) $. By totality of the order $ \preceq $, we have $ x_1 \preceq x_2 $.
	Observe now that by transitivity of the order $ \preceq $, we have $ x_1\in \min(X\cup Y,\preceq) $ if and only if $ x_2\in \min(X\cup Y,\preceq) $. This implies the claim.
\end{proof}

In the following we present missing proofs for propositions from the paper.

\setcounterref{theorem}{lem:dyl_charactersion_interpretationsinTPO}
\addtocounter{theorem}{-1}
\begin{lemma}
	Let $ \dylRevision $ be a \dynamiclimited\ revision operator compatible with a limited assignment $ \Psi\mapsto\preceq_{\Psi} $. The following holds:
	
	\begin{enumerate}[(a)]
		\item If  $ \omega {\in} \modelsOf{\Psi} $, then $ \omega\in\scope{\Psi} $ if and only if $\varphi_{\omega}$ is \preservable\ in~$ \Psi $.
		\item If  $ \omega {\notin} \modelsOf{\Psi} $, then $ \omega\in\scope{\Psi} $ if and only if $\varphi_{\omega}$ is \inconlifted\ in~$ \Psi $.
	\end{enumerate}
\end{lemma}
\begin{proof}We show (a) and (b) independently.
	
	\smallskip
	\emph{Statement (a).} Let $ \omega\in\modelsOf{\Psi} $ and thus $ \beliefsOf{\Psi} \subseteq \Cn(\varphi_{\omega}) $. 
	If $ \omega\in\scope{\Psi} $, then by the faithfulness of $ \preceq_{\Psi} $ we obtain $ \omega\in{\min(\scope{\Psi},\preceq_{\Psi})} $. This implies $ \omega\in  {\min(\modelsOf{\beta},\preceq_{\Psi})} = \modelsOf{\Psi\dylRevision\beta} $ for $ \beta $ with $ \omega\models\beta $.
	Therefore, $ \varphi_{\omega} $ is \preservable.
	
	If $ \omega\notin\scope{\Psi} $, then by the non-triviality of $ \dylRevision $ we obtain $ \omega'\in\scope{\Psi} $. Now choose $ \omega\models\alpha $ where $ \alpha $ chosen such that $ \modelsOf{\alpha}=\{\omega,\omega'\} $. 
	We obtain that $ \varphi_{\omega} $ is not \preservable, because application of Definition \ref{def:dynamic_limited_revision} yields $ \modelsOf{\Psi\dylRevision\varphi_{\omega}} \not\subseteq \modelsOf{\Psi\dylRevision\alpha} $.

	\smallskip
	\emph{Statement (b).} Let $ \omega\notin\modelsOf{\Psi} $. 
	
	We Consider the case of  $ \omega\in\scope{\Psi} $ and show $ \varphi_{\omega} $ is \inconlifted.
	Let $ \beta $ a formula such that $ \omega\notin\modelsOf{\Psi\dylRevision\beta} $.
	If $ \modelsOf{\alpha}\cap\scope{\Psi}=\emptyset $, then $ \modelsOf{\Psi\dylRevision\alpha}=\modelsOf{\Psi} $ but $ \omega\notin \modelsOf{\Psi} $.
	If $ \modelsOf{\alpha}\cap\scope{\Psi}\neq\emptyset $, then $ \omega\notin{\min(\modelsOf{\alpha},\preceq_{\Psi})} $.
	Because $ \modelsOf{\Psi\dylRevision\varphi_{\omega}}=\{\omega\} $ we obtain $ \beliefsOf{\Psi\dylRevision\beta}\not\subseteq \beliefsOf{arg1} $.
	Consequently, $ \varphi_{\omega} $ is \inconlifted.

	We Consider the case of $ \omega\notin\scope{\Psi} $ and show $ \varphi_{\omega} $ is not \inconlifted.
	If $ \modelsOf{\Psi}=\emptyset $, then $ \modelsOf{\Psi\dylRevision\varphi_{\omega}}+\varphi_{\omega} $ is inconsistent.
	Consequently, we obtain that $ \varphi_{\omega} $ is not \inconlifted, because $ \modelsOf{\Psi\dylRevision\varphi_{\omega}} = \modelsOf{\Psi\dylRevision\varphi_{\omega}} $.
	If  $ \modelsOf{\Psi}\neq\emptyset $,
	then exist $ \omega'\in\modelsOf{\Psi} $. 
	Because $ \omega\notin\modelsOf{\Psi} $ we obtain $ \beliefsOf{\Psi\dylRevision\varphi_{\omega'}}+\varphi_{\omega} $ is inconsistent.
	From $ \omega\notin\scope{\Psi} $ and Definition \ref{def:dynamic_limited_revision} obtain that $ \modelsOf{\Psi\dylRevision\varphi_{\omega}}=\modelsOf{\Psi} $.
	Therefore, the formula $ \varphi_{\omega} $ is not \inconlifted. \qedhere
\end{proof} 

\setcounterref{theorem}{prop:inherent_dlr}
\addtocounter{theorem}{-1}
\begin{proposition}
	Let $ \dylRevision $ be a \dynamiclimited\ revision operator compatible with a limited assignment $ \Psi\mapsto\preceq_{\Psi} $.	
	A belief $ \alpha $ is \integral\  in $ \Psi $ if and only if $ \modelsOf{\alpha}\subseteq\scope{\Psi} $ and $ \alpha $ is consistent.
\end{proposition}
\begin{proof}

	By definition for every $ \omega\models\alpha $  we obtain that $ \varphi_{\omega} $ is \atomic. 
	From Lemma \ref{lem:dyl_charactersion_interpretationsinTPO} we obtain that $ \omega\in\scope{\Psi} $.
	Therefore $ \modelsOf{\alpha}\subseteq\scope{\Psi} $.
	
	Now let $ \alpha $ a formula such that $ \modelsOf{\alpha}\subseteq\scope{\Psi} $.
	If $ \beliefsOf{\Psi}+\alpha $ is consistent, then obtain $ \modelsOf{\Psi\dylRevision\alpha}\subseteq{\min(\scope{\Psi},\preceq_{\Psi})} $.
	Then for every $ \beta $ with $ \alpha\models\beta $ we obtain $ \modelsOf{\Psi\dylRevision\alpha}\subseteq \modelsOf{\Psi\dylRevision\beta} $ by Definition \ref{def:dynamic_limited_revision}. 
	This shows that $ \alpha $ is an \preservable\ belief in $ \Psi $.
	We show that $ \alpha $ is \inconlifted\ in $ \Psi $. 
	From $ \modelsOf{\Psi \dylRevision \alpha} \subseteq \modelsOf{\Psi \dylRevision \beta} $ obtain $ {\min(\modelsOf{\alpha},\preceq_{\Psi})} \subseteq {\modelsOf{\Psi \dylRevision \beta}} $.
	This yields immediately consistency of  $ \beliefsOf{\Psi\dylRevision\beta}+\alpha $ and thus $ \alpha $ is \inconlifted\ in $ \Psi $.
	\qedhere	
\end{proof} 

\setcounterref{theorem}{prop:scope_by_s}
\addtocounter{theorem}{-1}
\begin{theorem}
	For an epistemic state \( \Psi \) and \dynamiclimited\ revision operator \( \dylRevision \) the following statements hold:
\begin{itemize}
	\item Syntactically, the scope of \( \dylRevision \) and \( \Psi \) is given by:
	\begin{equation*}
		\Scope{\dylRevision}{\Psi} = \beliefsOf{\Psi} \cup \{ \alpha \mid \beta\models\alpha \ksAND \beta \text{ is \integral} \}
	\end{equation*}
	
	\item If \( \dylRevision \) is compatible with \( \Psi\mapsto(\prec_{\Psi},\scope{\Psi}) \), then: %
	\begin{equation*}
		\Scope{\dylRevision}{\Psi} = \beliefsOf{\Psi} \cup \{  \alpha \mid  \modelsOf{\alpha} \cap \scope{\Psi} \neq \emptyset   \} 
	\end{equation*}
\end{itemize}	
\end{theorem}
We start by showing that $ \Scope{\dylRevision}{\Psi} = \beliefsOf{\Psi} \cup \{  \alpha \mid  \modelsOf{\alpha} \cap \scope{\Psi} \neq \emptyset   \}  $.
By Definition \ref{def:dynamic_limited_revision}, if $ \alpha $ in \( \beliefsOf{\Psi} \), then we have \( \alpha \in \beliefsOf{\Psi\dylRevision\alpha}  \).
Thus, we have \( \alpha \in \Scope{\dylRevision}{\Psi} \) for all $ \alpha $ in \( \beliefsOf{\Psi} \), i.e. \( \beliefsOf{\Psi}\subseteq \Scope{\dylRevision}{\Psi} \).
For \( \alpha \in \Scope{\dylRevision}{\Psi}\setminus \beliefsOf{\Psi} \) we obtain \( \modelsOf{\alpha}\cap\scope{\Psi}\neq\emptyset \) from Definition \ref{def:dynamic_limited_revision}.
Likewise, if \( \alpha \notin\beliefsOf{\Psi} \) and \( \modelsOf{\alpha}\cap\scope{\Psi}\neq\emptyset \), then from Definition \ref{def:dynamic_limited_revision} we obtain \( \alpha \in \beliefsOf{\Psi\dylRevision\alpha} \). 

As next step, we show \( \Scope{\dylRevision}{\Psi} = \beliefsOf{\Psi} \cup \{ \alpha \mid \beta\models\alpha \ksAND \beta \text{ is \integral} \} \).
Observe that by Proposition \ref{prop:inherent_dlr} and Definition \ref{def:dynamic_limited_revision} we have 
that \( \modelsOf{\beta}\subseteq \scope{\Psi} \) for every reasonable belief \( \beta \).
Consequently, we obtain \( \{ \alpha \mid \beta\models\alpha \ksAND \beta \text{ is \integral} \} = \{  \alpha \mid  \modelsOf{\alpha} \cap \scope{\Psi} \neq \emptyset   \} \).
This shows the claim.

\setcounterref{theorem}{thm:dylr_representationtheorem}
\addtocounter{theorem}{-1}
\begin{theorem}
	A belief change operator $ \dylRevision $ is a \dynamiclimited\ revision operator if and only if $ \dylRevision $ satisfies \eqref{pstl:DL1}--\eqref{pstl:DL8}.
\end{theorem}
\begin{proof}
	\noindent\textit{The \enquote*{$ \Rightarrow $}-direction.} 
	Let $ \dylRevision $ be an operator satisfying the postulates \eqref{pstl:DL1}--\eqref{pstl:DL8}. 
	For $ \Psi $, we construct $ (\preceq_{\Psi},\Omega_\Psi) $ as follows,  let $ \preceq_{\Psi} $ be the relation with
	\(	\omega_1 \preceq_\Psi \omega_2 \ksIFF \omega_1 \in \modelsOf{\Psi\dylRevision (\varphi_{\omega_1}\lor\varphi_{\omega_2})}\)
	and $ \dom(\preceq_{\Psi})={\scope{\Psi}}= \{ \omega \mid \varphi_\omega \text{ is \atomic\ in $ \Psi $ for } \dylRevision \} $.
	The relation $ \preceq_{\Psi} $ is a total preorder:

	\emph{Totality/reflexivity.} Let $ \omega_1,\omega_2\in\scope{\Psi}) $. Then by \eqref{pstl:DL8} we have that $ \modelsOf{\Psi\dylRevision(\varphi_{\omega_1}\lor\varphi_{\omega_2})} $ is equivalent to $ \{ \omega_1 \} $ or $ \{ \omega_2 \} $ or $ \{ \omega_1,\omega_2 \} $. Therefore, the relation must be total. Reflexivity follows from totality.

	\emph{Transitivity.} 	Let $ \omega_1,\omega_2,\omega_3\in\scope{\Psi}) $ with $ \omega_1 \preceq_{\Psi} \omega_2 $ and $ \omega_2\preceq_{\Psi} \omega_3 $. Towards a contradiction assume that $ \omega_1 \not\preceq_{\Psi} \omega_3 $ holds. This implies  $ \modelsOf{\Psi\dylRevision\varphi_{\omega_1,\omega_3}}=\{\omega_3\} $.
	Now assume that $ \modelsOf{\Psi\dylRevision\varphi_{\omega_1,\omega_2,\omega_3}} = \{\omega_3\} $.
	Then by \eqref{pstl:DL8} we obtain that $ \modelsOf{\Psi\dylRevision\varphi_{\omega_2,\omega_3}}=\{\omega_3\} $, a contradiction to $ \omega_2\preceq_{\Psi}\omega_3 $.
	Assume for the remaining case $ \modelsOf{\Psi\dylRevision\varphi_{\omega_1,\omega_2,\omega_3}} \neq \{\omega_3\} $. 
	We obtain from \eqref{pstl:DL8} that $ \modelsOf{\Psi\dylRevision\varphi_{\omega_1,\omega_2,\omega_3}} $ equals either $ \modelsOf{\Psi\dylRevision\varphi_{\omega_1,\omega_2}} $ or $ \modelsOf{\Psi\dylRevision\varphi_{\omega_3}}$. 
	Since the second case is impossible, $ \omega_1\preceq_{\Psi} \omega_2 $ implies $ \omega_1\in \modelsOf{\Psi\dylRevision\varphi_{\omega_1,\omega_2}} $. Now apply \eqref{pstl:DL8} again to $ \modelsOf{\Psi\dylRevision\varphi_{\omega_1,\omega_2,\omega_3}} $ and obtain that it is either equivalent to $ \modelsOf{\Psi\dylRevision\varphi_{\omega_1,\omega_3}} $ or $ \modelsOf{\Psi\dylRevision\varphi_{\omega_2}} $.
	In both cases, we obtain a contradiction because $ \omega_1,\omega_3\in \modelsOf{\Psi\dylRevision\varphi_{\omega_1,\omega_2,\omega_3}} $.

	\smallskip
	\noindent The construction yields a faithful limited assignment:\\
	We show $ \modelsOf{\Psi}\cap\scope{\Psi})=\min(\Omega_\Psi,\preceq_{\Psi}) $. 
	Let $ \omega_1,\omega_2\in\scope{\Psi}) $ and $ \omega_1,\omega_2\in\modelsOf{\Psi} $. 
	By definition of $ {\scope{\Psi})} $ the interpretations $ \omega_1,\omega_2 $ are \atomic\ in $ \Psi $. 
	Thus, the formula $ \varphi_{\omega_1,\omega_2} $ is \integral.
	From \eqref{pstl:DL3} we obtain $ \modelsOf{\Psi\dylRevision\varphi_{\omega_1,\omega_2}}=\{ \omega_1,\omega_2 \} $ which yields by definition $ \omega_1 \preceq_{\Psi} \omega_2 $ and $ \omega_2 \preceq_{\Psi} \omega_1 $.
	Let $ \omega_1,\omega_2\in\scope{\Psi}) $ with  $ \omega_1\in\modelsOf{\Psi} $ and $ \omega_2\notin\modelsOf{\Psi} $.
	Then $ \varphi_{\omega_1,\omega_2} $ is consistent with $ \beliefsOf{\Psi} $. 
	Therefore, by 
	\eqref{pstl:DL3} and \eqref{pstl:DL5} we have $ \modelsOf{\Psi\dylRevision \varphi_{\omega_1,\omega_2}} = \modelsOf{\Psi}\cap\modelsOf{\varphi_{\omega_1,\omega_2}}=\{ \omega_1 \} $.
	Together we obtain faithfulness.
	
	\smallskip \noindent
	We show satisfaction of \eqref{eq:limited_revision} in two steps:
	
	For the first step, assume $ \modelsOf{\alpha}\cap\scope{\Psi})=\emptyset $.
	By the postulate  \eqref{pstl:DL2} we have either $ \beliefsOf{\Psi}=\beliefsOf{\Psi\dylRevision\alpha} $ or $ \beliefsOf{\Psi\dylRevision\alpha} $ is \integral\ in $ \Psi $. In the first case we are done. 
	For the second case, by the postulate \eqref{pstl:DL1}, we obtain $ \modelsOf{\Psi\dylRevision\alpha}\subseteq\modelsOf{\alpha} $. 
	Because $ \beliefsOf{\Psi\dylRevision\alpha} $ is \integral,
	the set  $ \modelsOf{\alpha }$ contains an interpretation $ \omega $ such that $ \varphi_{\omega} $ is \atomic\ in $ \Psi $, a contradiction to $ \modelsOf{\alpha}\cap\scope{\Psi})=\emptyset $.

	For the second step assume $ \modelsOf{\alpha}\cap\scope{\Psi})\neq\emptyset $.
	We show the equivalence $ \modelsOf{\Psi\dylRevision\alpha}=\min(\modelsOf{\alpha},\preceq_{\Psi}) $ by showing both set inclusions separately.
	
	We show $ \min(\modelsOf{\alpha},\preceq_{\Psi})\subseteq \modelsOf{\Psi\dylRevision\alpha} $.
	Let $ \omega\in\min(\modelsOf{\alpha},\preceq_{\Psi}) $ with $ \omega\notin\modelsOf{\Psi\dylRevision\alpha}$. 
	By construction of \( \scope{\Psi}) \) the formula $ \varphi_{\omega} $ is \atomic, and therefore \integral. 
	From postulate \eqref{pstl:DL4}, we obtain that $\beliefsOf{\Psi\dylRevision\alpha} $ is \integral\ in $ \Psi $. 
	Consequently, by Lemma \ref{prop:inherent_dlr}, every interpretation in $ \modelsOf{\Psi\dylRevision\alpha} $ is \atomic\ in $ \Psi $.
	Hence, there is at least one $ \omega'\in\beliefsOf{\Psi\dylRevision\alpha}  $ with $ \omega'\in\scope{\Psi}) $.
	If $ \beliefsOf{\Psi\dylRevision\alpha}=\beliefsOf{\Psi} $ we obtain $ \min(\modelsOf{\alpha},\preceq_{\Psi}) = \modelsOf{\Psi\dylRevision\alpha} $ by the faithfulness of $ \Psi\mapsto\leq_{\Psi} $. If $ \beliefsOf{\Psi\dylRevision\alpha}\neq\beliefsOf{\Psi} $, then by \eqref{pstl:DL1} we have $ \modelsOf{\Psi\dylRevision\alpha}\subseteq\modelsOf{\alpha} $.
	Let $ \beta=\varphi_{\omega}\lor\varphi_{\omega'} $ and $ \gamma=\beta\lor\gamma' $ such that $ \modelsOf{\gamma}=\modelsOf{\alpha} $ and $ \modelsOf{\gamma'}=\modelsOf{\alpha}\setminus\{\omega,\omega'\} $.
	By \eqref{pstl:DL8} we have either $ \modelsOf{\Psi\dylRevision\alpha}=\modelsOf{\Psi\dylRevision\beta} $ or $ \modelsOf{\Psi\dylRevision\alpha}=\modelsOf{\Psi\dylRevision\gamma'} $ or $ \modelsOf{\Psi\dylRevision\alpha}=\modelsOf{\Psi\dylRevision\beta}\cup\modelsOf{\Psi\dylRevision\gamma'} $.
	The first and the third case are impossible, because $ \omega\notin\modelsOf{\Psi\dylRevision\alpha} $ and by the minimality of $ \omega $ we have $ \omega \in \modelsOf{\Psi\dylRevision\beta} $.
	It remains the case of $ \modelsOf{\Psi\dylRevision\alpha}=\modelsOf{\Psi\dylRevision\gamma'} $. 
	Let $ \omega_{\gamma'} $ such that $ \omega_{\gamma'}\in\modelsOf{\Psi\dylRevision\alpha} $. Note that $ \omega_{\gamma'}\in\modelsOf{\gamma'}\subseteq\modelsOf{\alpha} $ and $ \omega_{\gamma'}\in\scope{\Psi}) $.
	Now let $ \delta=\beta'\lor\delta' $ with $ \delta\equiv\alpha $ such that $ \modelsOf{\beta'}=\{\omega,\omega_{\gamma'} \} $ and $ \modelsOf{\delta'}=\modelsOf{\alpha}\setminus\{ \omega,\omega_{\gamma'} \} $. 
	By minimality of $ \omega $ we have $ \omega\in\modelsOf{\Psi\dylRevision\beta'} $.
	By \eqref{pstl:DL8} we obtain $ \modelsOf{\Psi\dylRevision\alpha} $ is either equivalent to $ \modelsOf{\Psi\dylRevision\beta'} $ or to $ \modelsOf{\Psi\dylRevision\delta'} $ or to $  \modelsOf{\Psi\dylRevision\beta'} \cup \modelsOf{\Psi\dylRevision\delta'} $. The first and third case are impossible since $ \omega\notin \modelsOf{\Psi\dylRevision\alpha} $. Moreover, the second case is also impossible, because of $ \omega_{\gamma'}\notin \modelsOf{\Psi\dylRevision\delta'} $.

	We show $ \modelsOf{\Psi\dylRevision\alpha} \subseteq \min(\modelsOf{\alpha},\preceq_{\Psi}) $.
	Let  $ \omega\in\modelsOf{\Psi\dylRevision\alpha} $ with  $ \omega\notin \min(\modelsOf{\alpha},\preceq_{\Psi})  $.
	From non-emptiness of $ \min(\modelsOf{\alpha},\preceq_{\Psi}) $ and $ \min(\modelsOf{\alpha},\preceq_{\Psi})\subseteq \modelsOf{\Psi\dylRevision\alpha} $ we obtain $ \varphi_{\omega'} $ is \atomic\ in $ \Psi $ where $ \omega'\in\modelsOf{\Psi\dylRevision\alpha} $ such that $ \omega'\in \min(\modelsOf{\alpha},\preceq_{\Psi}) $.
	
	Assume that $ \varphi_\omega $ is not \atomic\ in $ \Psi $, and therefore, $ {\omega\notin\scope{\Psi})} $.
	By the postulate \eqref{pstl:DL4} we obtain from the existence of $ \omega' $ that $ \beliefsOf{\Psi\dylRevision\alpha} $ is \integral\ in $ \Psi $.  Therefore, by Lemma \ref{lem:inherent_inherence_limited} every model in $ \modelsOf{\Psi\dylRevision \alpha} $ is \atomic\ in $ \Psi $, a contradiction, and  therefore, $ \varphi_\omega $ has to be \atomic\ in $ \Psi $.
	
	Because $ \varphi_\omega $ is \atomic\ in $ \Psi $ we obtain $ \modelsOf{\Psi\dylRevision\varphi_{\omega}}=\{\omega\} $.
	Using the postulate \eqref{pstl:DL4} we obtain that $ \beliefsOf{\Psi\dylRevision\alpha} $ is \integral\ in $ \Psi $.
	Since $ \omega $ is not minimal, we have $ \omega'\in\modelsOf{\Psi\dylRevision\varphi_{\omega,\omega'}} $ and $ \omega\notin\modelsOf{\Psi\dylRevision\varphi_{\omega,\omega'}} $
	Now let $ \gamma=\varphi_{\omega,\omega'}\lor\gamma' $ with $ \modelsOf{\gamma'}=\modelsOf{\alpha}\setminus\{\omega,\omega'\} $.
	By \eqref{pstl:DL8} we have either $ \modelsOf{\Psi\dylRevision\alpha}=\modelsOf{\Psi\dylRevision\beta} $ or $ \modelsOf{\Psi\dylRevision\alpha}=\modelsOf{\Psi\dylRevision\gamma'} $ or $ \modelsOf{\Psi\dylRevision\alpha}=\modelsOf{\Psi\dylRevision\beta}\cup\modelsOf{\Psi\dylRevision\gamma'} $.
	All cases are impossible because for every case we obtain $ \omega\notin\modelsOf{\Psi\dylRevision\alpha} $.
	This shows $ \modelsOf{\Psi\dylRevision\alpha} \subseteq \min(\modelsOf{\alpha},\preceq_{\Psi}) $, and, in summary, we obtain $ \modelsOf{\Psi\dylRevision\alpha} = \min(\modelsOf{\alpha},\preceq_{\Psi}) $ and thus, \eqref{eq:limited_revision} holds.
	
	\smallskip
	\noindent\textit{The \enquote*{$ \Leftarrow $}-direction.}
	Let $ \Psi \mapsto (\preceq_\Psi,\Omega_\Psi) $ be a dynamic-limited assignment compatible with $ \dylRevision $. 
	Note that by Lemma \ref{prop:inherent_dlr} for every model $ \omega $ of an \atomic\ beliefs $ \alpha $ we obtain that $ \varphi_{\omega} $ is \atomic\ in $ \Psi $. %
	Moreover, $ \varphi_\omega $ is \atomic\ in $ \Psi $ if $ \omega\in\scope{\Psi} $.
	
	We show the satisfaction of \eqref{pstl:DL1}--\eqref{pstl:DL8}.
	From \eqref{eq:limited_revision} we obtain straightforwardly \eqref{pstl:DL1}, \eqref{pstl:DL2}, \eqref{pstl:DL5} and \eqref{pstl:DL6}.
	\begin{description}
		\item[\eqref{pstl:DL3}] Assume $ \alpha $ to be \integral\ in $ \Psi $ and $ \modelsOf{\Psi}\cap\modelsOf{\alpha}\neq\emptyset $.
		By Proposition \ref{prop:inherent_dlr} we have $ \modelsOf{\Psi}\cap\modelsOf{\alpha} \subseteq \scope{\Psi} $.
		From faithfulness of $ \Psi\mapsto\preceq_{\Psi} $ we obtain $ \min(\modelsOf{\alpha},\preceq_{\Psi})=\modelsOf{\Psi}\cap\modelsOf{\alpha} $.
		\item[\eqref{pstl:DL4}] Let $ \alpha\models\beta $ and $ \alpha $ \integral.
		By Proposition \ref{prop:inherent_dlr} we have $ \modelsOf{\alpha}\subseteq\scope{\Psi}) $.
		Then by \eqref{eq:limited_revision} we obtain that $ \beliefsOf{\Psi\dylRevision\beta} $ is \integral.
		\item[\eqref{pstl:DL8}] 
		Suppose $ \modelsOf{\alpha\lor\beta}\cap\scope{\Psi}=\emptyset $, then by \eqref{eq:limited_revision} for every formula $ \gamma $ with $ \modelsOf{\gamma} \subseteq \modelsOf{\alpha\lor\beta} $ we obtain $ \modelsOf{\Psi\dylRevision(\alpha\lor\beta)}=\modelsOf{\Psi}=\modelsOf{\Psi\dylRevision\gamma} $.
		
		Assume that $ \modelsOf{\alpha}\cap\scope{\Psi}\neq\emptyset $ and $ \modelsOf{\beta}\cap\scope{\Psi}=\emptyset $.
		Then by the postulate \eqref{eq:limited_revision} we obtain $ \modelsOf{\Psi\dylRevision(\alpha\lor\beta)}=\modelsOf{\Psi\dylRevision\alpha} $.
		The case of $ \modelsOf{\alpha}\cap\scope{\Psi}=\emptyset $ and $ \modelsOf{\beta}\cap\scope{\Psi}\neq\emptyset $ is analogue.

		Assume that $ \modelsOf{\alpha}\cap\scope{\Psi}\neq\emptyset $ and $ \modelsOf{\beta}\cap\scope{\Psi}\neq\emptyset $.
		Then we have $ \modelsOf{\Psi\dylRevision(\alpha\lor\beta)}={\min(\modelsOf{\alpha\lor\beta},\preceq_{\Psi})} $. Using logical equivalence, we obtain $ \min(\modelsOf{\alpha\lor\beta},\preceq_{\Psi})=\min(\modelsOf{\alpha}\cup\modelsOf{\beta},\preceq_{\Psi}) $.
		By Lemma \ref{lem:sem_trichotonomy} we obtain directly \eqref{pstl:DL8}. \qedhere
	\end{description}
\end{proof} 

\setcounterref{theorem}{prop:agm_scope}
\addtocounter{theorem}{-1}
\begin{proposition}
	For every AGM revision operator \( * \) we have \( Scope(*,\Psi)  =\propLang  \) for every \( \Psi\in\setAllES \).
\end{proposition}
\begin{proof}
	A direct consequence of Proposition \ref{prop:es_revision}.
\end{proof} 

\setcounterref{theorem}{prop:clr_scope}
\addtocounter{theorem}{-1}
\begin{proposition}	Let  \( \Psi \) be an epistemic state. The following statements hold:
	\begin{enumerate}[(a)]
		\item If \( \clRevision \) is a credibility-limited revision operator, then \( \beliefsOf{\Psi} \subseteq \Scope{\clRevision}{\Psi} \), and \( \Scope{\clRevision}{\Psi} \) satisfies \ref{pstl:singlesentenceclosure} and \ref{pstl:disjunction_completeness}.
		\item For each \( X\subseteq\mathcal{L} \) with  \( \beliefsOf{\Psi} \subseteq X \), and \( X \) satisfies \ref{pstl:singlesentenceclosure} and \ref{pstl:disjunction_completeness}, there exist a credibility-limited revision operator \( \clRevision \) such that \( \Scope{\clRevision}{\Psi} = X \).
	\end{enumerate}
\end{proposition}
\begin{proof}[Proof (sketch)] Remember that every credibility-limited revision operator \( \clRevision \) is compatible with a CLF-assignment \( \Psi\mapsto (\preceq_{\Psi},C_\Psi) \).
	Clearly, by the semantic characterisation of the operator we have \( \Scope{\clRevision}{\Psi}=\{ \alpha \mid \modelsOf{\alpha} \cap C_\Psi \neq \emptyset  \} \). Employing Lemma \ref{lem:ssc_dc_M} yields the statements.
\end{proof} 

\setcounterref{theorem}{prop:dlr_scope}
\addtocounter{theorem}{-1}
\begin{theorem}
	Let \( \Psi\in\setAllES \) be an epistemic state. The following two statements hold:
\begin{itemize}
	\item For every consistent set \( X  \) which satisfies \ref{pstl:singlesentenceclosure}  and \ref{pstl:disjunction_completeness}
	exists a \dynamiclimited\ revision operator \( \dylRevision \) with \( {\Scope{\dylRevision}{\Psi}=\beliefsOf{\Psi} \cup X} \).
	\item If \( \dylRevision \) is an \dynamiclimited\ revision operator, then \( {\Scope{\dylRevision}{\Psi}} \) satisfies \ref{pstl:singlesentenceclosure}. Moreover, the set \( {\Scope{\dylRevision}{\Psi}} \setminus \beliefsOf{\Psi} \) satisfies \ref{pstl:singlesentenceclosure}   and \ref{pstl:disjunction_completeness}.
\end{itemize}
\end{theorem}
\begin{proof}
A consequence of \eqref{eq:limited_revision} and Employing Lemma \ref{lem:ssc_dc_M} yields the statements.
\end{proof} 

\setcounterref{theorem}{prop:dylr_dp1}
\addtocounter{theorem}{-1}
\begin{proposition}
	A \dynamiclimited\ revision operator $ \dylRevision $ compatible with $ \Psi\mapsto(\preceq_{\Psi},\Omega_\Psi) $ satisfies \eqref{pstl:DP1} if and only if the following conditions hold:
	\begin{enumerate}[(i)]
		\item if \(  \omega_1,\omega_2 {\in} \modelsOf{\alpha} \cap \scope{\Psi}{\cap} \scope{\Psi} \), then
		\begin{equation*}
			\omega_1 \preceq_{\Psi} \omega_2 \Leftrightarrow \omega_1 \preceq_{\Psi\dylRevision\alpha} \omega_2
		\end{equation*}
		\item if \( \left| \scope{\Psi} \cap \modelsOf{\alpha} \right| \geq 2 \), then
		\begin{equation*}
			\scope{\Psi} \cap \modelsOf{\alpha} \subseteq \scope{\Psi\dylRevision\alpha}
		\end{equation*}
		otherwise
		\begin{equation*}
			\left(\scope{\Psi} \cap \modelsOf{\alpha}\right) \setminus \modelsOf{\Psi\dylRevision\alpha}  \subseteq \scope{\Psi\dylRevision\alpha}
		\end{equation*}
		\item if \( \left|  \modelsOf{\Psi} \right| \geq 2 \), then
		\begin{equation*}
			\scope{\Psi\dylRevision\alpha} \cap \modelsOf{\alpha} \subseteq \scope{\Psi}
		\end{equation*}
		otherwise
		\begin{equation*}
			\left(\scope{\Psi\dylRevision\alpha} \cap \modelsOf{\alpha}\right) \setminus \modelsOf{\Psi}  \subseteq \scope{\Psi}
		\end{equation*}
	\end{enumerate}
\end{proposition}
\begin{proof}
	\emph{The $ \Rightarrow $ direction.} Assume $ \dylRevision $ satisfies \eqref{pstl:DP1}.
	We show (i) to (iii).
	
	\smallskip
	\emph{(i).} Let $ \omega_1,\omega_2 \in \modelsOf{\alpha} $ and $ \omega_1,\omega_2\in\scope{\Psi}\cap\scope{\Psi\dylRevision\alpha} $. From \eqref{eq:limited_revision} obtain that $ \omega_1 \in \modelsOf{\Psi\dylRevision\varphi_{\omega_1,\omega_2}} $ implies $ \omega_1 \preceq_{\Psi} \omega_2 $. Consequently, by \eqref{pstl:DP1} obtain $ \omega_1\preceq_{\Psi}\omega_2 \Leftrightarrow   \omega_1\preceq_{\Psi\dylRevision\alpha}\omega_2 $.
	
	\smallskip
	\emph{(ii)} Consider the case of $ \left|\scope{{\Psi}}\cap\modelsOf{\alpha}\right| \geq 2 $. 
	Suppose there is some $ \omega\in \scope{\Psi} \cap \modelsOf{\alpha} $ but $ \omega\notin\scope{\Psi\dylRevision\alpha} $. Let $ \omega' \in\scope{\Psi} $. Observe that $ \modelsOf{\Psi\dylRevision\varphi_{\omega}} =\{\omega\} $ and $ \modelsOf{\Psi\dylRevision\varphi_{\omega'}} =\{\omega'\} $ by \eqref{eq:limited_revision}.
	If $ \modelsOf{\Psi\dylRevision\alpha}\neq\{\omega\} $, then from \eqref{eq:limited_revision} and $ \omega\notin\scope{\Psi\dylRevision\alpha} $ follows $ \modelsOf{\Psi\dylRevision\alpha\dylRevision\varphi_{\omega}} \neq\{\omega\} $ a contradiction to \eqref{pstl:DP1}.
	If $ \modelsOf{\Psi\dylRevision\alpha}=\{\omega\} $, then $ \min(\modelsOf{\alpha},\preceq_{\Psi}) =\{\omega\} $.
	Now let $ \beta=\varphi_{\omega,\omega'} $.
	From faithfulness obtain that $ \modelsOf{\Psi\dylRevision\beta}=\{\omega\} $. 
	The are two possibilities: $ \modelsOf{\Psi\dylRevision\alpha\dylRevision\beta}=\{\omega\} $ or $ \modelsOf{\Psi\dylRevision\alpha\dylRevision\beta}=\{\omega'\} $. The latter case directly contradicts \eqref{pstl:DP1}.
	The first case implies $ \modelsOf{\Psi\dylRevision\alpha\dylRevision\varphi_{\omega'}} =\{\omega\} $ a contradiction to \eqref{pstl:DP1}, because of $ \modelsOf{\Psi\dylRevision\varphi_{\omega'}} =\{\omega'\} $.
	
	Consider the case of $ \left|\scope{\Psi}\cap\modelsOf{\alpha}\right| < 2 $.
	Suppose there is some $ \omega\in \scope{\Psi} \cap \modelsOf{\alpha} $ but $ \omega\notin \modelsOf{\Psi\dylRevision\alpha}\cup\scope{\Psi\dylRevision\alpha} $. Then we obtain from \eqref{eq:limited_revision} a contradiction to \eqref{pstl:DP1}, because $ \modelsOf{\Psi\dylRevision\varphi_{\omega}} =\{\omega\} $ and $ \modelsOf{\Psi\dylRevision\alpha\dylRevision\varphi_{\omega}} \neq \{\omega\} $.
	
	\smallskip
	\emph{(iii)} Let $ \omega\in \scope{\Psi\dylRevision\alpha}\cap\modelsOf{\alpha} $ and therefore $ \modelsOf{\Psi\dylRevision\alpha\dylRevision\varphi_{\omega}}=\{\omega\} $.	
	If $ \left|\modelsOf{\Psi}\right| < 2 $ and $ \omega\notin\scope{\Psi}\cap\modelsOf{\Psi} $, then $ \modelsOf{\Psi\dylRevision\varphi_{\omega}}=\modelsOf{\Psi}\neq\{\omega\}  $. 	
	If $ \left|\modelsOf{\Psi}\right| \geq 2 $ and $ \omega\notin\scope{\Psi} $, then $ \modelsOf{\Psi\dylRevision\varphi_{\omega}}=\modelsOf{\Psi}\neq \{\omega\}  $.	
	A contradiction to \eqref{pstl:DP1} in both cases.
	
	\medskip
	\noindent\emph{The $ \Rightarrow $ direction.} 
	We prove by contraposition and show that a violation of (i), (ii) or (iii) implies a violation of \eqref{pstl:DP1}:
	
	\smallskip
	\emph{(i)} If $ \omega_1 \preceq_{\Psi} \omega_2 \not\Leftrightarrow \omega_1 \preceq_{\Psi\dylRevision\alpha} \omega_2 $, then because of $ \omega_1,\omega_2\in\scope{\Psi}\cap\scope{\Psi\dylRevision\alpha} $ we obtain $ \modelsOf{\Psi\dylRevision\varphi_{\omega_1,\omega_2}} \neq \modelsOf{\Psi\dylRevision\alpha\dylRevision\varphi_{\omega_1,\omega_2}} $ from \eqref{eq:limited_revision}.
	
	\smallskip
	\emph{(ii)} Let $ \omega\notin\scope{\Psi\dylRevision\alpha} $ and $ \omega\in\scope{\Psi}\cap\modelsOf{\alpha} $. Then, $ \modelsOf{\Psi\dylRevision\varphi_{\omega}} = \{ \omega \} $ and $ \modelsOf{\Psi\dylRevision\alpha\dylRevision\varphi_{\omega}} = \modelsOf{\Psi} $.
	
	If $ \left|\scope{\Psi}\cap\modelsOf{\alpha}\right| \geq 2 $, 
	then let  $ \omega' \scope{\Psi}\cap\modelsOf{\alpha} $ with $ \omega'\neq\omega $.
	For $ \modelsOf{\Psi\dylRevision\alpha}\neq\{\omega\} $, we directly obtain a violation of \eqref{pstl:DP1}.
	For $ \modelsOf{\Psi\dylRevision\alpha} = \{\omega\} $ obtain that $ \min(\modelsOf{\alpha},\preceq_{\Psi})=\{\omega\} $. Now observe that either $ \modelsOf{\Psi\dylRevision\alpha\dylRevision\varphi_{\omega'}} = \{\omega'\} $ or $ \modelsOf{\Psi\dylRevision\alpha\dylRevision\varphi_{\omega'}} = \modelsOf{\Psi\dylRevision\alpha} $.
	In the first case we obtain that $ \modelsOf{\Psi\dylRevision\varphi_{\omega,\omega'}}\neq\modelsOf{\Psi\dylRevision\alpha\dylRevision\varphi_{\omega,\omega'}} $. Likewise, the second case yields a violation of \eqref{pstl:DP1}, because $ \modelsOf{\Psi\dylRevision\alpha\dylRevision\varphi_{\omega'}} = \{\omega\} \neq \modelsOf{\Psi\dylRevision\varphi_{\omega'}} $.
	
	If $ \left|\scope{\Psi}\cap\modelsOf{\alpha}\right| < 2 $ and $ \omega\notin\modelsOf{\Psi\dylRevision\alpha} $, then we directly obtain $ \modelsOf{\Psi\dylRevision\alpha\dylRevision\varphi_{\omega}}=\modelsOf{\Psi\dylRevision\alpha}\neq \modelsOf{\Psi\dylRevision\varphi_{\omega}} $.
	
	\smallskip
	\emph{(iii)}  Let $ \omega\notin\scope{\Psi} $ and $ \omega\in\scope{\Psi\dylRevision\alpha}\cap\modelsOf{\alpha} $. Then $ \modelsOf{\Psi\dylRevision\alpha\dylRevision\varphi_{\omega}} = \{ \omega \} $ and $ \modelsOf{\Psi\dylRevision\varphi_{\omega}} = \modelsOf{\Psi} $.
	If $ \left|\modelsOf{\Psi}\right| \geq 2 $, or $ \left|\modelsOf{\Psi}\right| < 2 $ and $ \omega\notin\modelsOf{\Psi} $, then $ \modelsOf{\Psi\dylRevision\varphi_{\omega}} \neq \modelsOf{\Psi\dylRevision\alpha\dylRevision\varphi_{\omega}} $. \qedhere

\end{proof} 

\setcounterref{theorem}{prop:dylr_dp2}
\addtocounter{theorem}{-1}
\begin{proposition}
	A \dynamiclimited\ revision operator $ \dylRevision $ compatible with $ \Psi\mapsto(\preceq_{\Psi},\Omega_\Psi) $ satisfies \eqref{pstl:DP2} if and only if the following conditions hold:
	\begin{enumerate}[(i)]
		\item if \(  \omega_1,\omega_2 {\in} \modelsOf{\negOf{\alpha}} \cap \scope{\Psi}{\cap} \scope{\Psi} \), then
		\begin{equation*}
			\omega_1 \preceq_{\Psi} \omega_2 \Leftrightarrow \omega_1 \preceq_{\Psi\dylRevision\alpha} \omega_2
		\end{equation*}
		\item if \( \left| \scope{\Psi} \cap \modelsOf{\negOf{\alpha}} \right| \geq 2 \), then
		\begin{equation*}
			\scope{\Psi} \cap \modelsOf{\negOf{\alpha}} \subseteq \scope{\Psi\dylRevision\alpha}
		\end{equation*}
		otherwise
		\begin{equation*}
			\left(\scope{\Psi} \cap \modelsOf{\negOf{\alpha}}\right) \setminus \modelsOf{\Psi\dylRevision\alpha}  \subseteq \scope{\Psi\dylRevision\alpha}
		\end{equation*}
		\item if \( \left|  \modelsOf{\Psi} \right| \geq 2 \), then
		\begin{equation*}
			\scope{\Psi\dylRevision\alpha} \cap \modelsOf{\negOf{\alpha}} \subseteq \scope{\Psi}
		\end{equation*}
		otherwise
		\begin{equation*}
			\left(\scope{\Psi\dylRevision\alpha} \cap \modelsOf{\negOf{\alpha}}\right) \setminus \modelsOf{\Psi}  \subseteq \scope{\Psi}
		\end{equation*}
	\end{enumerate}
\end{proposition}

\begin{proof}
	Analogue to the proof of Proposition \ref{prop:dylr_dp1}.
\end{proof}

\setcounterref{theorem}{prop:dylr_dp3}
\addtocounter{theorem}{-1}
\begin{proposition}
	A \dynamiclimited\ revision operator $ \dylRevision $ compatible with $ \Psi\mapsto(\preceq_{\Psi},\Omega_\Psi) $ satisfies \eqref{pstl:DP3} if and only if the following conditions hold:
	\begin{enumerate}[(i)]
		\item if and  \(  \omega_1 {\models} \alpha \) and \( \omega_2 {\not\models} \alpha\) and \( \omega_1,\omega_2{\in}\scope{\Psi}{\cap}\scope{\Psi\dylRevision\alpha} \), then
		\begin{equation*}
			\omega_1 \prec_{\Psi} \omega_2 \Rightarrow \omega_1 \prec_{\Psi\dylRevision\alpha} \omega_2
		\end{equation*}
		\item if  \(  \omega_1 \models\alpha \) and \( \omega_2 \not\models\alpha\) and \( \omega_1 \prec_{\Psi} \omega_2 \), then
		\begin{equation*}
			\omega_2 \in \scope{\Psi\dylRevision\alpha} \Rightarrow  \omega_1 \in \scope{\Psi\dylRevision\alpha} 
		\end{equation*}
		\item if \( \omega \not\models \alpha \) and \( \Psi\models \alpha  \), then  
		\(  \omega \in\scope{\Psi\dylRevision\alpha} \Rightarrow \omega \in\scope{\Psi} \)
		\item if  \(  \omega_1 {\models} \alpha \) and \( \omega_2 {\not\models} \alpha\) and \( \omega_1 {\in} \scope{\Psi} \) and \( \omega_2 {\in} \scope{\Psi\dylRevision\alpha} \), then
			\begin{equation*}
				\left( \omega_1\notin\scope{\Psi\dylRevision\alpha} \ksOR \omega_2 \preceq_{\Psi\dylRevision\alpha} \omega_1 \right) \Rightarrow  \omega_2 \in \scope{\Psi}
		\end{equation*}
	\end{enumerate}
\end{proposition}
\begin{proof}
	\emph{The \enquote{$ \Rightarrow $} direction.} 
	We prove satisfaction of (i) -- (iv) in the presence of \eqref{pstl:DP3}:
	
	\smallskip%
	\emph{(i)} Let \( \omega_1,\omega_2\in\scope{\Psi}\cap\scope{\Psi\dylRevision\alpha} \) with \(  \omega_1 \models\alpha \) and \( \omega_2 \not\models\alpha\), and \( \omega_1 \prec_{\Psi}\omega_2 \).
	From \eqref{eq:limited_revision} and \eqref{pstl:DP3} we easily obtain \( \omega_1 \prec_{\Psi\dylRevision\alpha} \omega_2 \) by choosing \( \beta=\varphi_{\omega_1,\omega_2} \).

	\smallskip%
	\emph{(ii)} Assume \( \omega_2\in\scope{\Psi\dylRevision\alpha} \) and \( \omega_1\notin \scope{\Psi\dylRevision\alpha} \).
	Then for \( \beta=\varphi_{\omega_1,\omega_2} \)  we obtain \( \Psi\dylRevision\beta\models\alpha  \) and \( \Psi\dylRevision\alpha\dylRevision\beta\models\alpha  \).
	A contradiction to \eqref{pstl:DP3}.

	\smallskip%
	\emph{(iii)} Let \( \omega\in\scope{\Psi\dylRevision\alpha} \) and \( \omega\notin\scope{\Psi} \). 
	From \( \Psi\models\alpha \) and \eqref{eq:limited_revision} obtain \(  \Psi\dylRevision\varphi_{\omega} \models\alpha \).
	But by \( \omega\not\models\alpha \) and \( \omega\in\scope{\Psi\dylRevision\alpha} \) we obtain a contradiction to \eqref{pstl:DP3}, because \( \Psi\dylRevision\alpha\dylRevision\varphi_{\omega}\not\models\alpha \).
	
	\smallskip%
	\emph{(iv)} Suppose \( \omega_2\notin\scope{\Psi} \) and let \( \beta=\varphi_{\omega_1,\omega_2} \).
	Both, \( \omega_1\notin \scope{\Psi\dylRevision\alpha} \) or \( \omega_2 \preceq_{\Psi\dylRevision\alpha} \omega_1 \), yields \( \Psi\dylRevision\alpha\dylRevision\beta\not\models\alpha \) because of \eqref{eq:limited_revision}.
	From \eqref{pstl:DP3} obtain \( \Psi\dylRevision\beta\not\models\alpha \). However, from \( \omega_2\notin\scope{\Psi} \) and \( \omega_1\in\scope{\Psi} \) we obtain \( \Psi\dylRevision\beta\models\alpha \).
	
	\medskip\noindent
	\emph{The \enquote{$ \Leftarrow $} direction.} Let \( \Psi\dylRevision\beta\models\alpha \). We show \( \Psi\dylRevision\alpha\dylRevision\beta\models\alpha \).
	Towards a contradiction, assume \( \omega\in\modelsOf{\Psi\dylRevision\alpha\dylRevision\beta} \) with \( \omega\notin\modelsOf{\alpha} \).
	
	Consider the case of \( \omega\in\scope{\Psi\dylRevision\alpha} \).
	If \( \omega\in\scope{\Psi} \),  then by \( \Psi\dylRevision\beta\models\alpha \) there exist \( \omega'\models\beta \land\alpha \) with \( \omega' \prec_\Psi \omega \) .
	From (i) and (ii) obtain the contradiction \( \omega' \prec_{\Psi\dylRevision\alpha} \omega \).\\
	Having \( \omega\notin\scope{\Psi} \) and \( \Psi\models\alpha \) at the same time is impossible by (iii). 
	From %
	\( \Psi\not\models\alpha \) %
	obtain that \( \modelsOf{\Psi\dylRevision\beta} = \min(\modelsOf{\beta},\preceq_{\Psi})  \) and therefore there exist \( \omega'\in\scope{\Psi} \) with \( \omega'\in \modelsOf{\alpha\land\beta} \).
	From (i) and (iv) obtain the contradiction \( \Psi\dylRevision\beta\not\models\alpha \) %
	Consider the case of \( \omega\notin\scope{\Psi\dylRevision\alpha} \). Then obtain \( \modelsOf{\Psi\dylRevision\alpha}=\modelsOf{\Psi\dylRevision\alpha\dylRevision\beta} \) from \eqref{eq:limited_revision}.
	Because of \( \omega\not\models\alpha \) and \( \omega\in\modelsOf{\Psi\dylRevision\alpha} \) we have \( \modelsOf{\Psi\dylRevision\alpha}=\modelsOf{\Psi} \) and \( \Psi\not\models\alpha \). 
	We obtain \( \modelsOf{\Psi\dylRevision\beta} \subseteq \scope{\Psi} \) as consequence.
	Because  \( \Psi\dylRevision\beta\models\alpha \) there exists \( \omega'\in\scope{\Psi} \) with \( \omega'\models\alpha\land\beta \).
	This contradicts \( \modelsOf{\Psi\dylRevision\alpha}=\modelsOf{\Psi} \).\qedhere	
\end{proof} 

\setcounterref{theorem}{prop:dylr_dp4}
\addtocounter{theorem}{-1}
\begin{proposition}
	A \dynamiclimited\ revision operator $ \dylRevision $ compatible with $ \Psi\mapsto(\preceq_{\Psi},\Omega_\Psi) $ satisfies \eqref{pstl:DP4} if and only if the following conditions hold:
	\begin{enumerate}[(i)]
		\item if and  \(  \omega_1 {\models} \alpha \) and \( \omega_2 {\not\models} \alpha\) and \( \omega_1,\omega_2{\in}\scope{\Psi}{\cap}\scope{\Psi\dylRevision\alpha} \), then
		\begin{equation*}
			\omega_1 \prec_{\Psi\dylRevision\alpha} \omega_2 \Rightarrow \omega_1 \prec_{\Psi} \omega_2
		\end{equation*}
		\item if  \(  \omega_1 \models\alpha \) and \( \omega_2 \not\models\alpha\) and \( \omega_2 \prec_{\Psi\dylRevision\alpha} \omega_1 \), then
		\begin{equation*}
			\omega_1 \in \scope{\Psi} \Rightarrow  \omega_2 \in \scope{\Psi} %
		\end{equation*}
		\item if \( \omega \not\models \alpha \) and \( \Psi \not\models \negOf{\alpha}  \), then 
		\(  \omega \in\scope{\Psi\dylRevision\alpha} \Rightarrow \omega \in\scope{\Psi} \)
		\item if  \(  \omega_1 {\models} \alpha \) and \( \omega_2 {\not\models} \alpha\) and \( \omega_1 {\in} \scope{\Psi} \) and \( \omega_2 {\in} \scope{\Psi\dylRevision\alpha} \), then
		\begin{equation*}
			\left( \omega_2\notin\scope{\Psi} \ksOR \omega_1 \preceq_{\Psi} \omega_2 \right) \Rightarrow  \omega_1 \in \scope{\Psi\dylRevision\alpha}
		\end{equation*}
	\end{enumerate}
\end{proposition}
\begin{proof}
	Analogue to the proof of Proposition \ref{prop:dylr_dp3}.
\end{proof} 

\setcounterref{theorem}{prop:clp_dyl}
\addtocounter{theorem}{-1}
\begin{proposition}
	Let $ \dylRevision $ be a \dynamiclimited\ revision operator compatible with $ \Psi\mapsto(\preceq_{\Psi},\scope{\Psi}) $. Then  \( \dylRevision  \) satisfies \eqref{pstl:CLP} if and only if the following conditions holds for all \( \omega_1,\omega_2 \) with \( \varphi_{\omega_1}\in\beliefsOf{\Psi\dylRevision\varphi_{\omega_1}} \) and \( \varphi_{\omega_2}\in\beliefsOf{\Psi\dylRevision\varphi_{\omega_2}} \):
	\begin{enumerate}[(i)]
		\item \(  \ksIF \omega_1\models\alpha \ksAND \omega_2\not\models\alpha \ksAND \omega_1,\omega_2{\in}\scope{\Psi}{\cap}\scope{\Psi\dylRevision\alpha}  \), then
		\begin{equation*}
			\omega_1 \preceq_{\Psi} \omega_2 \Rightarrow \omega_1 \prec_{\Psi\dylRevision\alpha} \omega_2
		\end{equation*}
		\item \( \ksIF \omega_1\models\alpha \ksAND \omega_2\not\models\alpha \ksAND  \omega_1 \preceq_{\Psi} \omega_2 \), then:
		\begin{equation*}
			\omega_2 \in \scope{\Psi\dylRevision\alpha} \Rightarrow \omega_1 \in \scope{\Psi\dylRevision\alpha}
		\end{equation*}
		\item \( \ksIF \omega\not\models\alpha \ksAND \Psi\not\models\negOf{\alpha} \), then \( \omega\in\scope{\Psi\dylRevision\alpha} \Rightarrow \omega\in\scope{\Psi} \)
		\item \(  \ksIF \omega_1{\models}\alpha \ksAND \omega_2{\not\models}\alpha \ksAND \omega_1{\in}\scope{\Psi} \ksAND  \omega_2{\in}\scope{\Psi\dylRevision\alpha} \), then
		\begin{equation*}
			(\omega_2\notin\scope{\Psi} \ksOR \omega_1\preceq_{\Psi} \omega_2 ) \Rightarrow \omega_1\in\scope{\Psi\dylRevision\alpha}
		\end{equation*}
	\end{enumerate}
\end{proposition}
\begin{proof}
Analogue to the proof of Proposition \ref{prop:dylr_dp3}.
\end{proof} 

\setcounterref{theorem}{prop:inherence_agm}
\addtocounter{theorem}{-1}
\begin{proposition}
	Let \( \setAllES \) be \ref{pstl:unbiased}. For every AGM revision operator $ * $  every consistent belief is \immanent, and a belief $ \alpha\in\propLang $ is \inherent\ if and only if it has exactly one model.
\end{proposition}
\begin{proof}
	For every consistent belief set $ L $ there exists at least one state $ \Psi $ having this belief set, i.e., $ \beliefsOf{\Psi}=L $.
	For $ \omega\in\Omega $, regardless of the state, we have $ \modelsOf{\Psi*\varphi_{\omega}}=\{\omega\} $ due to Proposition \ref{prop:es_revision}.
	Thus, every $ \varphi_{\omega} $ is an \inherent\ belief with exactly one model.
Assume now that $ \alpha\in\propLang $ has two or more models. Let \( \omega \in \modelsOf{\alpha} \).
	Then there is some belief set \( \modelsOf{L} \) with \( \modelsOf{L}=\modelsOf{\alpha} \setminus\{\alpha\} \).
	Because \( L \) is consistent and \( \setAllES \) \ref{pstl:unbiased}, there is some state \( \Psi_L\in \setAllES \) with \( \beliefsOf{\Psi_L}= L \). By Proposition \ref{prop:es_revision}, we obtain $ \modelsOf{\Psi \revision \alpha}=\modelsOf{L} $. Thus, \( \alpha \) is no \inherent\ belief.
	As $ \varphi_{\omega} $ for every $ \omega\in\Omega $  is \inherent, every consistent belief is \immanent. \qedhere
\end{proof} 

\setcounterref{theorem}{prop:inherent_clr}
\addtocounter{theorem}{-1}
\begin{proposition}
	Let $ \nrRevision $ be a credibility-limited revision operator and $ \Psi\mapsto(\leq_\Psi,C_\Psi) $ be a corresponding CLF-assignment. A belief $ \alpha $ is \inherent\ for $ \nrRevision $ if and only if $ \alpha $ has exactly one model $ \omega $ and $ \omega\in C_\Psi $ for every $ \Psi\in\setAllES $.
	Therefore, there is a credibility-limited revision operator with no \immanent\  and \inherent\ beliefs.
\end{proposition}
\begin{proof}
	If $ \alpha $ has more than one model, then by \eqref{pstl:LR2} it is no \inherent\ belief when coosing $ \Psi $ with $ \modelsOf{\Psi}\subsetneq \modelsOf{\alpha} $.
	Likewise by Equation \eqref{eq:cl_revision} and \eqref{pstl:LR2} the condition $ \modelsOf{\alpha} \subseteq C_\Psi $ for all $ \Psi\in\setAllES $ is easy to see.	
	
	For the last statement, choose a  CLF-assignment $ \Psi\mapsto(\leq_\Psi,C_\Psi) $ with $ \modelsOf{\Psi}=C_\Psi $ for every epistemic state $ \Psi $. \qedhere
\end{proof} 

\setcounterref{theorem}{lem:inherent_full_dynamic}
\addtocounter{theorem}{-1}
\begin{proposition}
	Let \( \setAllES \) be \ref{pstl:unbiased}, and $ \Psi \mapsto (\preceq_{\Psi},\scope{\Psi}) $ a limited assignment compatible with a \dynamiclimited\ revision operator $ \dylRevision $.
	If $ \alpha $ is \inherent, then either  $ \modelsOf{\alpha} \subseteq \scope{\Psi} $ or \( \scope{\Psi} \cap \modelsOf{\alpha}=\emptyset  \). Moreover, in the latter case \( \modelsOf{\Psi}=\modelsOf{\alpha} \).
\end{proposition}
\begin{proof}
	Let $ \alpha $ be an \inherent\ belief of $ \circ $ and let $ \Psi\in\setAllES $ an arbitrary epistemic state, with $ \scope{\Psi}=\{\omega_1,\ldots,\omega_k\} $.
	Assume \( \omega \) such that \( \omega\models\alpha \), but \( \omega\notin \scope{\Psi} \). 
	If \( \scope{\Psi}\cap\modelsOf{\alpha} \neq \emptyset \), then by Definition \ref{def:dynamic_limited_revision} we obtain \( \modelsOf{\Psi\dylRevision\alpha} \subsetneq \modelsOf{\alpha} \). 
	This is a contradiction, because \( \alpha \) is \inherent. Consequently, we obtain \( \scope{\Psi}\cap\modelsOf{\alpha} = \emptyset \). Moreover, because \( \alpha \) is \inherent, we obtain \( \modelsOf{\Psi}=\modelsOf{\alpha} \).\qedhere

\end{proof} 

\setcounterref{theorem}{lem:inherent_inherence_limited}
\addtocounter{theorem}{-1}
\begin{lemma}
	Let $ \nrRevision $ be an \ihlimited\ revision operator. 
	If $ \alpha $ is \immanent\ for $ \circ $, then every $ \varphi_\omega\models\alpha $ is \inherent\ for $ \circ $.
\end{lemma}
\begin{proof}
	Let $ \alpha $ and $ \varphi_\omega $ be as above. Let $ \Psi $ be such that $ \modelsOf{\Psi}=\{\omega\} $. By \eqref{pstl:IL3}, we have $ \modelsOf{\Psi\nrRevision\alpha}=\{\omega\} $. From \eqref{pstl:IL4} we obtain that $ \Cn(\varphi_{\omega})  $ is an \inherent\ belief set. This implies that $ \varphi_{\omega} $ is \inherent. \qedhere
\end{proof} 

\setcounterref{theorem}{thm:inherence_limited_operator}
\addtocounter{theorem}{-1}
\begin{theorem}
A belief change operator $ \nrRevision $ is an \ihlimited\ revision operator
if and only if \( \nrRevision \) is a \dynamiclimited\ revision operator compatible with $ \Psi \mapsto {(\preceq_\Psi,\scope{\Psi})} $ such that there is \( \Omega'\subseteq \Omega \)  with \( \scope{\Psi} = \Omega' \) for every epistemic states $ \Psi $.
\end{theorem}
\begin{proof}
\noindent\textit{The \enquote*{$ \Rightarrow $}-direction.} 
Let $ \nrRevision $ be an operator satisfying the postulates \eqref{pstl:IL1}--\eqref{pstl:IL8}. 
For $ \Psi $, we construct $ (\preceq_{\Psi},\scope{\Psi}) $. In particular, let $ \preceq_{\Psi} $ be the relation with
\(	\omega_1 \preceq_\Psi \omega_2 \ksIFF \omega_1 \in \modelsOf{\Psi\nrRevision (\varphi_{\omega_1}\lor\varphi_{\omega_2})}\)
and $ \scope{\Psi}=\dom(\preceq_\Psi)= \{ \omega \mid \varphi_\omega \text{ is \inherent\ for } \nrRevision \} $.
By the definition of inherence, $ \scope{\Psi} $ is the same for each $ \Psi $, i.e. for every $ \Psi,\Gamma\in\setAllES $ we have $ \omega\in{\scope{\Psi}} $ if and only if $ \omega\in{\scope{\Psi}} $. 
The order $ \preceq_{\Psi} $ is a total preorder:

	\emph{Totality/reflexivity.} Let $ \omega_1,\omega_2\in\scope{\Psi} $. Then by \eqref{pstl:IL8} we have that $ \modelsOf{\Psi\nrRevision(\varphi_{\omega_1}\lor\varphi_{\omega_2})} $ is equivalent to $ \{ \omega_1 \} $ or $ \{ \omega_2 \} $ or $ \{ \omega_1,\omega_2 \} $. Therefore, the relation must be total. Reflexivity follows from totality.

\emph{Transitivity.} 	Let $ \omega_1,\omega_2,\omega_3\in\scope{\Psi} $ with $ \omega_1 \preceq_{\Psi} \omega_2 $ and $ \omega_2\preceq_{\Psi} \omega_3 $. Towards a contradiction assume that $ \omega_1 \not\preceq_{\Psi} \omega_3 $ holds. This implies  $ \modelsOf{\Psi\nrRevision\varphi_{\omega_1,\omega_3}}=\{\omega_3\} $.
	Now assume that $ \modelsOf{\Psi\nrRevision\varphi_{\omega_1,\omega_2,\omega_3}} = \{\omega_3\} $.
	Then by \eqref{pstl:IL8} we obtain that $ \modelsOf{\Psi\nrRevision\varphi_{\omega_2,\omega_3}}=\{\omega_3\} $, a contradiction to $ \omega_2\preceq_{\Psi}\omega_3 $.
	Assume for the remaining case $ \modelsOf{\Psi\nrRevision\varphi_{\omega_1,\omega_2,\omega_3}} \neq \{\omega_3\} $. 
	We obtain from \eqref{pstl:IL8} that $ \modelsOf{\Psi\nrRevision\varphi_{\omega_1,\omega_2,\omega_3}} $ equals either $ \modelsOf{\Psi\nrRevision\varphi_{\omega_1,\omega_2}} $ or $ \modelsOf{\Psi\nrRevision\varphi_{\omega_3}}$. 
	Since the second case is impossible, $ \omega_1\preceq_{\Psi} \omega_2 $ implies $ \omega_1\in \modelsOf{\Psi\nrRevision\varphi_{\omega_1,\omega_2}} $. Now apply \eqref{pstl:IL8} again to $ \modelsOf{\Psi\nrRevision\varphi_{\omega_1,\omega_2,\omega_3}} $ and obtain that it is either equivalent to $ \modelsOf{\Psi\nrRevision\varphi_{\omega_1,\omega_3}} $ or $ \modelsOf{\Psi\nrRevision\varphi_{\omega_2}} $.
	In both cases, we obtain a contradiction because $ \omega_1\in \modelsOf{\Psi\nrRevision\varphi_{\omega_1,\omega_2,\omega_3}} $.

The construction yields a faithful limited assignment:\\
We show $ \modelsOf{\Psi}\cap\scope{\Psi}=\min(\Omega,\preceq_{\Psi}) $. 
Let $ \omega_1,\omega_2\in\scope{\Psi} $ and $ \omega_1,\omega_2\in\modelsOf{\Psi} $. By definition of $ {\scope{\Psi}} $ the interpretations $ \omega_1,\omega_2 $ are \inherent. From \eqref{pstl:IL3} we obtain $ \modelsOf{\Psi\nrRevision(\varphi_{\omega_1}\lor\varphi_{\omega_2})}=\{ \omega_1,\omega_2 \} $ which yields by definition $ \omega_1 \preceq_{\Psi} \omega_2 $ and $ \omega_2 \preceq_{\Psi} \omega_1 $.
Let $ \omega_1,\omega_2\in\scope{\Psi} $ with  $ \omega_1\in\modelsOf{\Psi} $ and $ \omega_2\notin\modelsOf{\Psi} $.
Then $ \varphi_{\omega_1}\lor\varphi_{\omega_2} $ is consistent with $ \beliefsOf{\Psi} $. 
Therefore, by 
\eqref{pstl:IL3} and \eqref{pstl:IL5} we have $ \modelsOf{\Psi\nrRevision (\varphi_{\omega_1}\lor\varphi_{\omega_2})} = \modelsOf{\Psi}\cap\modelsOf{\varphi_{\omega_1}\lor\varphi_{\omega_2}}=\{ \omega_1 \} $.
Together we obtain faithfulness.

We show the satisfaction of \eqref{eq:limited_revision} in two case.

For the first case, assume $ \modelsOf{\alpha}\cap\scope{\Psi}=\emptyset $.
By the postulate  \eqref{pstl:IL2} we have either $ \beliefsOf{\Psi}=\beliefsOf{\Psi\nrRevision\alpha} $ or $ \beliefsOf{\Psi\nrRevision\alpha} $ is \immanent. In the first case we are done. 
For the second case, by the postulate \eqref{pstl:IL1}, we obtain $ \modelsOf{\Psi\nrRevision\alpha}\subseteq\modelsOf{\alpha} $. 
Thus, by Lemma \ref{lem:inherent_inherence_limited} the set  $ \modelsOf{\alpha }$ contains an interpretation $ \omega $ such that $ \varphi_{\omega} $ is \inherent, a contradiction to $ \modelsOf{\alpha}\cap\scope{\Psi}=\emptyset $.

For the second case assume $ \modelsOf{\alpha}\cap\scope{\Psi}\neq\emptyset $.
We show the equivalence $ \modelsOf{\Psi\nrRevision\alpha}=\min(\modelsOf{\alpha},\preceq_{\Psi}) $ by showing both set inclusions separately.

	We show $ \min(\modelsOf{\alpha},\preceq_{\Psi})\subseteq \modelsOf{\Psi\nrRevision\alpha} $.
Let $ \omega\in\min(\modelsOf{\alpha},\preceq_{\Psi}) $ with $ \omega\notin\modelsOf{\Psi\nrRevision\alpha}$. 
By construction of \( \scope{\Psi} \) and \eqref{pstl:IL3} we have $ \omega\in\modelsOf{\Psi\nrRevision\varphi_{\omega} }$. 
From Lemma \ref{lem:inherent_full_dynamic} and Lemma \ref{lem:inherent_inherence_limited} and postulate \eqref{pstl:IL4}, we obtain that $\beliefsOf{\Psi\nrRevision\alpha} $ is \immanent\ and every $ \omega\in\modelsOf{\Psi\nrRevision\alpha} $ is \inherent. 
Hence, there is at least one $ \omega'\in\beliefsOf{\Psi\nrRevision\alpha}  $ with $ \omega'\in\scope{\Psi} $.
If $ \beliefsOf{\Psi\nrRevision\alpha}=\beliefsOf{\Psi} $ we obtain by the faithfulness $ \min(\modelsOf{\alpha},\preceq_{\Psi}) = \modelsOf{\Psi\nrRevision\alpha} $. If $ \beliefsOf{\Psi\nrRevision\alpha}\neq\beliefsOf{\Psi} $, then by \eqref{pstl:IL1} we have $ \modelsOf{\Psi\nrRevision\alpha}\subseteq\modelsOf{\alpha} $.
Let $ \beta=\varphi_{\omega}\lor\varphi_{\omega'} $ and $ \gamma=\beta\lor\gamma' $ such that $ \modelsOf{\gamma}=\modelsOf{\alpha} $ and $ \modelsOf{\gamma'}=\modelsOf{\alpha}\setminus\{\omega,\omega'\} $.
By \eqref{pstl:IL8} we have either $ \modelsOf{\Psi\nrRevision\alpha}=\modelsOf{\Psi\nrRevision\beta} $ or $ \modelsOf{\Psi\nrRevision\alpha}=\modelsOf{\Psi\nrRevision\gamma'} $ or $ \modelsOf{\Psi\nrRevision\alpha}=\modelsOf{\Psi\nrRevision\beta}\cup\modelsOf{\Psi\nrRevision\gamma'} $.
The first and the third case are impossible, because $ \omega\notin\modelsOf{\Psi\nrRevision\alpha} $ and by the minimality of $ \omega $ we have $ \omega \in \modelsOf{\Psi\nrRevision\beta} $.
It remains the case of $ \modelsOf{\Psi\nrRevision\alpha}=\modelsOf{\Psi\nrRevision\gamma'} $. 
Let $ \omega_{\gamma'} $ such that $ \omega_{\gamma'}\in\modelsOf{\Psi\nrRevision\alpha} $. Note that $ \omega_{\gamma'}\in\modelsOf{\gamma'}\subseteq\modelsOf{\alpha} $ and $ \omega_{\gamma'}\in\scope{\Psi} $.
Now let $ \delta=\beta'\lor\delta' $ with $ \delta\equiv\alpha $ such that $ \modelsOf{\beta'}=\{\omega,\omega_{\gamma'} \} $ and $ \modelsOf{\delta'}=\modelsOf{\alpha}\setminus\{ \omega,\omega_{\gamma'} \} $. 
By minimality of $ \omega $ we have $ \omega\in\modelsOf{\Psi\nrRevision\beta'} $.
By \eqref{pstl:IL8} we obtain $ \modelsOf{\Psi\nrRevision\alpha} $ is either equivalent to $ \modelsOf{\Psi\nrRevision\beta'} $ or to $ \modelsOf{\Psi\nrRevision\delta'} $ or to $  \modelsOf{\Psi\nrRevision\beta'} \cup \modelsOf{\Psi\nrRevision\delta'} $. The first and third case are impossible since $ \omega\notin \modelsOf{\Psi\nrRevision\alpha} $. Moreover, the second case is also impossible, because of $ \omega_{\gamma'}\notin \modelsOf{\Psi\nrRevision\delta'} $.

	We show $ \modelsOf{\Psi\nrRevision\alpha} \subseteq \min(\modelsOf{\alpha},\preceq_{\Psi}) $.
	Let  $ \omega\in\modelsOf{\Psi\nrRevision\alpha} $ with  $ \omega\notin \min(\modelsOf{\alpha},\preceq_{\Psi})  $.
	From non-emptiness of $ \min(\modelsOf{\alpha},\preceq_{\Psi}) $ and $ \min(\modelsOf{\alpha},\preceq_{\Psi})\subseteq \modelsOf{\Psi\nrRevision\alpha} $ we obtain an \inherent\ belief $ \varphi_{\omega'} $, where $ \omega'\in\modelsOf{\Psi\nrRevision\alpha} $ such that $ \omega'\in \min(\modelsOf{\alpha},\preceq_{\Psi}) $.

Assume that $ \varphi_\omega $ is not \inherent, and therefore, $ {\omega\notin\scope{\Psi}} $.
By the postulate \eqref{pstl:IL4} we obtain from the existence of $ \omega' $ that $ \beliefsOf{\Psi\nrRevision\alpha} $ is \immanent.  Therefore, by Lemma \ref{lem:inherent_inherence_limited} every model in $ \modelsOf{\Psi\nrRevision \alpha} $ is \inherent, a contradiction, and  therefore, $ \varphi_\omega $ has to be \inherent.

From inherence of $ \varphi_\omega $ we obtain $ \modelsOf{\Psi\nrRevision\varphi_{\omega}}=\{\omega\} $.
Using the postulate \eqref{pstl:IL4} we obtain that $ \beliefsOf{\Psi\nrRevision\alpha} $ is \immanent.
Since $ \omega $ is not minimal, we have $ \omega'\in\modelsOf{\Psi\nrRevision(\varphi_{\omega,\omega'})} $ and $ \omega\notin\modelsOf{\Psi\nrRevision(\varphi_{\omega,\omega'})} $
Now let $ \gamma=\varphi_{\omega,\omega'}\lor\gamma' $ with $ \modelsOf{\gamma'}=\modelsOf{\alpha}\setminus\{\omega,\omega'\} $.
	By \eqref{pstl:IL8} we have either $ \modelsOf{\Psi\nrRevision\alpha}=\modelsOf{\Psi\nrRevision\beta} $ or $ \modelsOf{\Psi\nrRevision\alpha}=\modelsOf{\Psi\nrRevision\gamma'} $ or $ \modelsOf{\Psi\nrRevision\alpha}=\modelsOf{\Psi\nrRevision\beta}\cup\modelsOf{\Psi\nrRevision\gamma'} $.
	All cases are impossible because for every case we obtain $ \omega\notin\modelsOf{\Psi\nrRevision\alpha} $.
This shows $ \modelsOf{\Psi\nrRevision\alpha} \subseteq \min(\modelsOf{\alpha},\preceq_{\Psi}) $, and, in summary, we obtain $ \modelsOf{\Psi\nrRevision\alpha} = \min(\modelsOf{\alpha},\preceq_{\Psi}) $ and thus, \eqref{eq:limited_revision} holds.

\smallskip
\noindent\textit{The \enquote*{$ \Leftarrow $}-direction.}
Let $ \Psi \mapsto (\preceq_\Psi,\scope{\Psi}) $ be a limited assignment compatible with \( \nrRevision \) such that \( \scope{\Psi}=\scope{\Gamma} \) for all \( \Psi,\Gamma\in\setAllES \).
Furthermore, let $ \alpha $ be an \inherent\  belief of $ \nrRevision $. Then by Lemma \ref{lem:inherent_full_dynamic} every model $ \omega $ of $ \alpha $ is an element of $ \scope{\Psi} $. Moreover, every $ \varphi_\omega $ with $ \omega\in\scope{\Psi} $ is an \inherent\ belief.

We show the satisfaction of \eqref{pstl:IL1}--\eqref{pstl:IL8}.
From \eqref{eq:limited_revision} we obtain straightforwardly \eqref{pstl:IL1}, \eqref{pstl:IL2}, \eqref{pstl:IL5} and \eqref{pstl:IL6}.
\begin{description}
	\item[\eqref{pstl:IL3}] Assume $ \alpha $ to be \immanent\ and $ \modelsOf{\Psi}\cap\modelsOf{\alpha}\neq\emptyset $.
	By Lemma \ref{lem:inherent_full_dynamic} we have $ \modelsOf{\Psi}\cap\modelsOf{\alpha} \subseteq \scope{\Psi} $.
	Thus, $ \min(\modelsOf{\alpha},\preceq_{\Psi})\neq\emptyset $. From faithfulness we obtain $ \min(\modelsOf{\alpha},\preceq_{\Psi})=\modelsOf{\Psi}\cap\modelsOf{\alpha} $.
	\item[\eqref{pstl:IL4}] Let $ \alpha $ be immanent and thus by Lemma \ref{lem:inherent_full_dynamic} we have $ \modelsOf{\Psi}\cap\modelsOf{\alpha} \subseteq \scope{\Psi} $.
	This implies $ \alpha\models\beta $ and we obtain $ \mid(\modelsOf{\beta},\preceq_{\Psi})\neq\emptyset $. Then by \eqref{eq:limited_revision} we obtain $ \beta\in \beliefsOf{\Psi\nrRevision} $.
	\item[\eqref{pstl:IL8}] 
	Suppose $ \modelsOf{\alpha\lor\beta}\cap\scope{\Psi}=\emptyset $, then by \eqref{eq:limited_revision} for every formula $ \gamma $ with $ \modelsOf{\gamma} \subseteq \modelsOf{\alpha\lor\beta} $ we obtain $ \modelsOf{\Psi\nrRevision(\alpha\lor\beta)}=\modelsOf{\Psi}=\modelsOf{\Psi\nrRevision\gamma} $.
	
	Assume that $ \modelsOf{\alpha}\cap\scope{\Psi}\neq\emptyset $ and $ \modelsOf{\beta}\cap\scope{\Psi}=\emptyset $.
	Then by the postulate \eqref{eq:limited_revision} we obtain $ \modelsOf{\Psi\nrRevision(\alpha\lor\beta)}=\modelsOf{\Psi\nrRevision\alpha} $.
	The case of $ \modelsOf{\alpha}\cap\scope{\Psi}=\emptyset $ and $ \modelsOf{\beta}\cap\scope{\Psi}\neq\emptyset $ is analogue.

	Assume that $ \modelsOf{\alpha}\cap\scope{\Psi}\neq\emptyset $ and $ \modelsOf{\beta}\cap\scope{\Psi}\neq\emptyset $.
	Then we have $ \modelsOf{\Psi\nrRevision(\alpha\lor\beta)}={\min(\modelsOf{\alpha\lor\beta},\preceq_{\Psi})} $. Using logical equivalence, we obtain $ \min(\modelsOf{\alpha\lor\beta},\preceq_{\Psi})=\min(\modelsOf{\alpha}\cup\modelsOf{\beta},\preceq_{\Psi}) $.
	By Lemma \ref{lem:sem_trichotonomy} we obtain directly \eqref{pstl:IL8}. \qedhere
\end{description}
\end{proof} 

\setcounterref{theorem}{prop:inhlimited_clr_agm}
\addtocounter{theorem}{-1}
\begin{proposition}
	Let \( \setAllES \) be \ref{pstl:unbiased}. A belief change operator $ \circ $ is an \ihlimited\ revision operator and a credibility-limited revision operator at the same time if and only if $ \circ $ is an AGM revision operator.
\end{proposition}
\begin{proof}
	Because every faithful assignment is also a CLF-assignment and a faithful limited assignment, every AGM revision operator is also a credibility-limited revision operator and an \ihlimited\ revision operator.
	
	Let $ \nrRevision $ an \ihlimited\  revision operator, but no AGM revision operator.
	Then there is an faithful limited assignment  $ \Psi\mapsto(\preceq_{\Psi},\Omega') $ with $ \Omega'\subsetneq \Omega $.
	Let $ \Psi_\top $ an epistemic state with $ \modelsOf{\Psi_\top}=\Omega$. 
	Let $ \alpha $ such that there is an $ \omega\models\alpha $ with $ \omega\notin\Omega' $. Then, because of \eqref{pstl:LR2}, for every credibility-limited revision $ \clRevision $ operator $ \omega\in\modelsOf{\Psi_\top\clRevision\alpha} $, but $ \omega\notin\modelsOf{\Psi_\top\nrRevision\alpha} $.

	Let $ \clRevision $ an credibility-limited revision operator, but no AGM revision operator. Then there is a CLF-assignment  $ \Psi\mapsto(\leq_{\Psi},C_\Psi) $ with $ \modelsOf{\Psi}\subseteq C_\Psi\subseteq \Omega $. 
	Moreover, there is at least one $ \Gamma\in\setAllES $ with $ C_\Gamma\subsetneq \Omega $, i.e., there is $ \omega\not\in C_\Gamma $, and therefore, $ \omega\notin\modelsOf{\Gamma} $. 
	Let again denote $ \Psi_\top $ an epistemic state with $ \modelsOf{\Psi_\top}=\Omega$.
	We obtain, because of \eqref{pstl:LR2}, that $ \modelsOf{\Gamma\clRevision\varphi_{\omega}}=\modelsOf{\Gamma} $ and $ \modelsOf{\Psi_\top\clRevision\varphi_{\omega}}=\{\omega\} $. 
	Assume that $ \clRevision $ is also an \ihlimited\ revision operator. 
	From \eqref{pstl:IL2} and $ \modelsOf{\Psi_\top\clRevision\varphi_{\omega}}=\{\omega\} $ we obtain the immanence of $ \varphi_{\omega} $.
	Lemma \ref{lem:inherent_inherence_limited} implies that $ \varphi_{\omega} $ is \inherent. This leads to the contradiction $ \modelsOf{\Gamma\clRevision\varphi_{\omega}}=\{\omega\} $. \qedhere
\end{proof} 

\end{document}